%% file: main_arxiv.tex
\documentclass{article}

\usepackage[english]{babel}
\usepackage[utf8]{inputenc}
\usepackage[T1]{fontenc}
\usepackage[a4paper,left=3cm,right=3cm,top=3cm,bottom=3.2cm]{geometry} 
\usepackage{tikz}
\usetikzlibrary[topaths] 

\usepackage{amsmath, mathrsfs, amssymb, amsthm, mathtools, bbm, bm} 

\usepackage{subfig} 

\usepackage{caption,float}

\usepackage{makecell} 
\usepackage{enumitem} 

\newtheorem{theorem}{Theorem}
 
\newtheorem{definition}{Definition}
\newtheorem{lemma}{Lemma}
\newtheorem{remark}{Remark}

\newtheorem{proposition}{Proposition}

\newtheorem*{theorem*}{Theorem}
\newtheorem*{example*}{Example} 
\newtheorem*{definition*}{Definition}
\newtheorem*{lemma*}{Lemma}
\newtheorem*{remark*}{Remark}
\newtheorem*{corollary*}{Corollary}
\newtheorem*{proposition*}{Proposition}
\newtheorem*{assumption*}{Assumption}
\newtheorem*{claim*}{Claim}

\newtheoremstyle{TheoremNum}
        {\topsep}{\topsep}              
        {\itshape}                      
        {}                              
        {\bfseries}                     
        {.}                             
        { }                             
        {\thmname{#1}\thmnote{ \bfseries #3}}
\theoremstyle{TheoremNum}

\newtheoremstyle{LemmaNum}
        {\topsep}{\topsep}              
        {\itshape}                      
        {}                              
        {\bfseries}                     
        {.}                             
        { }                             
        {\thmname{#1}\thmnote{ \bfseries #3}}
\theoremstyle{LemmaNum}

\usepackage{hyperref} 

\usepackage{natbib, float}

\input{macros}

\begin{document} 

\title{Stochastic Zeroth Order Gradient and Hessian Estimators: Variance Reduction and Refined Bias Bounds} 

\author{Yasong Feng\footnote{ysfeng20@fudan.edu.cn} \quad and \quad Tianyu Wang\footnote{wangtianyu@fudan.edu.cn}}  

\date{} 

\maketitle

\begin{abstract}
    We study stochastic zeroth order gradient and Hessian estimators for real-valued functions in $\mathbb{R}^n$. We show that, via taking finite difference along random orthogonal directions, the variance of the stochastic finite difference estimators can be significantly reduced. 
    In particular, we design estimators for smooth functions such that, if one uses $ \Theta \left( k \right) $ random directions sampled from the Stiefel manifold $ \text{St} (n,k) $ and finite-difference granularity $\delta$, the variance of the gradient estimator is bounded by $ \mathcal{O} \left( \left( \frac{n}{k}  - 1 \right) + \left( \frac{n^2}{k}  - n \right) \delta^2 + \frac{ n^2 \delta^4 }{ k } \right) $, and the variance of the Hessian estimator is bounded by $\mathcal{O} \left( \left( \frac{n^2}{k^2} - 1 \right) + \left( \frac{n^4}{k^2} - n^2 \right) \delta^2 + \frac{n^4 \delta^4 }{k^2} \right) $. When $k = n$, the variances become negligibly small. 
    In addition, we provide improved bias bounds for the estimators. The bias of both gradient and Hessian estimators for smooth function $f$ is of order 
    $\mathcal{O} \left( \delta^2 \Gamma \right)$,
    where $\delta$ is the finite-difference granularity, and $ \Gamma $ depends on high order derivatives of $f$. 
    Our results are evidenced by empirical observations.
\end{abstract}

\input{intro}
\input{prelim}

\input{grad}

\input{hessian}

\input{exp-grad}

\clearpage

\input{exp-hess}

\clearpage

\input{conclusion} 


\bibliographystyle{apalike} 
\bibliography{references}

\appendix

\input{appendix}

\end{document}

%% file: macros.tex
\usepackage{amsmath, mathrsfs, amssymb, amsthm, mathtools, bbm, bm} 

\usepackage{tikz} 
\usetikzlibrary[topaths] 

\usepackage{caption} 
\newsavebox\mysavebox

\usepackage{subfig} 

\usepackage{caption,float}

\usepackage{makecell} 
\usepackage{enumitem}

\newcommand{\pgftextcircled}[1]{
    \setbox0=\hbox{#1}%
    \dimen0\wd0%
    \divide\dimen0 by 2%
    \begin{tikzpicture}[baseline=(a.base)]%
        \useasboundingbox (-\the\dimen0,0pt) rectangle (\the\dimen0,1pt);
        \node[circle,draw,outer sep=0pt,inner sep=0.1ex] (a) {#1};
    \end{tikzpicture}
}

\let\textcircled=\pgftextcircled

\usepackage{makecell} 

\usepackage{subfig}
\usepackage{float}

\usepackage{enumitem} 

\usepackage{algorithm} 
\usepackage{algcompatible}

\newcommand{\e}{ \bm{ e } } 
\newcommand{\tr}{\mathsf{T}}

\renewcommand{\S}{ \mathbb{ S } }

\renewcommand{\H}{ \mathrm{ H } }

\renewcommand{\O}{\mathcal{O}}

\newcommand{\R}{\mathbb{R}}
\newcommand{\B}{\mathbb{B}}
\newcommand{\E}{\mathbb{E}}

\renewcommand{\[}{\left[ }
\renewcommand{\]}{\right] }

\newcommand{\<}{\left< }
\renewcommand{\>}{\right> }

\renewcommand{\(}{\left( }
\renewcommand{\)}{\right) }

\newcommand{\wt}{\widetilde }
\newcommand{\wh}{\widehat }

%% file: intro.tex
\section{Introduction} 

Since Newton's time, people have been using finite difference principles to estimate derivatives. This classic problem has recently revived, as tasks of stochastic derivative estimation in high dimension become prevalent. 

Various bias bounds have been derived for gradient and Hessian estimators \citep[e.g.,][]{flaxman2005online,nesterov2017random,balasubramanian2021zeroth,wang2022hess}. 
Yet the statistical convergence to these bias bounds is slow due to large variance, especially in high dimensional spaces. 
The bias of a gradient estimator $ \wh{\nabla} f (x) $ is 
\begin{align*} 
    \left\| \E \[ \wh{\nabla} f (x) \] - \nabla f (x) \right\|, 
\end{align*} 
and the bias of a Hessian estimator is similarly defined. 
In practice, the estimation error of $  \wh{\nabla} f (x) $ is measured by $ \left\| \wh{\nabla} f (x) - \nabla f (x) \right\| $ (similarly for the Hessian counterpart). Unless the estimator is highly concentrated around its expectation, the theoretical bias bound may not be aligned with the empirical observations. This discrepancy calls for careful study on the variance and variance reduction for the estimators. 
To this end, we introduce variance-reduced methods for stochastic zeroth order gradient and Hessian estimation, and provide performance guarantees for these methods. For estimating the gradient of a function $f$ in $ \R^n $, we propose to uniformly sample a matrix $ [ v_1, v_2, \cdots, v_k ] $ from the real Stiefel manifold $ \text{St} (n,k) := \{ X \in \R^{n\times k}: X^\top X = I \} $, and estimate the gradient of $f$ at $x$ by 
\begin{align} 
    \wh{\nabla}  f_k^\delta (x) := \frac{n}{ 2 \delta k} \sum_{ i=1 }^k \( f (x + \delta v_i) - f (x - \delta v_i) \) v_i , \label{eq:def-grad-est} 
\end{align} 
where $\delta$ is the finite difference granularity.
When $k=1$, sampling is over the unit sphere and (Eq. \ref{eq:def-grad-est}) reduces to the estimator introduced by \citet{flaxman2005online}.

We show that the variance of (Eq. \ref{eq:def-grad-est}) for a $ (3,L_3) $-smooth (See Definition \ref{def:smooth}) function $f$ satisfies 
\begin{align*} 
    &\; \E \[ \left\| \wh{\nabla} f_k^\delta (x) - \E \[ \wh{\nabla}  f_k^\delta (x) \] \right\|^2 \] \\
    \le& \;  
    \( \frac{n}{k}  - 1 \) \| \nabla f (x) \|^2 + \frac{L_3 \delta^2}{3} \( \frac{n^2}{k}  - n \) \| \nabla f (x) \| + \frac{L_3^2 n^2 \delta^4 }{36 k } , \quad \forall x \in \R^n, 
\end{align*} 
where $\| \cdot \|$ is the Euclidean norm. 
When $k = n$, the variance of the estimator becomes negligibly small. 


For estimating the Hessian of a function $f$ in $ \R^n $, we propose to independently uniformly sample two matrices $ [ v_1, v_2, \cdots, v_k ] $ and $ [ w_1, w_2, \cdots, w_k ] $ from the Stiefel manifold $ \text{St} (n,k) $ and estimate the Hessian of $f$ at $x$ by 
\begin{align} 
    \wh{\H}f_k^\delta (x) 
    :=&\; 
    \frac{n^2}{ 8 \delta^2 k^2 } \sum_{ i,j=1 }^k \( f (x + \delta v_i + \delta w_j ) - f (x - \delta v_i + \delta w_j) - f (x + \delta v_i - \delta w_j) + f (x - \delta v_i - \delta w_j) \) \nonumber \\ 
    &\qquad \qquad  \cdot ( v_i w_j^\top + w_j v_i^\top ) ,  \label{eq:def-hess-est} 
\end{align}
where $\delta$ is the finite difference step size. When $k=1$, the sampling is over the unit sphere and (Eq. \ref{eq:def-hess-est}) reduces to the one introduced by the second author \citep{wang2022hess}. 

The variance of (Eq. \ref{eq:def-hess-est}) for a $ (4,L_4) $-smooth and $ (6,L_6) $-smooth (See Definition \ref{def:smooth}) function $f$ satisfies, for all $x \in \R^n$, 
\begin{align*} 
    &\; \E \[ \left\| \wh{\H} f_k^\delta (x) - \E \[ \wh{\H} f_k^\delta (x) \] \right\|_F^2 \] \\
        \le& \; 
        \left\| \nabla^2 f (x) \right\|_F^2 \( \frac{n^2}{k^2} - 1 \) + 2 \delta^2 L_4 \left\| \nabla^2 f (x) \right\| \( \frac{n^4}{k^2} - n^2 \) + \mathcal{O} \( \( L_6 n^2 \| \nabla^2 f (0) \|  + \frac{n^4 L_4^2}{k^2} \) \delta^4 \) , 
\end{align*} 
where $ \| \cdot \|_F $ is the Frobenius norm, and $\| \cdot \|$ is the spectral norm. 
Similar to the gradient case, the variance of (Eq. \ref{eq:def-hess-est}) becomes negligibly small when $k = n$. 

In addition, the above estimators do not sacrifice any bias accuracy. The bias of (Eq. \ref{eq:def-grad-est}) for sufficiently smooth $f$ at $x$ is of order 
\begin{align*}
    \mathcal{O} \(  \delta^2 \Gamma_3 (x) \),
\end{align*} 
    where $ \Gamma_3 (x) $ depends on the third order derivatives of $f$ at $x$.
Similar results hold for the Hessian estimator. The bias of (Eq. \ref{eq:def-hess-est}) for sufficiently smooth $f$ at $ x $ is of order 
\begin{align*} 
    \mathcal{O} \( \delta^2 \Gamma_4 (x) \) ,
\end{align*}
where $ \Gamma_4 (x) $ depends on the fourth order derivatives of $f$ at $x$. These refined bias bounds improve best previous results on bias of the estimators \citep[][]{flaxman2005online,wang2022hess}. 

\begin{remark} 
    The bias bound of the Hessian estimator depends on the fourth-order total derivative of the function, which is a 4-linear form (or a $(4, 0)$-tensor). More specifically, this estimation bias depends on a special norm of the fourth-order total derivative. See more discussions after Theorem \ref{thm:hess-bias}. 
\end{remark} 

\noindent\textbf{Theory Meets Practice} 

In practice, the observed errors are highly aligned with our theoretical bounds. 
The expected error of the gradient estimator $ \wh{\nabla}  f_k^\delta (x) $ can be bounded by 
\begin{align}
    \E \[ \underbrace{\left\| \wh{\nabla}  f_k^\delta (x) - \nabla f (x) \right\|}_{\text{``error''}} \] \nonumber
    \le& \;  
    \E \[ \left\| \wh{\nabla}  f_k^\delta (x) - \E \[ \wh{\nabla}  f_k^\delta (x) \] \right\| \] + \E \[ \left\| \E \[ \wh{\nabla}  f_k^\delta (x) \] - \nabla f (x) \right\| \] \\ 
    \le& \; 
    \sqrt{ \underbrace{ \E \[ \left\| \wh{\nabla}  f_k^\delta (x) - \E \[ \wh{\nabla}  f_k^\delta (x) \] \right\|^2 \]}_{\text{``variance''}}} + \underbrace{\left\| \E \[ \wh{\nabla}  f_k^\delta (x) \] - \nabla f (x) \right\|}_{\text{``bias''}} , \label{eq:error-decomp}
\end{align} 
which implies that the variance is critical in bridging the theoretical bias bounds and the practical performance. 
In fact, the variance bound is highly aligned with the empirical error, as illustrated in Figure \ref{fig:intro}. 

\begin{figure}[h!]
    \centering 
    \captionsetup{singlelinecheck=off}
    \includegraphics[scale = 0.8]{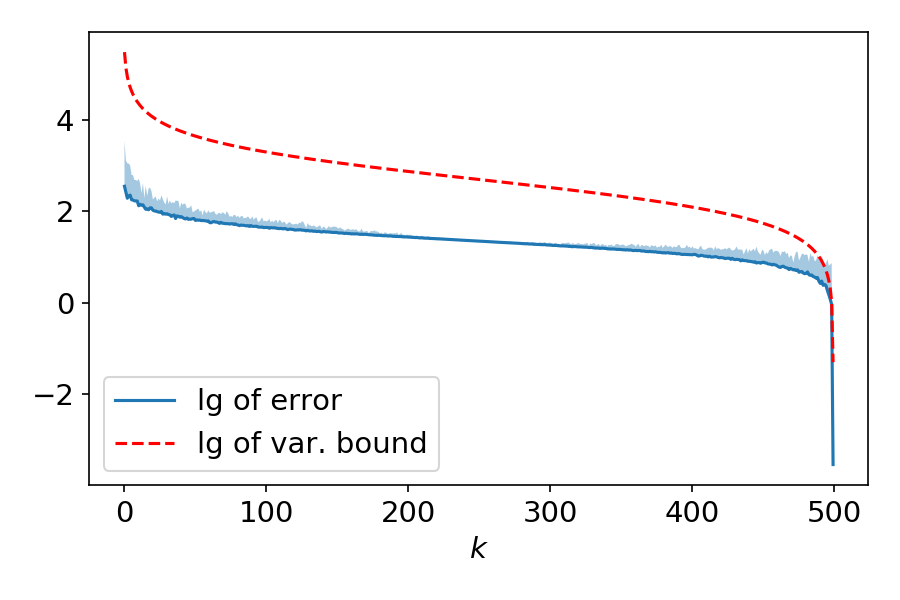} 
    \caption{ 
    Errors (Eq. \ref{eq:error-decomp}) of gradient estimators $ \wh{\nabla} f_k^\delta (x) $ with $k$ ranging from $1$ to $n$, in base-10 log-scale. Here $\delta = 0.1$ and $n = 500$. The underlying test function is $ f (x) = \exp ( (x_1 - 1) (x_2 + 2)) + \sum_{j=1}^{500} \sin (x_j)  $, where $x_j$ denotes the $j$-th component of vector $x$. The gradient is estimated at $x = 0$. The solid blue curve plots the  errors of gradient estimator in logarithmic scale, and is averaged over 10 runs. The shaded area above the solid curve shows 10 times standard deviation of the errors in logarithmic scale. The dashed red curve, as a function of $k$, is $ c (k) = \lg \( \| \nabla f (x) \|^2 \( \frac{n}{k} - 1 \) + \delta^2 \( \frac{n^2}{k} - n \) \| \nabla f (x) \| + \frac{ \delta^4 n^2 }{ k } \)$, which is the base-10 log of of variance bound for the gradient estimators (up to constants). The shapes of the two curves are highly aligned. More details are in Section \ref{sec:exp}. \label{fig:intro} 
    } 
\end{figure}

\noindent\textbf{``Better-Than-Definition'' Accuracy} 

In most numerical analysis textbooks, the default finite-difference gradient/Hessian estimator is the entry-wise estimator: We perform a 1-dimensional finite difference estimation for each entry of the gradient/Hessian, and gather all entries to output a gradient/Hessian estimator. Often times this method is considered the ``definition'' for the task of zeroth order gradient/Hessian estimation. 

As one would naturally expect, it is hard, if possible, to outperform this ``definition'' in an environment where $(i)$ one can sample as many zeroth-order function evaluations as she wants, and $(ii)$ all function evaluations are noise-free. Surprisingly, when $k=n$, our estimators (Eq. \ref{eq:def-grad-est}) and (Eq. \ref{eq:def-hess-est}) \emph{can outperform} the entry-wise estimators (the ``definition''). Some numerical comparisons between our estimators and the entry-wise estimators are in Table \ref{tab:intro}, and more details can be found in Section \ref{sec:exp}. 


Note that this observation is not in conflict with previous works \citep{flaxman2005online,wang2021GW,nesterov2017random,balasubramanian2021zeroth,wang2022hess}, since they focus on scenarios where either $(i)$ one has to estimate the gradient/Hessian with number of samples much smaller than the dimensionality of the space \citep[][]{flaxman2005online,wang2021GW}, or $(ii)$ there is noise in the zeroth-order function evaluations \citep[e.g.,][]{balasubramanian2021zeroth,wang2022hess}.



\begin{table}
        \centering
        \caption{The errors of gradient estimators listed against finite-difference granularity $\delta$. The first row shows $\delta$; The second row shows errors of $ \wh{\nabla}  f_n^\delta (x) $ ($n = 500$ is the dimension); The third row shows errors of the entry-wise estimator (More details in Section \ref{sec:exp}). 
        The error of an estimator is its distance to the true gradient (in Euclidean norm). Errors of $ \wh{\nabla}  f_n^\delta (x) $ are in average $\pm$ standard derivation format, where each average and standard deviation gather information from 10 runs. The entry-wise estimator is not random and its error is computed in a single run. The underlying test function is the same as the test function in Figure \ref{fig:intro}, and $x = 0$ is used for all evaluations.  \label{tab:intro}} 
        \centering 
        \begin{tabular}{|c|| c c c|}  
            \hline 
            $\delta$ & $0.1$ & $0.01$ & $0.001$  \\  \hline 
            \makecell{Stiefel sampling \\ errors} & 2.8e-4$\pm$4.0e-6 & 2.8e-6$\pm$1.0e-07 & 2.9e-08$\pm$6.4e-10  \\ \hline 
            \makecell{Entry-wise \\ errors} & 3.8e-2 & 3.7e-4 & 3.7e-6 \\ \hline 
        \end{tabular} 
\end{table}



\subsection*{Related Works} 

Zeroth order optimization is a central topic in many fields \citep[e.g.,][]{nelder1965simplex,goldberg1988genetic,conn2009introduction,nemirovski2009robust,shahriari2015taking}. Among many zeroth order optimization mechanisms, a classic and prosperous line of works focuses on estimating higher order derivatives using zeroth order information (See \citep{liu2020primer} for a recent survey). 

Previously, \citet{flaxman2005online} studied the stochastic gradient estimator using a single-point function evaluation for the purpose of bandit learning. \citet{duchi2015optimal} studied stabilization of the stochastic gradient estimator via two-points (or multi-points) evaluations.  \citet{nesterov2017random,balasubramanian2021zeroth} studied gradient/Hessian estimators using Gaussian smoothing, and investigated downstream optimization methods using the estimated gradient. In particular, \citet{balasubramanian2021zeroth} studied the zeroth order Hessian estimators via the Stein's identity \citep{10.1214/aos/1176345632} and applied the estimator to cubic regularized Newton's method \citep{nesterov2006cubic}. 
Also, zeroth order optimization via finite difference method along canonical coordinates have also been studied \citep{kiefer1952stochastic,spall1998overview}. More recently, zeroth order optimization algorithms using estimators via Rademacher random vectors are studied, especially when sparsity or compressibility conditions are imposed \citep{wang2018stochastic,doi:10.1137/21M1392966}. 
In addition to the above mentioned works, comparison-based gradient estimator has also been considered by \citet{CAI2022242}, which follows from a rich line of works in information theory \citep[e.g.,][]{raginsky2011information,jamieson2012query,plan2012robust,plan2014dimension}. 

Perhaps the most relevant works are \citep{flaxman2005online} for gradient estimators, and \citep{wang2022hess} for Hessian estimators. As for gradient estimators, (Eq. \ref{eq:def-grad-est}) includes the estimator by \citet{flaxman2005online} as a special case when $k = 1$. We show that variance of the gradient estimator can be significantly reduced as we increase $k$. Also, we provide an $\mathcal{O} (\delta^2)$ bias bound for the gradient estimator, which improves the $\mathcal{O} (\delta) $ bias bound by \citep{flaxman2005online}. 
We also provide improved bias bounds for the Hessian estimators. When the underlying space is Euclidean, the results in this paper are finer than those in \citep{wang2022hess}. 

%% file: prelim.tex
\section{Preliminaries} 

We list here some preliminaries, assumptions, and notations before proceeding to subsequent sections. Throughout the paper, we restrict our attention to real-valued functions defined over $ \R^n $. The letter $n$ is reserved for the dimension of the space, unless otherwise specified. Also, $ \| \cdot \| $ is reserved for the Euclidean norm when applied to vectors, and reserved for the spectral norm when applied to matrices (or symmetric tensors). For a function that is $m$-times continuously differentiable, let $\partial^m f $ denote the bundle of the total derivatives of $f$. More specifically, for any $x \in \R^n$, $ \partial^m f (x) $ is an $m$-order symmetric multi-linear form (a tensor). In particular, $ \partial^1 f = \nabla f $ and $\partial^2 f = \nabla^2 f$. For this multi-linear form $\partial^m f (x)$, which maps $m$ vectors in $\R^n$ to $\R$, we write $ \partial^m f (x) [v] $ as a shorthand for $ \partial^m f (x) [\underbrace{v, v, \cdots, v}_{m \text{ times}}] $. 

For different tasks, we make different assumptions about the smoothness of the function, which is described by the following $ (p,L) $-smoothness terminology. 

\begin{definition} 
    \label{def:smooth} 
    A function $f : \R^n \rightarrow \R$ is called $ (p,L) $-smooth ($p \in \mathbb{N}_+$, $L > 0$) if it is $p$-times continuously differentiable, and 
    \begin{align*} 
        \| \partial^{p-1} f (x) - \partial^{p-1} f (x') \| \le L \| x - x' \|, \quad \forall x,x' \in \R^n, 
    \end{align*}  
    where $\partial^{p-1} f (x) $ is the $(p-1)$-th total derivative of $f$ at $x$, and $ \| \cdot \| $ is the spectral norm when applied to symmetric multi-linear forms and the Euclidean norm when applied to vectors. The spectral norm of a symmetric multi-linear form $F$ is $ \| F \| : = \sup_{ v\in\R^n, \| v \| = 1 } F[v] $. 
\end{definition} 

Throughout, $ \| \cdot \| $ is reserved for the spectral norm when applied to symmetric multi-linear forms (including symmetric matrices). 
Although the smoothness quantification in Definition \ref{def:smooth} is global, a local version can be similarly defined, and all subsequent results can be obtained, using similar arguments. All $(p,L)$-smooth functions satisfy the following proposition. 

\begin{proposition}
    \label{prop:smooth}
    If $f$ is $(p,L)$-smooth, then 
    \begin{align*}
        \| \partial^{p} f (x) \| \le  L , \quad \forall x \in \R^n. 
    \end{align*} 
\end{proposition} 

\begin{proof}
    See Appendix. 
\end{proof}

We use $ \S^{n-1} $ (resp. $\B^n$) to denote the unit sphere (resp. ball) in $\R^n$. Also, given $ \delta > 0 $, $ \delta \S^{n-1} $ (resp. $\delta \B^n$) refers to the origin centered sphere (resp. ball) of radius $\delta$ in $\R^n$. Several useful identities are stated below in Propositions \ref{prop:proj}, Proposition \ref{prop:fourth}, and Proposition \ref{prop:mat}. References for the following propositions include \citep{nesterov2017random,wang2022hess,CAI2022242}. Their proofs are included in the appendix. 

\begin{proposition} 
    \label{prop:proj}
    Let $v$ be a vector uniformly randomly sampled from $ \S^{n-1} $. Then it holds that 
    \begin{align*}
        \E \[ v v^\top \] = \frac{1}{n} I, 
    \end{align*}
    where $I$ is the identity matrix (of size $n \times n$). 
\end{proposition} 

\begin{proof} 
    See Appendix. 
\end{proof} 

\begin{proposition}
    \label{prop:fourth}
    Let $v$ be a vector uniformly randomly sampled from $ \S^{n-1} $, and let $v_i$ be the $i$-th component of $v$. Then it holds that 
    \begin{itemize}
        \item $ \E \[ v_i^4 \] = \frac{3}{n^2 + 2n } $ for all $i = 1,2,\cdots, n$; 
        \item $ \E \[ v_i^2 v_j^2 \] = \frac{1}{n^2 + 2n } $ for all $i, j = 1,2,\cdots, n$ and $i \neq j$; 
        \item 
        $ \E \[ v_i v_j v_k v_l \] = 0 $ for all $i, j, k, l = 1,2,\cdots, n$ and $i \notin \{j,k,l\}$. 
    \end{itemize} 
\end{proposition}

\begin{proof} 
    See Appendix. 
\end{proof}

\begin{proposition}
    \label{prop:mat}
    Let $v,w$ be two independent vectors uniformly randomly sampled from $ \S^{n-1} $, and let $A$ be a symmetric matrix. 
    It holds that 
    \begin{align*}
        \E \[ \( v^\top A w \) v w^\top \] = \E \[ \( v^\top A w \) w v^\top \] = \frac{1}{n^2} A. 
    \end{align*}
\end{proposition}

\begin{proof} 
    See Appendix. 
\end{proof} 





%% file: grad.tex
    

\section{Gradient Estimation} 

Since both the gradient and Hessian estimators (Eq. \ref{eq:def-grad-est}) and (Eq. \ref{eq:def-hess-est}) use random directions sampled from the Stiefel manifold, we first describe the sampling process for generating such random directions. 
%
This sampling procedure is summarized in Algorithm \ref{alg:gram-schmidt}.

\begin{algorithm}[ht] 
	\caption{Stiefel Sampling $\texttt{S}(n,k)$ \label{alg:gram-schmidt}} 
	\begin{algorithmic}[1] 
		\STATE \textbf{Input}: Dimension $n$. Number of vectors $k$.
        \STATE Sample a random matrix $U \in \R^{n \times k}$ such that $  U_{i,j} \overset{i.i.d.}{\sim} \mathcal{N} (0,1) $.
        \STATE Let $ [v_1, v_2, \cdots, v_k] = U \( U^\top U \)^{-1/2} $. 
        \STATEx \quad /* With probability $ 1 $, $[ v_1, v_2, \cdots, v_k] $ is well-defined. */ 
		\STATE \textbf{Output:} $\texttt{S}(n,k) = v_1, v_2, \cdots, v_k$. 
	\end{algorithmic} 
\end{algorithm}



The marginal distribution for any vector from Algorithm \ref{alg:gram-schmidt} is a uniform distribution over $\S^{n-1}$, as summarized in Proposition \ref{prop:proj}. 


\begin{proposition}[\citet{chikuse2003mani}] 
    \label{prop:uniform} 
    Let $v_1, v_2, \cdots, v_k$ be vectors sampled from Algorithm \ref{alg:gram-schmidt}. The marginal distribution for any $v_i$ is uniform over $\S^{n-1}$. 
\end{proposition} 

Intuitively, Proposition \ref{prop:uniform} is due to the fact that the uniform measure over the Stiefel manifold $\text{St} (n,k)$ (the Hausdorff measure with respect to the Frobenius inner product of proper dimension, which is rotation-invariant) can be decomposed into a wedge product of the spherical measure over $ \S^{n-1} $ and the uniform measure over $ \text{St} (n-1,k-1) $ \citep{chikuse2003mani}. One can link the decomposition of measure to the following sampling process. We first sample $v_1$ uniformly from $\S^{n-1}$, then $v_2$ uniformly from $\S^{n-1} \cap \{ v_1 \}^\perp$ and so on. 
By symmetry, the marginal distribution for any $ v_i $ generated from Algorithm \ref{alg:gram-schmidt} is uniform over $\S^{n-1}$, as stated in Proposition \ref{prop:uniform}. 
We refer the readers to Chapters 1 \& 2 in  \citep{chikuse2003mani} for more details on distribution over Stiefel manifolds.

Using the vectors sampled from Algorithm \ref{alg:gram-schmidt}, we define the gradient and Hessian estimators (Eq. \ref{eq:def-grad-est}) and (Eq. \ref{eq:def-hess-est}). This sampling trick can significantly reduce variance.




\subsection{Variance of Gradient Estimator} 

Define $ u_i := \frac{n}{2 \delta } \( f (\delta v_i) - f (-\delta v_i) \) v_i \approx n v_i v_i^\top \nabla f (0) $, where the approximate equality follows from Taylor's theorem. If $[v_1, v_2, \cdots, v_k]$ is sampled from Algorithm \ref{alg:gram-schmidt}, $ \{u_1, u_2,\cdots, u_k \} $ are mutually perpendicular and we have 
\begin{align*} 
    &\; \frac{1}{k^2} \E \[ \left\| \sum_{ i=1 }^k u_i \right\|^2 \] - \frac{1}{k^2} \left\| \E \[ \sum_{ i=1 }^k u_i \] \right\|^2 \\ 
    =& \; 
    \frac{1}{k^2} \E \[ \sum_{ i,j: 1\le i,j \le k, i \neq j }  u_i^\top u_j \] + \frac{1}{k^2} \( \E \[ \sum_{ i=1 }^k \left\| u_i \right\|^2 \] - \left\| \E \[ \sum_{ i=1 }^k u_i \] \right\|^2 \)  \\
    \overset{\textcircled{1}}{=}& \; 
    \frac{1}{k^2} \( \E \[ \sum_{ i=1 }^k \left\| u_i \right\|^2 \] - \left\| \E \[ \sum_{ i=1 }^k u_i \] \right\|^2 \) \\
    \overset{\textcircled{2}}{\approx}& \; 
    \frac{1}{k^2} \( n^2 \E \[ \sum_{ i=1 }^k \nabla f (0)^\top v_i v_i^\top \nabla f (0) \] - \left\| n \E \[ \sum_{i=1}^k v_i v_i^\top \nabla f (0) \] \right\|^2 \) \\ 
    \overset{\textcircled{3}}{=}& \; 
    \( \frac{n}{k} - 1 \) \left\| \nabla f (0) \right\|^2 , 
\end{align*} 
where \textcircled{1} uses orthogonality of $u_i$ and $u_j$ for $i\neq j$, \textcircled{2} uses $ u_i \approx n v_i v_i^\top \nabla f (0) $, and \textcircled{3} uses Proposition \ref{prop:proj}. With a Taylor expansion with higher precision, we have the following theorem.

\begin{theorem} 
    \label{thm:grad-var}
    If $f$ is $(3,L_3)$-smooth, the variance of the gradient estimator for $f$ (Eq. \ref{eq:def-grad-est}) satisfies 
    \begin{align*} 
        &\; \E \[ \left\| \wh{\nabla}  f_k^\delta (x) - \E \[ \wh{\nabla}  f_k^\delta (x) \] \right\|^2 \] \\ 
        \le& \; 
        \( \frac{n}{k}  - 1 \) \| \nabla f (x) \|^2 + \frac{L_3 \delta^2}{3} \( \frac{n^2}{k}  - n \) \| \nabla f (x) \| + \frac{L_3^2 n^2 \delta^4 }{36 k } , \quad \forall x \in \R^n. 
    \end{align*} 
\end{theorem} 

\begin{proof} 
Without loss of generality, let $x = 0$. 
By Taylor expansion, we know that for any $v_i \in \S^{n-1}$ and small $\delta$, 
\begin{align*} 
    \frac{1}{2} \big( f (\delta v_i) - f (-\delta v_i) \big) 
    = 
    \delta v_i^\top \nabla f (0) + \frac{ \delta^{3} \( \partial^{3} f (z_i) [v_i] + \partial^{3} f (z_i') [v_i] \) }{ 12 }, 
\end{align*} 
where $z_i, z_i'$ depend on $v_i$ and $\delta$. 
For simplicity, let $ R_i = \frac{ \delta^{3} \( \partial^{3} f (z_i) [v_i] + \partial^{3} f (z_i') [v_i] \) }{ 12 } $, and $ | R_i | \le \frac{ L_3 \delta^3 }{ 6 } $ for all $i = 1,2,\cdots,k$. 


For any $i,k,n$, it holds that 
\begin{align*} 
    &\; \E \[ \left\| \frac{1}{2} \( f (\delta v_i) - f (-\delta v_i) \) v_i \right\|^2 \] - \left\| \E \[ \frac{ \sqrt{k}}{2} \( f (\delta v_i) - f (-\delta v_i) \) v_i \] \right\|^2 \\ 
    =& \; 
    \E \[ \left\| \( \delta v_i^\top \nabla f (0) + R_i \) v_i \right\|^2 \] 
    - 
    k \left\| \E \[ \( \delta v_i^\top \nabla f (0) + R_i \) v_i \]  \right\|^2 \\ 
    \overset{\textcircled{1}}{=}& \; 
    \E \[ \delta^2 \nabla f (0)^\top  v_i v_i^\top \nabla f (0) + 2 \delta \nabla f (0)^\top v_i R_i + R_i^2 \] 
    - 
    k \left\| \E \[ \delta v_i v_i^\top \nabla f (0) + R_i v_i \] \right\|^2 . 
\end{align*} 

Since $ \E \[ v_i v_i^\top \] = \frac{1}{n} I$, \textcircled{1} gives 
\begin{align*}
    & \; \E \[ \left\| \frac{1}{2} \( f (\delta v_i) - f (-\delta v_i) \) v_i \right\|^2 \] - \left\| \E \[ \frac{ \sqrt{k} }{2} \( f (\delta v_i) - f ( - \delta v_i) \) v_i \] \right\|^2 \\
    =& \; 
    \frac{ \delta^2 }{ n } \| \nabla f (0) \|^2 + 2 \delta \nabla f (0)^\top \E \[ R_i v_i \] + \E \[ R_i^2 \] 
    - 
    k \left\| \E \[ \frac{ \delta}{n} \nabla f (0) + R_i v_i \] \right\|^2 \\ 
    =& \; 
    \( \frac{\delta^2}{n}  - \frac{\delta^2 k }{n^2} \) \| \nabla f (0) \|^2 + \( 2 \delta  - \frac{2 \delta k }{n} \) \nabla f (0)^\top \E \[ R_i v_i \] + \( \E \[ R_i^2 \] - k \E \[ R_i v_i \]^\top \E \[ R_i v_i \] \) \\ 
    \overset{\textcircled{2}}{\le}& \; 
    \( \frac{\delta^2}{n}  - \frac{\delta^2 k }{n^2} \) \| \nabla f (0) \|^2 + \frac{L_3 \delta^4}{3} \( 1  - \frac{ k }{n} \) \| \nabla f (0) \| 
    +
    \frac{L_3^2 \delta^6 }{36} . 
\end{align*}

For the variance of the gradient estimator, we have 
\begin{align*} 
    &\; \E \[ \left\| \wh{\nabla} f_k^\delta (0) - \E \[ \wh{\nabla} f_k^\delta (0) \] \right\|^2 \] \\
    =& \; 
    \E \[ \left\| \wh{\nabla} f_k^\delta (0) \right\|^2 \] - \left\| \E \[ \wh{\nabla} f_k^\delta (0) \] \right\|^2 \\ 
    =& \; 
    \E \[ \left\| \frac{n}{2 \delta k} \sum_{i=1}^k \( f (\delta v_i) - f (\delta v_i) \) v_i \right\|^2 \] 
    - \left\| \frac{n}{2 \delta k} \sum_{i=1}^k \E \[ \( f (\delta v_i) - f (-\delta v_i) \) v_i \] \right\|^2 \\ 
    \overset{\textcircled{3}}{=}& \; 
    \frac{n^2}{\delta^2 k^2} \sum_{i=1}^k \E \[ \left\| \frac{1}{2} \( f (\delta v_i) - f (\delta v_i) \) v_i \right\|^2 \] \\
    &- \frac{ n^2 }{4 \delta^2 k^2 } \sum_{i,j=1}^k \E \[ \( f (\delta v_i) - f (-\delta v_i) \) v_i \]^\top \E \[ \( f (\delta v_j) - f (-\delta v_j) \) v_j \] , 
\end{align*} 
where the last equation follows from the orthonormality of $\{ v_1, v_2. \cdots, v_k \}$. By Proposition \ref{prop:uniform}, we know that 
$ \E \[ \( f (\delta v_i) - f (-\delta v_i) \) v_i \] = \E \[ \( f (\delta v_j) - f (-\delta v_j) \) v_j \]$ for all $i,j = 1,2,\cdots,k$. Thus \textcircled{3} gives
\begin{align*}
    & \; \E \[ \left\| \wh{\nabla} f_k^\delta (0) - \E \[ \wh{\nabla} f_k^\delta (0) \] \right\|^2 \] \\
    \overset{\textcircled{4}}{=}& \; 
    \frac{n^2}{\delta^2 k^2} \sum_{i=1}^k 
    \( \E \[ \left\| \frac{1}{2} \( f (\delta v_i) - f (-\delta v_i) \) v_i \right\|^2 \] - \left\| \E \[ \frac{ \sqrt{k} }{2} \( f (\delta v_i) - f (-\delta v_i) \) v_i \] \right\|^2 \).
\end{align*}

Combining \textcircled{2} and \textcircled{4} gives 
\begin{align*}
    & \; \E \[ \left\| \wh{\nabla}f_k^\delta (0) - \E \[ \wh{\nabla}f_k^\delta (0) \] \right\|^2 \] \\
    =& \; 
    \frac{n^2}{\delta^2 k^2}\sum_{i=1}^k \(  \E \[ \left\| \frac{1}{ 2 } \( f (\delta v_i) - f (\delta v_i) \) v_i \right\|^2 \] - \left\| \E \[ \frac{ \sqrt{k} }{2 } \( f (\delta v_i) - f (-\delta v_i) \) v_i \] \right\|^2 \) \\ 
    \le& \; 
    \( \frac{n}{k}  - 1 \) \| \nabla f (0) \|^2 + \frac{L_3 \delta^2}{3} \( \frac{n^2}{k}  - n \) \| \nabla f (0) \| + \frac{L_3^2 n^2 \delta^4 }{ 36 k } . 
\end{align*} 

\end{proof} 

\subsection{Bias of Gradient Estimator} 

While the estimator $\wh{\nabla}  f_k^\delta$ can significantly reduce variance, it does not sacrifice any bias accuracy. There has been a sequence of works on bias of gradient estimators of (Eq. \ref{eq:def-grad-est}) or alternatives of (Eq. \ref{eq:def-grad-est}) \citep{flaxman2005online,nesterov2017random,wang2021GW}. In this section, we provide a refined analysis on the bias, and present a bias bound of order $ \mathcal{O} \( \min \{ \delta, \delta^2 \} \)  $. 

\begin{theorem} 
    \label{thm:grad-bias}
    The gradient estimator $ \wh{\nabla}  f_k^\delta $ satisfies
    \begin{enumerate}[label=(\alph*)]
        \item If $ f $ is $(2, L_1)$-smooth, then for all $ x \in \R^n $, $ \left\| \E \[ \wh{\nabla}  f_k^\delta (x)  \] - \nabla f (x) \right\| \le \frac{L_1 n  \delta}{n+1} $.  
        \item If $f$ is $(4,L_4)$-smooth, then for all $x\in\mathbb{R}^n$, the bias of gradient estimator satisfies
            $$\left\|\E \[\wh{\nabla}  f_k^\delta (x)\]-\nabla f(x)\right\|\leq\frac{\delta^2}{2n}\sqrt{\sum_{i=1}^n\left(\sum_{j=1}^nF_{jji}\right)^2} + \frac{\delta^3 L_4 n}{24},$$
        where $F_{ijl}$ denotes the $(i,j,l)$-component of $\partial^3 f(x)$.
    \end{enumerate}
\end{theorem} 

Item (a) in Theorem \ref{thm:grad-bias} is due to \citet{flaxman2005online}. This fact can be viewed as a consequence of the Stokes' theorem, or the fundamental theorem of (geometric) calculus, or the divergence theorem. A proof for item (a) using divergence theorem is in Appendix for completeness. 

Item (b) is a refined analysis of Taylor expansion and spherical random projection properties. More specifically, we expand $f (x+ \delta v_i) v_i$ up to fourth order and repeatedly exploit Propositions \ref{prop:proj} and \ref{prop:fourth}. Below we provide a detailed proof for item (b). 




\begin{proof}[Proof of Theorem \ref{thm:grad-bias}(b)]
Without loss of generality, let $x=0$. By Proposition \ref{prop:uniform}, we know that $\E [\wh{\nabla}  f_k^\delta (x)] = \frac{n}{\delta}\E [f(\delta v)v]$, where $v$ is uniformly sampled from $\mathbb{S}^{n-1}$. Taylor expansion gives that
$$f(\delta v) = f(0) + \nabla f(0)^\tr(\delta v) + \frac{1}{2}(\delta v)^\tr\partial^2f(0)[\delta v] +\frac{1}{6}\partial^3f(0)[\delta v]+\frac{1}{24}\partial^4f(\xi(\delta v))[\delta v],$$
where we use notation $\xi(\delta v)$ to indicate that $\xi$ is a function of $\delta v$.

Therefore, the expectation can be written as
\begin{align*} 
    \E[f(\delta v)v] 
    \overset{\textcircled{1}}{=}& \; 
    \E \[ f(0)v \] + \delta \E \[ v v^\tr \nabla f(0) \] + \frac{\delta^2}{2} \E \[ \left(v^\tr\nabla^2f(0)v\right)v \] \\
    &+ \frac{\delta^3}{6} \E \[ \partial^3 f(0)[v] v \] + \frac{\delta^4}{24} \E \[ \partial^4f(\xi(\delta v))[v] v \] . 
\end{align*} 
The first and the third term in \textcircled{1} are zero, since the expectation is with respect to a uniform distribution over the unit sphere and the integrands are odd. Thus we have 
\begin{align*}
    \E \[ f(0)v \] = \E \[ \left(v^\tr\nabla^2f(0)v\right)v \] \overset{\textcircled{2}}{=} 0
\end{align*} 
For the second term in \textcircled{1}, 
Proposition \ref{prop:uniform} gives 
\begin{align*}
    \E \[ v v^\tr \nabla f(0) \] \overset{\textcircled{3}}{=} \frac{1}{n} \nabla f (0). 
\end{align*} 
For the forth term, we denote $\partial^3(f(0))$ as $F$ for simplicity. Then the $i$-th component of $\partial^3f(0)[v]v$ is $\sum_{j,l,p}F_{jlp}v_jv_lv_pv_i$. 
By Proposition \ref{prop:fourth}, we have 
\begin{align*}
    \( \E \[ \partial^3f(0)[v]v \] \)_i
    =&\;
    \sum_{j,l,p}F_{jlp} \E \[ v_j v_l v_p v_i \]\\
    =&\; 
    \sum_{j\neq i}\frac{1}{n^2+2n}\left(F_{jji}+F_{jij}+F_{ijj}\right)+\frac{3}{n^2+2n}F_{iii} \\
    =& \; 
    \frac{3}{n^2+2n}\sum_{j=1}^nF_{jji}. 
\end{align*}
Thus it holds that 
\begin{align*}
    \frac{\delta^3}{6}
    \E \[ \partial^3 f(0)[ v]v \] \overset{\textcircled{4}}{=} 
    \frac{\delta^3}{2n^2+4n}\sqrt{\sum_{i=1}^n\left(\sum_{j=1}^nF_{jji}\right)^2}.
\end{align*}

For the fifth term, Proposition \ref{prop:smooth} gives that $\|\partial^4f(x)\|\leq L_4$, so we have
\begin{align*} 
    \left\|\frac{\delta^4}{24} \E \[ \partial^4f(\xi(\delta v))[\delta v]v \] \right\| 
    \overset{\textcircled{5}}{\leq} 
    \frac{\delta^4 L_4}{24}.
\end{align*} 
Now combining \textcircled{2} \textcircled{3} \textcircled{4} \textcircled{5},  
we arrive at the upper bound
\begin{align*}
\left\|\E[f(\delta v)v]-\frac{\delta}{n}\nabla f(0)\right\|
\leq \frac{\delta^3}{2n^2+4n}\sqrt{\sum_{i=1}^n\left(\sum_{j=1}^n F_{jji}\right)^2} + \frac{\delta^4 L_4 }{24}
\end{align*}
and
$$\left\|\frac{n}{\delta}\E[f(\delta v)v]-\nabla f(0)\right\|\leq\frac{\delta^2}{2n}\sqrt{\sum_{i=1}^n\left(\sum_{j=1}^nF_{jji}\right)^2} + \frac{\delta^3 L_4 n}{24}.$$
\end{proof}

%% file: hessian.tex
\section{Hessian Estimation}


Previously, \citet{wang2022hess} introduced the Hessian estimator (Eq. \ref{eq:def-hess-est}) with $k = 1$, over Riemannian manifolds. Similar to the gradient case, when we sample orthogonal frames for estimation, the variance can be reduced. 
As previously discussed, the variance of $\wh{\H} f_k^\delta (x)$ for a smooth $f$ defined in $\R^n$ is of order 
\begin{align*} 
    \mathcal{O} \( \( \frac{n^2}{k^2} - 1 \) + \( \frac{n^4}{k^2} - k^2 \) \delta^2 + \frac{n^4 \delta^4 }{k^2} \) . 
\end{align*} 
Similar to its gradient estimator counterpart, variance of $\wh{\H} f_k^\delta (x)$ eventually goes to $ \mathcal{O} \( \delta^4 \) $ when $k = n$. However, the task of Hessian estimation is harder, since: 
\begin{enumerate}
    \item When $k$ is small compare to $n$, the variance is of order $\mathcal{O} \( n^2 \)$, which is worse than that for the gradient estimator. 
    \item The estimator $ \wh{\H} f_k^\delta (x) $ requires $ \Theta (k^2) $ samples, which means it takes $ \Theta (n^2) $ samples to reach a negligible variance. 
\end{enumerate} 




\subsection{Variance of Hessian Estimator} 

Similar to that for gradient estimators, the variance of Hessian estimator is bounded via high-order Taylor expansion and random projection arguments. A high-precision bound on the variance is in Theorem \ref{thm:hess-var}. 

\begin{theorem}
    \label{thm:hess-var}
    If the underlying function is $(4, L_4)$-smooth and $(6, L_6)$-smooth, then the Hessian estimator $ \wh{\H} f_k^\delta (x) $ (Eq. \ref{eq:def-hess-est}) with $\delta > 0$ satisfies, for all $ x \in \R^n$
    \begin{align*}
        &\; \E \[ \left\| \wh{\H} f_k^\delta (x) - \E \[ \wh{\H} f_k^\delta (x) \] \right\|_F^2 \] \\
        \le& \; 
        \left\| \nabla^2 f (x) \right\|_F^2 \( \frac{n^2}{k^2} - 1 \) + 2 \delta^2 L_4 \left\| \nabla^2 f (x) \right\| \( \frac{n^4}{k^2} - n^2 \) + \mathcal{O} \( \( L_6 n^2 \| \nabla^2 f (x) \| + \frac{n^4 L_4^2 }{k^2} \) \delta^4 \) . 
    \end{align*} 
\end{theorem}


\begin{proof} 
Without loss of generality, let $x = 0$. 
Using Taylor series expansion, for any $v_i,w_j$, we have 
\begin{align*} 
    &\; \frac{1}{4} \( f (\delta v_i + \delta w_j ) - f (\delta v_i - \delta w_j ) - f ( - \delta v_i + \delta w_j ) + f ( - \delta v_i - \delta w_j ) \) \\
    =& \;  
    \delta^2 v_i^\top \nabla^2 f (0) w_j  + \delta^4 R_{ij} + \delta^6 S_{ij}, 
\end{align*} 
where 
\begin{align*} 
    R_{ij} &= \frac{1}{96} \( \partial^4 f (0) [v_i + w_j] - \partial^4 f (0) [-v_i + w_j] - \partial^4 f (0) [v_i - w_j] + \partial^4 f (0) [-v_i - w_j] \) , \\
    S_{ij} &= \frac{1}{2880} \( \partial^6 f (z_1) [v_i + w_j] - \partial^6 f (z_2) [-v_i + w_j] - \partial^6 f (z_3) [v_i - w_j] + \partial^6 f (z_4) [-v_i - w_j] \) 
\end{align*} 
with $ z_1, z_2, z_3, z_4 $ depending on $v_i,w_j$ and $\delta$. Since $f$ is $(4,L_4)$-smooth and $ (6,L_6) $-smooth, $ | R_{ij} | \le \frac{ 2 L_4 }{ 3 } $ and $ | S_{ij} | \le \frac{ 4 L_6 }{ 45 } $.


For the Frobenius norm of the Hessian estimator, we have 
\begin{align*} 
    \left\| \wh{\H} f_k^\delta (0) \right\|_F^2 
    =& \; 
    \frac{ n^4 }{ 4 k^4 \delta^4 } \left\| \sum_{i,j=1}^k \( \frac{\delta^2}{2} v_i^\top \nabla^2 f (0) w_j + \frac{\delta^2}{2} w_j^\top \nabla^2 f (0) v_i + \delta^4 R_{ij} + \delta^6 S_{ij} \) (v_i w_j^\top + w_j v_i^\top ) \right\|_F^2 \\ 
    \le& \; 
    \frac{ n^4 }{ 2 k^4 } \left\| \sum_{i,j=1}^k \(  \frac{1}{2} v_i^\top \nabla^2 f (0) w_j + \frac{1}{2} w_j^\top \nabla^2 f (0) v_i + \delta^2 R_{ij} + \delta^4 S_{ij} \) v_i w_j^\top \right\|_F^2 \\
    &+ \frac{ n^4 }{ 2 k^4 } \left\| \sum_{i,j=1}^k \( \frac{1}{2} v_i^\top \nabla^2 f (0) w_j + \frac{1}{2} w_j^\top \nabla^2 f (0) v_i + \delta^2 R_{ij} + \delta^4 S_{ij} \) w_j v_i^\top \right\|_F^2 \\
    \overset{\textcircled{1}}{\le}&\; 
    \frac{ n^4 }{ 4 k^4 } \left\| \sum_{i,j=1}^k \( v_i^\top \nabla^2 f (0) w_j + \delta^2 R_{ij} + \delta^4 S_{ij} \) v_i w_j^\top \right\|_F^2 \\ 
    &+ \frac{ n^4 }{ 4 k^4 } \left\| \sum_{i,j=1}^k \( w_j^\top \nabla^2 f (0) v_i + \delta^2 R_{ij} + \delta^4 S_{ij} \) v_i w_j^\top \right\|_F^2 \\
    &+ \frac{ n^4 }{ 4 k^4 } \left\| \sum_{i,j=1}^k \( v_i^\top \nabla^2 f (0) w_j + \delta^2 R_{ij} + \delta^4 S_{ij} \) w_j v_i^\top \right\|_F^2 \\
    &+ \frac{ n^4 }{ 4 k^4 } \left\| \sum_{i,j=1}^k \( w_j^\top \nabla^2 f (0) v_i + \delta^2 R_{ij} + \delta^4 S_{ij} \) w_j v_i^\top \right\|_F^2 , 
\end{align*} 
where both inequalities use $ \| A + B \|_F^2 \le 2 \| A \|_F^2 + 2 \| B \|_F^2 $ for any matrices $A,B$ of same size. Next, we will bound the expectation of the first term in \textcircled{1}, all other terms can be bounded using similar arguments. 


For the first term in \textcircled{1}, we have 
\begin{align*} 
    &\; \left\| \sum_{i,j=1}^k \( v_i^\top \nabla^2 f (0) w_j + \delta^2 R_{ij} + \delta^4 S_{ij} \) v_i w_j^\top \right\|_F^2 \\
    =&\; 
    tr \( \( \sum_{i,j=1}^k \( v_i^\top \nabla^2 f (0) w_j + \delta^2 R_{ij} + \delta^4 S_{ij} \) v_i w_j^\top \)^\top \( \sum_{a,b=1}^k \( v_a^\top \nabla^2 f (0) w_b + \delta^2 R_{ab} + \delta^4 S_{ab} \) v_a w_b^\top \) \) \\ 
    =&\; 
    \sum_{i,j=1}^k \( v_i^\top \nabla^2 f (0) w_j + \delta^2 R_{ij} + \delta^4 S_{ij} \)^2 \\ 
    \overset{\textcircled{2}}{=}& \; 
    \sum_{i,j=1}^k \( v_i^\top \nabla^2 f (0) w_j \)^2 + \sum_{i,j=1}^k 2 \( v_i^\top \nabla^2 f (0) w_j \) \(  \delta^2 R_{ij} + \delta^4 S_{ij} \) + \sum_{i,j=1}^k \( \delta^2 R_{ij} + \delta^4 S_{ij} \)^2 , 
\end{align*} 
where the second last line uses orthogonality of $ v_i$'s and $w_j$'s. 

Since $ \E \[ \( v_i^\top \nabla^2 f (0) w_j \)^2 \] = \E \[ tr \( v_i v_i^\top \nabla^2 f (0) w_j w_j^\top \nabla^2 f (0) \) \] = \frac{1}{n^2} \| \nabla^2 f (0) \|_F^2 $, taking expectation on both sides of \textcircled{2} gives, 
\begin{align*}
    &\; \E \[ \left\| \sum_{i,j=1}^k \( v_i^\top \nabla^2 f (0) w_j + \delta^2 R_{ij} + \delta^4 S_{ij} \) v_i w_j^\top \right\|_F^2 \] \\ 
    \le&\; 
    \frac{k^2}{n^2} \left\| \nabla^2 f (0) \right\|_F^2 + 2 \delta^2 \sum_{i,j=1}^k \E \[ v_i^\top \nabla^2 f (0) w_j R_{ij} \] + 2 k^2 \left\| \nabla^2 f (0) \right\| \frac{ 4 L_6 \delta^4 }{ 45 } + k^2 \( \frac{ 2 L_4 \delta^2}{3} + \frac{ 4 L_6 \delta^4}{45 } \)^2 . 
\end{align*} 




Also, we have 
\begin{align*} 
    &\; \left\| \E \[ \wh{\H} f_k^\delta (0) \] \right\|_F^2 \\
    =&\; \frac{n^4}{4 k^4} \left\| \E \[ \sum_{i,j = 1}^k \( \frac{ 1 }{2} v_i^\top \nabla^2 f (0) w_j + \frac{ 1 }{2} w_j^\top \nabla^2 f (0) v_i + \delta^2 R_{ij} + \delta^4 S_{ij} \) (v_i w_j^\top + w_j v_i^\top ) \]  \right\|_F^2 \\ 
    \overset{\textcircled{3}}{=}& \; 
    \frac{n^4}{4 k^4} \left\| \frac{ 2 k^2}{n^2} \nabla^2 f (0) + \E \[ \sum_{i,j = 1}^k \( \delta^2 R_{ij} + \delta^4 S_{ij} \) (v_i w_j^\top + w_j v_i^\top ) \]  \right\|_F^2 \\ 
    \ge& \; 
    \left\| \nabla^2 f (0) \right\|_F^2 + 2 \frac{n^4}{4 k^4} \cdot \frac{2 k^2}{n^2} tr \( \nabla^2 f (0) \E \[ \sum_{i,j=1}^k \( \delta^2 R_{ij} + \delta^4 S_{ij} \) (v_i w_j^\top + w_j v_i^\top ) \] \) \\
    \ge& \; 
    \left\| \nabla^2 f (0) \right\|_F^2 + \frac{2n^2 \delta^2}{k^2} \sum_{i,j=1}^k \E \[ v_i^\top \nabla^2 f (0) w_j R_{ij}  \] - 2 n^2 \left\| \nabla^2 f (0) \right\| \frac{ 4 L_6 \delta^4 }{45} , 
\end{align*} 
where \textcircled{3} uses Proposition \ref{prop:mat}. 

Collecting terms gives 
\begin{align*}
    &\; \left\| \wh{\H} f_k^\delta (0) \right\|_F^2 - \left\| \E \[ \wh{\H} f_k^\delta (0) \] \right\|_F^2 \\
    \le& \;  
    \frac{n^4}{k^4} \( \frac{k^2}{n^2} \left\| \nabla^2 f (0) \right\|_F^2 + 2 \delta^2 \sum_{i,j=1}^k \E \[ v_i^\top \nabla^2 f (0) w_j R_{ij} \] + 2 k^2 \left\| \nabla^2 f (0) \right\| \frac{4 L_6 \delta^4 }{45} + k^2 \( \frac{2 L_4 \delta^2}{3} + \frac{4 L_6 \delta^4}{45} \)^2 \) \\ 
    &- \left\| \nabla^2 f (0) \right\|_F^2 - \frac{2n^2 \delta^2}{k^2} \sum_{i,j=1}^k \E \[ v_i^\top \nabla^2 f (0) w_j R_{ij} \] + 2 n^2 \left\| \nabla^2 f (0) \right\| \frac{4 L_6 \delta^4 }{45} \\ 
    \le& \; 
    \left\| \nabla^2 f (0) \right\|_F^2 \( \frac{n^2}{k^2} - 1 \) + 2 \delta^2 \sum_{i,j=1}^k \E \[ v_i^\top \nabla^2 f (0) w_j R_{ij} \] \( \frac{n^4}{k^4} - \frac{n^2}{k^2} \) + \O \( \( L_6 n^2 \| \nabla^2 f (0) \|  + \frac{n^4 L_4^2}{k^2} \) \delta^4 \) \\
    \le& \; 
    \left\| \nabla^2 f (0) \right\|_F^2 \( \frac{n^2}{k^2} - 1 \) + 2 \delta^2 L_4 \left\| \nabla^2 f (0) \right\| \( \frac{n^4}{k^2} - n^2 \) + \O \( \( L_6 n^2 \| \nabla^2 f (0) \|  + \frac{n^4 L_4^2}{k^2} \) \delta^4 \) . 
\end{align*} 


\end{proof}

\subsection{Bias of Hessian Estimator} 

Similar to the gradient case, the Hessian estimator $ \wh{\H} f_k^\delta (x)$ does not sacrifice any bias accuracy. Previously, \citet{wang2022hess} showed that the bias of $ \wh{\H} f_k^\delta (x) $ is of order $\mathcal{O} (\delta)$. In particular, 
\citet{wang2022hess} derived a formula for how the local geometry of the Riemannian manifold would affect the bias of the Hessian estimator. 
In this paper, we focus on providing refined bias bounds in the Euclidean case, and provide an $\mathcal{O} \( \delta^2 \)$ bias bound, which is stated below in Theorem \ref{thm:hess-bias}. 


\begin{theorem} 
    \label{thm:hess-bias}
    The Hessian estimator $ \wh{\H} {f}_k^\delta $ satisfies
    \begin{enumerate}[label=(\alph*)]
        \item If $ f $ is $(3, L_2)$-smooth, then for all $ x \in \R^n $, $ \left\| \E \[ \wh{\H} {f}_k^\delta (x)  \] - \nabla^2 f (x) \right\| \le \frac{2n L_2 \delta}{n+1} $.  
        \item If $f$ is $(5,L_5)$-smooth, then for all $x\in\mathbb{R}^n$, the bias of $\wh{\H} f_k^\delta $ satisfies
        \begin{align*}
            \left\|\E\[\wh{\H} f_k^\delta (x)\]-\nabla^2f(x)\right\|\leq\frac{\delta^2}{n + 2}\left\|\widetilde{F}\right\|+\frac{4 \delta^3 L_5 n^2 }{15}, 
        \end{align*}
        where $\widetilde{F}$ is an $n\times n$ matrix and $\widetilde{F}_{ij}=\sum_{m=1}^n \[ \partial^4 f(x) \]_{mmij}$. 
    \end{enumerate} 
\end{theorem} 

A different form of Theorem \ref{thm:hess-bias} has appeared in \citep{wang2022hess}. The proof of item (a) in this theorem is provided in the Appendix, since its proof does not deviate much from that in \citep{wang2022hess}. 
In \citep{wang2022hess}, the second author showed that the leading term of bias of the Hessian estimator is of order $ O \( \| \partial^4 f (x) \| \delta^2 \). $ Here we provide a more refined bound. In particular, the order 4 tensor $\partial^4 f (x)$ is contracted to a matrix and the operator norm of the matrix divided by $n$ is used to bound the bias. The contraction followed by a division of $n$ means that this bound is more robust and refined than the previous one \citep{wang2022hess}. In cases where $\partial^4 f (x)$ has only a few large entries, the bound in Theorem \ref{thm:hess-bias} can be smaller than the previous bound $ O \( \| \partial^4 f (x) \| \delta^2 \) $. 


\begin{proof}[Proof of Theorem \ref{thm:hess-bias}(b)]
Without loss of generality, let $x=0$. From the definition 
\begin{align*}
    \wh{\H} f_k^\delta (0)=\frac{n^2}{8\delta^2k^2}\sum_{i,j=1}^k\big(&f(\delta v_i+\delta w_j)-f(-\delta v_i+\delta w_j)-f(\delta v_i-\delta w_j)\\
    &+f(-\delta v_i-\delta w_j)\big)\cdot(v_iw_j^\tr+w_jv_i^\tr). 
\end{align*}
Note that, by linearity of expectation and symmetry, $\E\[\wh{\H} f_k^\delta (0)\]=\frac{ n^2}{\delta^2}\E \[f(\delta v+\delta w)(vw^\tr)\]$, where $v$ and $w$ are independent and uniformly sampled from $\mathbb{S}^{n-1}$. Taylor expansion gives that
\begin{align*}
    f(\delta v+\delta w) = &f(0) + \partial^1f(0)[\delta v+\delta w] + \frac{1}{2}\partial^2f(0)[\delta v+\delta w]+\frac{1}{6}\partial^3f(0)[\delta v+\delta w]\\ &+\frac{1}{24}\partial^4f(0)[\delta v+\delta w]+\frac{1}{120}\partial^5f(\xi(\delta v+\delta w))[\delta v+\delta w],
\end{align*}
where we use notation $\xi(\delta v+\delta w)$ to show that $\xi$ is a function of $\delta v+\delta w$.

Therefore, the expectation can be written as
\begin{align*}
    & \;\E\[f(\delta v +\delta w)(vw^\tr)\]\\
    \overset{\textcircled{1}}{=}& \;
    \E \[ f(0) (vw^\tr) \] + \E \[  \partial^1f(0)[\delta v+\delta w] (vw^\tr) \] +  \frac{1}{2} \E \[ \partial^2f(0)[\delta v+\delta w] (vw^\tr) \] \\
    &+\frac{1}{6} \E \[ \partial^3f(0)[\delta v+\delta w] (vw^\tr) \] +
    \frac{1}{24} \E \[ \partial^4f(0)[\delta v+\delta w] (vw^\tr) \] \\ 
    &+ \frac{1}{120} \E \[ \partial^5f(\xi(\delta v+\delta w))[\delta v+\delta w] (vw^\tr) \]
\end{align*}

Since $\E \[ vw^\tr \] \overset{\textcircled{2}}{=}0$, the first term in \textcircled{1} equals to $0$. 

For the second term in \textcircled{1}, the $(i,j)$-component of $\partial^1f(0)[v+w]vw^\tr$ is $\sum_l\partial^1 f(0)_l (v_l+w_l) v_iw_j$. For any $l$, $v_lv_iw_j$ is odd in $w_j$. Thus we have 
\begin{align*} 
    \E \[  v_l v_i w_j | v_l = a, v_i = b \] = 0
\end{align*}
for any $ a,b $, which implies $ \E \[  v_l v_i w_j \] = 0$. Similarly $ \E \[ w_l v_i v_j \] = 0 $. This implies that the second term in \textcircled{1} is zero:  
\begin{align*}
    \E \[ \partial^1f(0)[\delta v+\delta w] (vw^\tr) \] \overset{\textcircled{3}}{=} 0.
\end{align*} 
Similar arguments show that the forth term equals to $0$: 
\begin{align*}
    \E \[ \partial^3f(0)[\delta v+\delta w] (vw^\tr) \] \overset{\textcircled{4}}{=} 0.
\end{align*} 


For the third term, the $(i,j)$-component of $\partial^1f(0)[v+w]vw^\tr$ is $\sum_{p,q}\partial^2f(0)_{pq}(v_p+w_p)(v_q+w_q)v_iw_j$. 
It holds that  
\begin{align*} 
    &\; \E \[ \partial^2 f (0) [ v + w] (vw^\tr) \] \\
    =&\; 
    \E \[ {2 \choose 0}\partial^2 f (0) [ v ] (vw^\tr) \] + \E \[ {2 \choose 1}\partial^2 f (0) [ v, w ] (vw^\tr) \] + \E \[ {2 \choose 2}\partial^2 f (0) [ w, w ] (vw^\tr) \] \\
    =&\; 
    \E \[ {2 \choose 1} \partial^2 f (0) [ v, w ] (vw^\tr) \] \\ 
    =&\; 
    \frac{2 \delta^2}{n^2} \partial^2 f (0) , 
\end{align*} 
where the second last equation follows from that expectation of terms of odd power of $v$ or $w$ is zero, and the last equation uses Proposition \ref{prop:mat}. 




Let $F = \partial^4 f (0)$ for simplicity. For the fifth term, 
it holds that 
\begin{align*} 
    \E \[ F [v+w] ( v w^\tr ) \] 
    =&\; 
    \E \[ {4 \choose 0} F [v] ( v w^\tr ) \] + \E \[ {4 \choose 1} F [v,v,v,w] ( v w^\tr ) \] +
    \E \[ {4 \choose 2} F [v,v,w,w] ( v w^\tr ) \] \\
    &+ \E \[ {4 \choose 3} F [v,w,w,w] ( v w^\tr ) \] + 
    \E \[ {4 \choose 4} F [w] ( v w^\tr ) \] \\
    =& \; 
    \E \[ {4 \choose 1} F [v,v,v,w] ( v w^\tr ) \] + \E \[ {4 \choose 3} F [v,w,w,w] ( v w^\tr ) \] \\
    \overset{\textcircled{5}}{=}& \;
    8 \E \[ F [v,v,v,w] ( v w^\tr ) \] , 
\end{align*} 
where the second equation uses the symmetric property to conclude that terms of odd powers are zero (similar to the previous arguments), and the last equation uses symmetry of $F$ and equivalence of $v$ and $w$.  
    
By Proposition \ref{prop:fourth}, we have that, for any $i,j$,  
\begin{align*}
    \E \[ \( F [v,v,v,w] ( v w^\tr ) \)_{ij} \] 
    =& \; 
    \E \[ \sum_{pqrs} F_{pqrs} v_p v_q v_r w_s v_i w_j \] \\
    =&\;  
    \E \[ F_{iiij} v_i^4 w_j^2 \] + 3 \E \[ \sum_{m: 1\le m \le n, m \neq i}F_{mmij} v_i^2 v_m^2 w_j^2 \] \\
    \overset{\textcircled{6}}{=}& \;  
    \frac{3}{n^2(n + 2) } \sum_{m=1}^n F_{mmij}. 
\end{align*} 

Combining \textcircled{5} and \textcircled{6} gives
\begin{align*} 
    \frac{1}{24} \E \[ \[ \partial^4f(0)[\delta v+\delta w] (vw^\tr)\]_{ij} \] 
    \overset{\textcircled{7}}{\le} 
    \frac{\delta^4}{ n^2 (n + 2) } \sum_{m=1}^n F_{mmij}, 
\end{align*} 
where $F = \partial^4 f (0)$. 


For the sixth term in \textcircled{1}, Proposition \ref{prop:smooth} gives that $\|\partial^5f(x)\|\leq L_5$, so $|\partial^5f(\xi(\delta v+\delta w))[v+w]|\leq 32 L_5$. Besides, since $v,w\in\mathbb{S}^{n-1}$, $\|vw^\tr\|\leq1$. Therefore, we have
\begin{align}\label{bias:v:e4}
    \frac{\delta^5}{120} \left\| \E \[  \partial^5 f(\xi(\delta v+\delta w))[v+w]vw^\tr \right\| \] 
    \overset{\textcircled{8}}{\le}
    \frac{4\delta^5 L_5 }{15}, 
\end{align}
where $ \wt{F} $ is 
Collecting terms from \textcircled{1} \textcircled{2} \textcircled{3} \textcircled{4} \textcircled{7} \textcircled{8}, we have 
\begin{align*}
    \left\|\frac{n^2}{\delta^2}\E [f(\delta v+\delta w)(vw^\tr)]-\nabla^2f(0)\right\|\leq\frac{\delta^2}{n + 2}\left\|\widetilde{F}\right\|+\frac{ 4 \delta^3 L_5 n^2}{15}, 
\end{align*}
 where $\widetilde{F}$ is an $n\times n$ matrix and $\widetilde{F}_{ij}=\sum_{m=1}^n \[ \partial^4 f(0) \]_{mmij}$. 
\end{proof}

%% file: exp-grad.tex
\section{Empirical Studies for Gradient Estimators}
\label{sec:exp}

In this section, we empirically study the newly introduced gradient estimator. The experiments are divided into three subsections. The first two subsections compare our method with existing methods, and the third subsection empirically verify the theoretical variance bound. To avoid clutter, only some results are listed in the main text, while additional results can be found in the Appendix. 




\subsection{Comparison with Stochastic Estimators} 

For the same number of function evaluations specified by $k$, and finite difference granularity $\delta$, we compare the estimators $ \wh{\nabla} f_k^\delta (x) $ with: 
\begin{itemize} 
    \item The estimator via spherical sampling \citep{flaxman2005online,wang2021GW}: 
    \begin{align} 
        \wh{\nabla} f_{k,S}^\delta (x) : = \frac{ n }{ 2 k \delta }\sum_{i=1}^k  \( f (x + \delta v_i) - f (x - \delta v_i) \) v_i , \label{eq:grad-sphere}
    \end{align}
    where $v_1, v_2, \cdots, v_k$ are uniformly $i.i.d.$ sampled from $ \S^{n-1} $.  
    \item The estimator via Gaussian sampling \citep{nesterov2017random}: 
    \begin{align}
        \wh{\nabla} f_{k,G}^\delta (x) : = \frac{ \sqrt{n} }{ 2 k \delta } \sum_{i=1}^k  \( f 
        \( x +  \frac{ \delta v_i}{\sqrt{n}} \) - f \(x - \frac{ \delta  v_i}{\sqrt{n}} \) \) v_i , \label{eq:grad-gauss} 
    \end{align} 
    where $v_1, v_2, \cdots, v_k \overset{i.i.d.}{\sim} \mathcal{N} \( 0 , I \)$. Note that the random vectors $v_i$ are divided by $\sqrt{n}$ so that in expectation the step size (granularity) is $ \Theta \( \delta \) $. 
    \item 
    The estimator via Rademacher random vectors \citep{wang2018stochastic,doi:10.1137/21M1392966}. Following \citep{doi:10.1137/21M1392966}, we say $ z = (z_1, z_2, \cdots, z_n) $ is a $k$-sparse Rademacher random vector if it can be constructed from the following sampling process. (1) A $k$-element subset $K$ of $\{ 1,2, \cdots, n \}$ is randomly selected (with uniform probability); Denote the elements of $K $ by $j_1, j_2, \cdots, j_k$.  
    (2) Sample $k$ $iid$ Rademacher random variables $r_{1}, r_{2}, \cdots, r_{k} $, and define $z = (z_1, z_2, \cdots, z_n)$ as  
    \begin{align*} 
        z_i
        = 
        \begin{cases} 
            r_{i}, & \text{if } j_i \in K, \\ 
            0, & \text{otherwise}. 
        \end{cases} 
    \end{align*} 
    The the gradient estimator based on Rademacher random vector is 
    \begin{align} 
        \wh{\nabla} f_{k,R}^\delta (x) : = 
        \( g_1, g_2, \cdots, g_n \) 
        \label{eq:grad-Rademacher} 
    \end{align} 
    where 
    \begin{align*} 
        g_i 
        = 
            \frac{ z_i (  f ( x + \delta z_i) - f (x) ) }{\delta} .  
    \end{align*} 
    In \citep{wang2018stochastic,doi:10.1137/21M1392966}, sparsity (or sparsity-type) constraints are imposed on the gradient. For our purpose, we do not assume sparsity and focus on the accuracy of the estimation. 
    \item Comparison-based gradient estimator $ \wh{\nabla} f_{k,C}^\delta (x) $, which estimates the normalized gradient. This estimator is defined in Algorithm \ref{alg:bit}\citep{CAI2022242}. 
    
\end{itemize}
    
    \begin{algorithm}[h]
        \caption{Comparison-based gradient estimator \citep{CAI2022242}} 
        \label{alg:bit} 
        \begin{algorithmic}[1]
            \STATE \textbf{Input:} Number of random vectors $k$; Finite difference granularity: $\delta$; Location for evaluation $x \in \R^n$; Comparison oracle for the target function $f$: $ \mathcal{C}_f ( x,y ) = \text{sign} ( f (x) - f (y) ) $ for all $x,y \in \R^d$; Sparsity parameter: $s$. 
            \STATE Uniformly sample $i.i.d.$ vectors $v_1 , v_2, \cdots, v_k $ from $\S^{n-1}$. 
            \STATE Let $ z_i = \mathcal{C}_f ( x + \delta v_i , x ) $ for $i = 1,2,\cdots, k$. 
            \STATE \textbf{Output:} $ \wh{\nabla} f_{k,C}^\delta (x) := \arg\max_{g: \| g \|_1 \le \sqrt{s} , \| g \| \le 1} \sum_{i=1}^k z_i v_i^\top g $. 
        \end{algorithmic} 
    \end{algorithm} 

Note that as per its definition, the comparison-based estimator $  \wh{\nabla} f_{k,C}^\delta (x) $ does not provide an estimate for $  {\nabla} f (x) $. Instead, it estimates $ \frac{  {\nabla} f (x) }{ \left\|  {\nabla} f (x) \right\| } $. For this reason, the comparison with $ \wh{\nabla} f_{k,C}^\delta (x) $, and the comparison with $ \wh{\nabla} f_{k,S}^\delta (x) $, $ \wh{\nabla} f_{k,G}^\delta (x) $ and $ \wh{\nabla} f_{k,R}^\delta (x) $ are measured under different scales. 

All methods are tested using the following function 
\begin{align} 
    f (x) := \exp ( (x_1 - 1) (x_2 + 2)) + \sum_{j=1}^n \sin (x_j) , \label{eq:test} 
\end{align} 
where $x_j$ denotes the $j$-th component of vector $x$. Note that, unlikely experiments in some previous works \citep[e.g.,][]{wang2022hess}, all function evaluations are \emph{noise-free}. On this test function (Eq. \ref{eq:test}), the methods are tested with different choices of $ k $, $\delta$, $x$. 
Example comparison between (Eq. \ref{eq:def-grad-est}) and (Eq. \ref{eq:grad-sphere}), (Eq. \ref{eq:grad-gauss}) can be found in Figure \ref{fig:grad-compare-main}. Example comparison between (Eq. \ref{eq:def-grad-est}) and $ \wh{\nabla} f_{k,C}^\delta (x) $ (Algorithm \ref{alg:bit}) can be found in Figure \ref{fig:grad-compare-main2}. More results can be found in the Appendix. 


\begin{figure} 
    \centering
    \subfloat[$ x = 0, \delta = 0.1, k = 300 $]{\includegraphics[width = 0.4\textwidth]{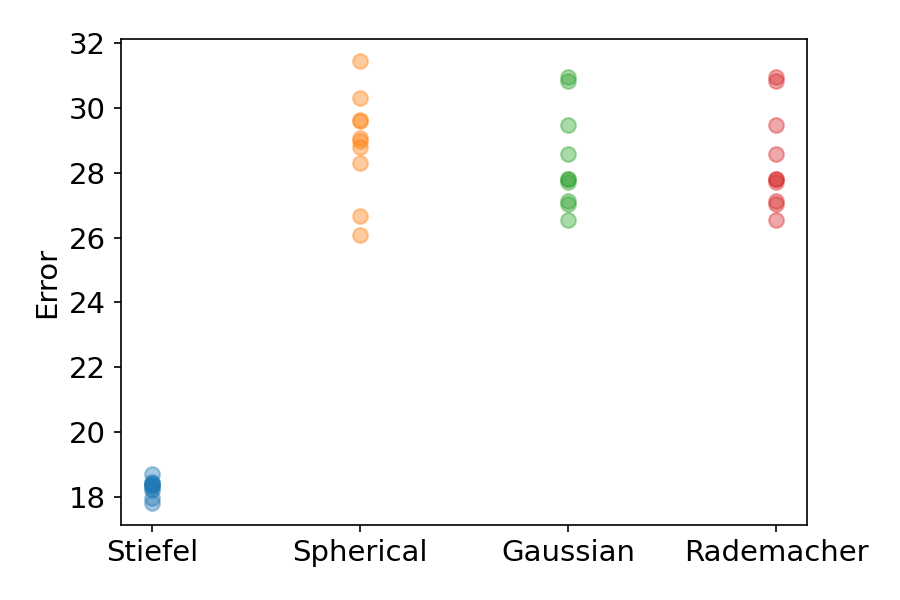}}
    \subfloat[$ x = 0, \delta = 0.1, k = 400 $]{\includegraphics[width = 0.4\textwidth]{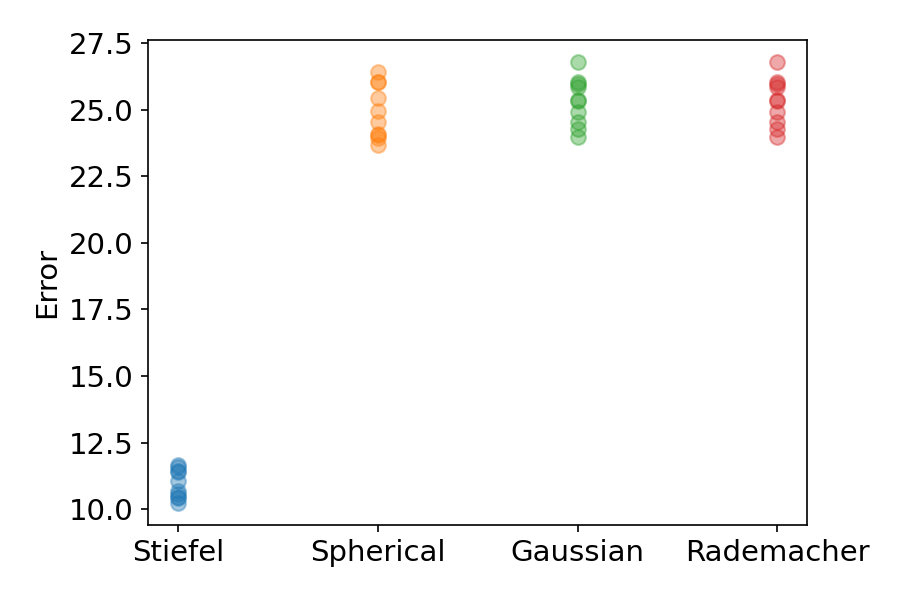}} \\
    \subfloat[$ x = 0, \delta = 0.01, k = 300 $]{\includegraphics[width = 0.4\textwidth]{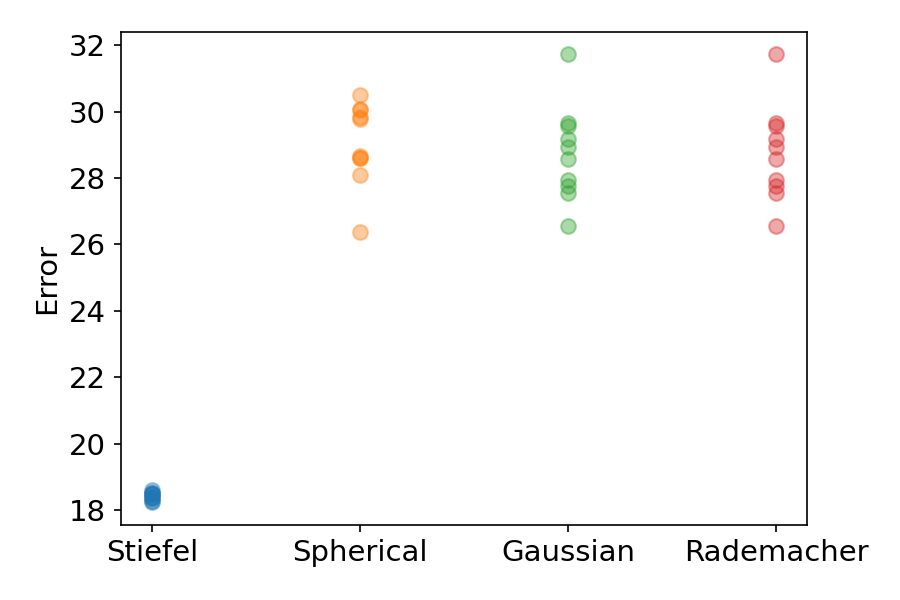}}
    \subfloat[$ x = 0, \delta = 0.01, k = 400 $]{\includegraphics[width =  0.4\textwidth]{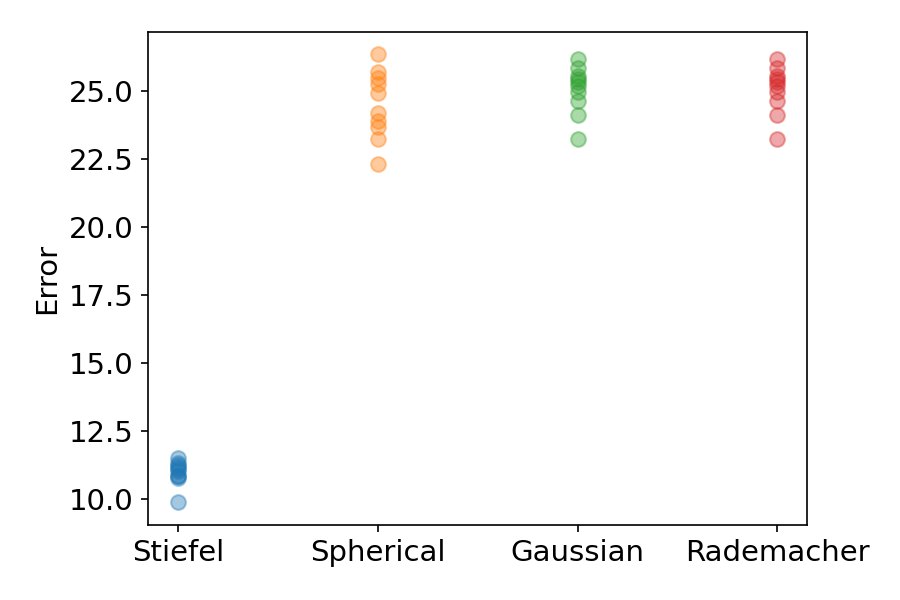}} \\ 
    \subfloat[$ x = 0, \delta = 0.001, k = 300 $]{\includegraphics[width = 0.4\textwidth]{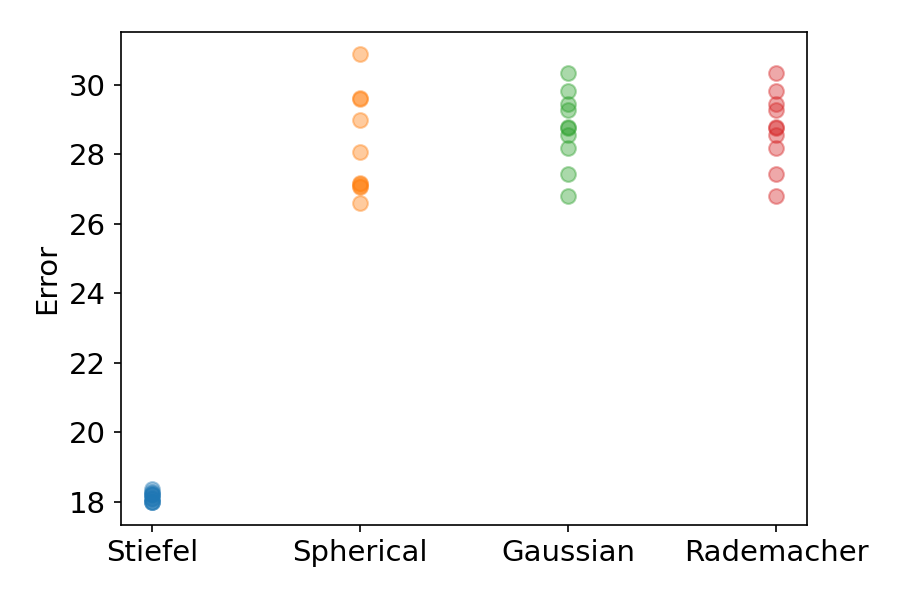}}
    \subfloat[$ x = 0, \delta = 0.001, k = 400 $]{\includegraphics[width = 0.4\textwidth]{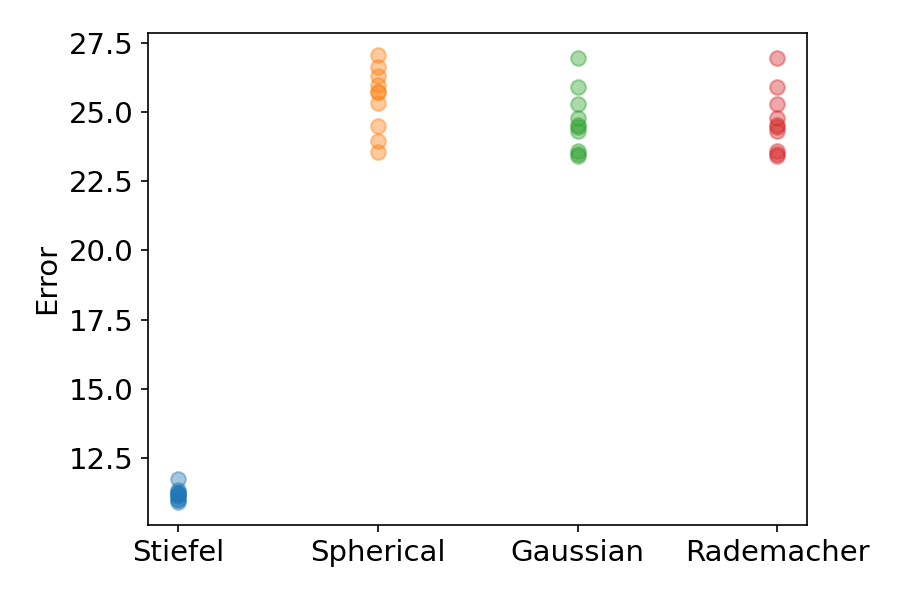}} 
    \caption{Errors of gradient estimators on the test function defined in (Eq. \ref{eq:test}) with $n = 500$. Each subfigure corresponds to a different combination of the location for estimation $x$, the finite difference granularity $ \delta $, and number of random directions $k$. The $x$-axis labels the estimators: ``Stiefel'' is the estimator $\wh{\nabla}{f}_k^\delta (x)$ (Eq. \ref{eq:def-grad-est}); ``Spherical'' is the estimator $\wh{\nabla}{f}_{k,S}^\delta (x)$ (Eq. \ref{eq:grad-sphere}); ``Gaussian'' is the estimator $\wh{\nabla}{f}_{k,G}^\delta (x)$ (Eq. \ref{eq:grad-gauss}); ``Rademacher'' is the estimator $\wh{\nabla}{f}_{k,R}^\delta (x)$ (Eq. \ref{eq:grad-Rademacher}). The $y$-axis is the error of the estimator. The error is $ \left\| \wh{\nabla} f_{k}^\delta (x) - \nabla f (x) \right\| $ (or $ \left\| \wh{\nabla} f_{k,S}^\delta (x) - \nabla f (x) \right\| $, $\left\| \wh{\nabla} f_{k,G}^\delta (x) - \nabla f (x) \right\| $, $\left\| \wh{\nabla} f_{k,R}^\delta (x) - \nabla f (x) \right\| $). 
    Each dot represents one observed error of one estimator. Each estimator is evaluated 10 times (thus 10 dots for each estimator). For example, in subfigure (a), the 10 blue dots scattered above ``Stiefel'' show that the errors of 10 evaluations of $\wh{f}_k^\delta (x)$ with parameters $x = 0, \delta = 0.1, k = 300$ are in range 17 to 20. 
    More results for other values of $(x,\delta,k)$ can be found in the Appendix. \label{fig:grad-compare-main}}
\end{figure}

\begin{figure}
    \centering
    \subfloat[$ x = 0, \delta = 0.1, k = 300 $]{\includegraphics[scale = 0.4]{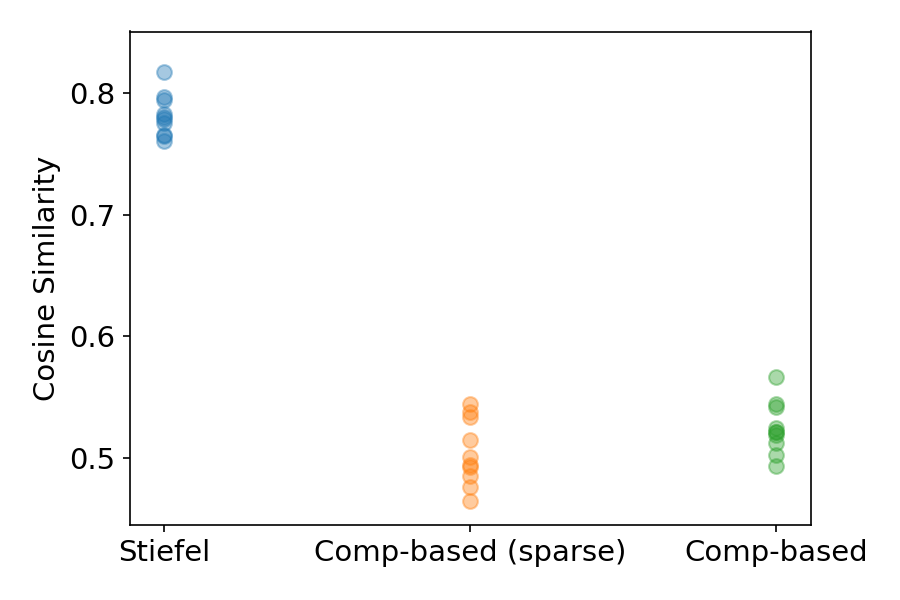}} 
    \subfloat[$ x = 0, \delta = 0.1, k = 400 $]{\includegraphics[scale = 0.4]{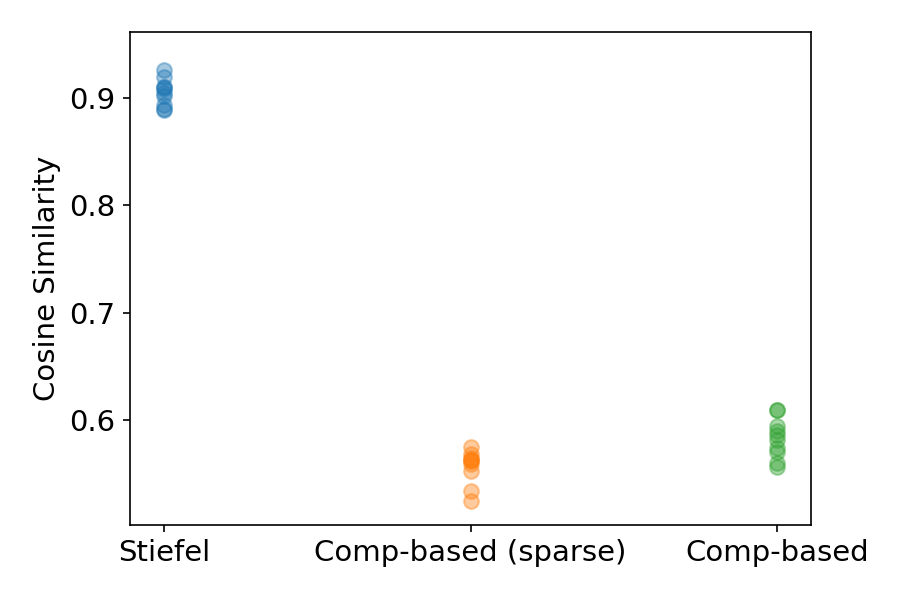}} \\
    \subfloat[$ x = 0, \delta = 0.01, k = 300 $]{\includegraphics[scale = 0.4]{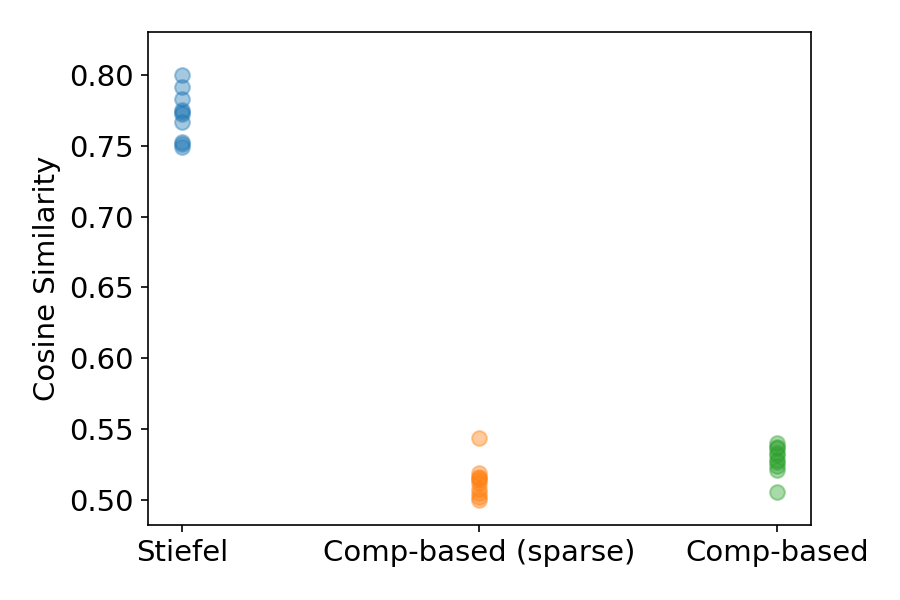}} 
    \subfloat[$ x = 0, \delta = 0.01, k = 400 $]{\includegraphics[scale = 0.4]{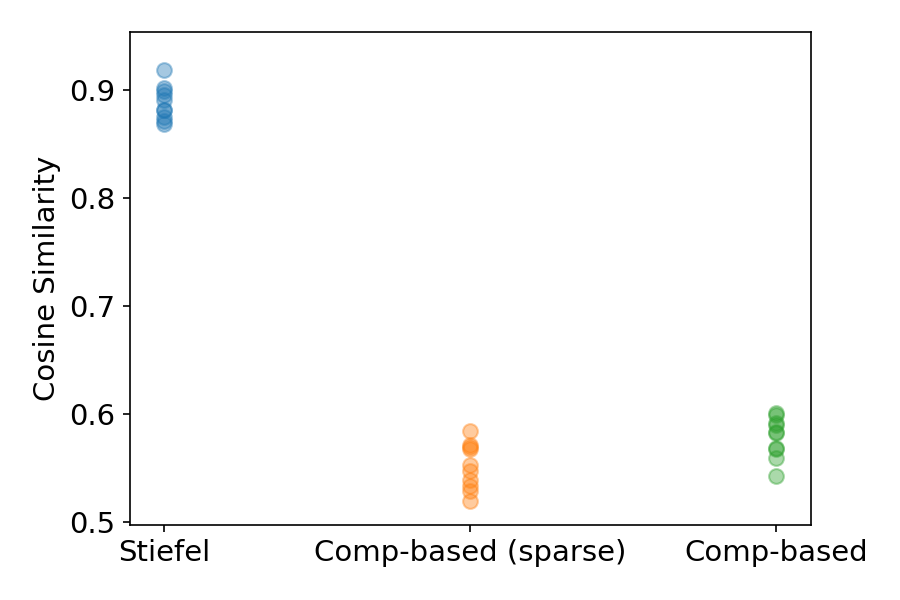}} \\
    \subfloat[$ x = 0, \delta = 0.001, k = 300 $]{\includegraphics[scale = 0.4]{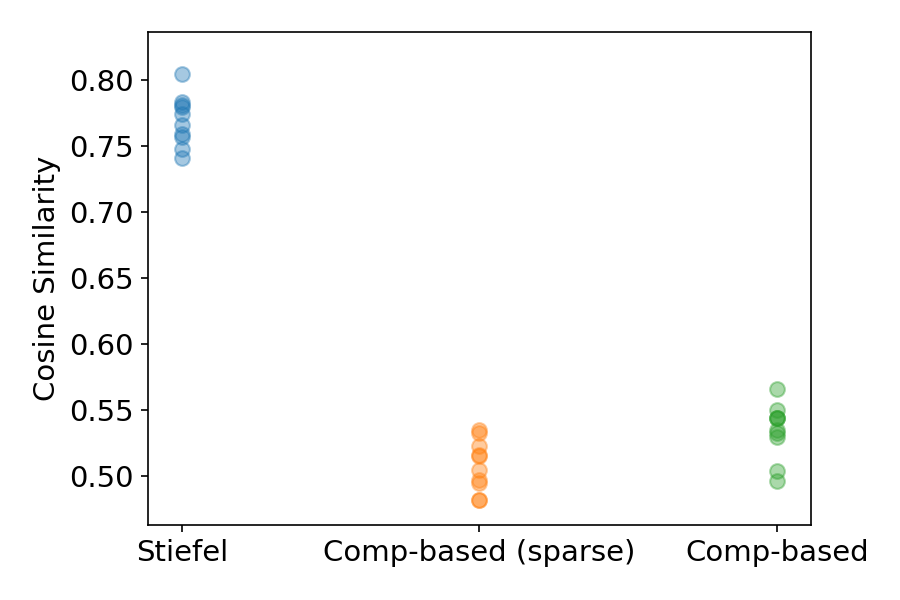}} 
    \subfloat[$ x = 0, \delta = 0.001, k = 400 $]{\includegraphics[scale = 0.4]{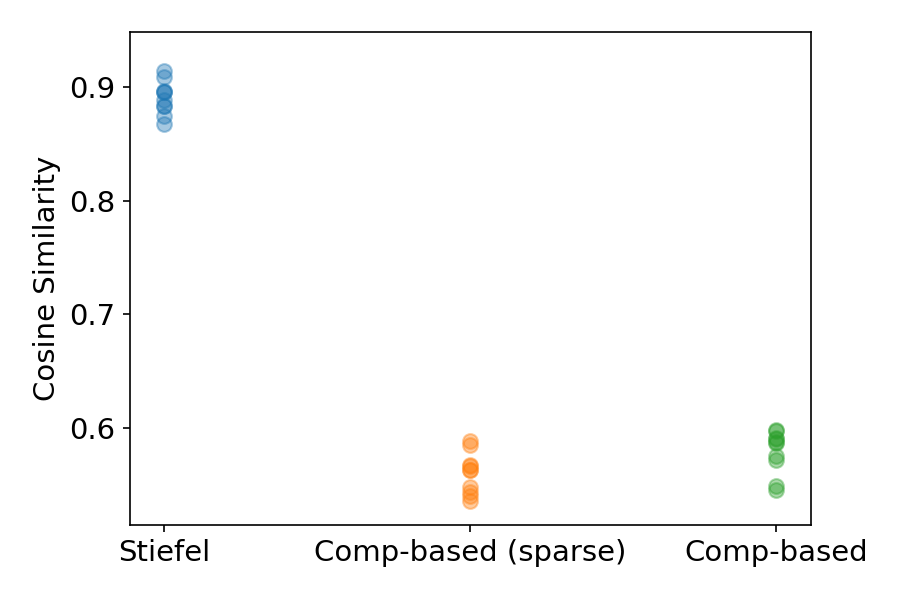}} 
    \caption{Performance of gradient estimators on the test function defined in (Eq. \ref{eq:test}), measured by cosine similarity with the true gradient. Each subfigure corresponds to a different combination of the location for estimation $x$, the finite difference granularity $ \delta $, and number of random directions $k$. The $x$-axis labels the estimators: ``Stiefel'' is the estimator $\wh{\nabla}{f}_k^\delta (x)$ (Eq. \ref{eq:def-grad-est}); ``Comp-based (sparse)'' is the estimator $\wh{\nabla}{f}_{k,C}^\delta (x)$ (Algorithm \ref{alg:bit}) with sparsity parameter $s = 100$; ``Comp-based'' is the estimator $\wh{\nabla}{f}_{k,C}^\delta (x)$ with sparsity parameter $s = \infty$ (no sparsity constrains). The $y$-axis is the cosine similarity between the estimator and the ground truth, which is $ \frac{ \< \wh{\nabla} f_k^\delta (x) , {\nabla} f (x) \> }{ \left\| \wh{\nabla} f_k^\delta (x) \right\| \left\| {\nabla} f (x) \right\| } $ (or $ \frac{ \< \wh{\nabla} f_{k,C}^\delta (x) , {\nabla} f (x) \> }{ \left\| \wh{\nabla} f_{k,C}^\delta (x) \right\| \left\| {\nabla} f (x) \right\| } $). 
    Each dot represents one observed cosine similarity of one estimator. Each estimator is evaluated 10 times (thus 10 dots for each estimator). For example, in subfigure (a), the 10 blue dots scattered above ``Stiefel'' show that the cosine similarity of 10 evaluations of $\wh{f}_k^\delta (x)$ with parameters $x = 0, \delta = 0.1, k = 300$ are in range 0.75 to 0.8. Note that larger cosine similarity means more accurate estimation. 
    More results for other values of $(x,\delta,k)$ can be found in the Appendix.  \label{fig:grad-compare-main2}}
\end{figure}



\subsection{Comparison with the Entry-wise Estimator} 

The estimator (Eq. \ref{eq:def-grad-est}) is also compared with the entry-wise estimator: 
\begin{align*}
    \wh{\nabla} f_{E}^\delta (x) := [ \wh{\nabla}_i f_{E}^\delta (x) ]_{i\in [n]}, 
\end{align*}
where $ \wh{\nabla}_i f_{E}^\delta (x) = \frac{1}{2\delta} \( f (x + \delta \e_i) - f (x - \delta \e_i)  \) $ and $\e_i$ is the vector with $1$ on the $i$-th entry and $0$ on all other entries. The comparison results are summarized in Tables \ref{tab:intro} and \ref{tab:entry}. 

\clearpage

\begin{table}[t!] 
    \caption{
    The errors of gradient estimators against finite-difference granularity $\delta$. The first row shows $\delta$; The second row shows errors of $ \wh{\nabla} f_n^\delta (x) $ ($n = 500$ is the dimension); The third row shows errors of the entry-wise estimator. 
    The error of an estimator is its distance to the true gradient in Euclidean norm: $\left\| \wh{\nabla} f_n^\delta (x) - \nabla f (x) \right\|$. Errors of $ \wh{\nabla}  f_n^\delta (x) $ are in average $\pm$ standard derivation format, where each average and standard deviation gather information from 10 runs. The entry-wise estimator is not random and its error is computed in a single run. The test function is defined in (Eq. \ref{eq:test}). For all evaluations in this table, $x = \frac{\pi}{4} \mathbf{1}$ is used. 
    \label{tab:entry}} 
    \centering
    \begin{tabular}{|c|| c c c|} 
        \hline 
        $\delta$ & $0.1$ & $0.01$ & $0.001$  \\ \hline
        \makecell{Stiefel sampling errors} & 2.4e-4$\pm$1.0e-5 & 2.5e-6$\pm$1.5-07 & 2.5e-8$\pm$6.8e-10 \\ \hline 
        \makecell{Entry-wise errors} & 3.2e-2 & 3.2e-4 & 3.2e-6 \\ \hline 
    \end{tabular} 
\end{table}

\begin{table}[h!] 
    \caption{
    The errors of Hessian estimators against finite-difference granularity $\delta$. The first row shows $\delta$; The second row shows errors of $ \wh{\H} f_n^\delta (x) $ ($n = 100$ is the dimension); The third row shows errors of the entry-wise estimator $ \wh{\H} f_E^\delta (x) $. 
    The error of an estimator is its distance to the true Hessian in spectral norm: $\left\| \wh{\H} f_n^\delta (x) - \nabla^2 f (x) \right\|$ (or $\left\| \wh{\H} f_E^\delta (x) - \nabla^2 f (x) \right\|$). Errors of $ \wh{\H}  f_n^\delta (x) $ are in ``average $\pm$ standard derivation'' format, where each average and standard deviation gather information from 10 runs. The entry-wise estimator is not random and its error is computed in a single run. The test function is defined in (Eq. \ref{eq:test}). For all evaluations in this table, $x = \frac{\pi}{2} \mathbf{1}$ is used. 
    \label{tab:hess-entry1}} 
    \centering
    \begin{tabular}{|c|| c c c|} 
        \hline 
        $\delta$ & $0.1$ & $0.01$ & $0.001$  \\ \hline
        \makecell{Stiefel sampling errors} &  0.17$\pm$0.024 & 1.7e-3$\pm$0.16e-4 & 1.6e-5$\pm$1.6e-6 \\ \hline 
        \makecell{Entry-wise errors} & 4.4 & 4.3e-2 & 4.3e-4  \\ \hline 
    \end{tabular} 
\end{table}

\begin{table}[h!] 
    \centering 
    \caption{
    The errors of Hessian estimators against finite-difference granularity $\delta$. The first row shows $\delta$; The second row shows errors of $ \wh{\H} f_n^\delta (x) $ ($n = 100$ is the dimension); The third row shows errors of the entry-wise estimator $ \wh{\H} f_E^\delta (x) $. 
    The error of an estimator is its distance to the true Hessian in spectral norm: $\left\| \wh{\H} f_n^\delta (x) - \nabla^2 f (x) \right\|$ (or $\left\| \wh{\H} f_E^\delta (x) - \nabla^2 f (x) \right\|$). Errors of $ \wh{\H}  f_n^\delta (x) $ are in ``average $\pm$ standard derivation'' format, where each average and standard deviation gather information from 10 runs. The entry-wise estimator is not random and its error is computed in a single run. The test function is defined in (Eq. \ref{eq:test}). For all evaluations in this table, $x = \frac{\pi}{4} \mathbf{1}$ is used. 
    \label{tab:hess-entry2}} 
    \centering
    \begin{tabular}{|c|| c c c|} 
        \hline 
        $\delta$ & $0.1$ & $0.01$ & $0.001$  \\ \hline
        \makecell{Stiefel sampling errors} & 4.1e-3$\pm$5.3e-4 & 3.8e-5$\pm$4.63e-6 & 3.8e-7$\pm$3.7e-8 \\ \hline 
        \makecell{Entry-wise errors} & 0.12 & 1.2e-3 & 1.2e-5  \\ \hline 
    \end{tabular} 
\end{table} 



\subsection{Empirical Verification of the Theorems} 

As discussed in the introduction (Figure \ref{fig:intro}), the error is highly aligned with the variance bound. Here we present more versions of Figure \ref{fig:intro}, with different values of $\delta$ and $x$. These results can be found in Figure \ref{fig:trend}. 


\begin{figure}[h!]
    \centering
    \subfloat[$x = 0$, $\delta = 0.001$]{\includegraphics[width = 0.4\textwidth]{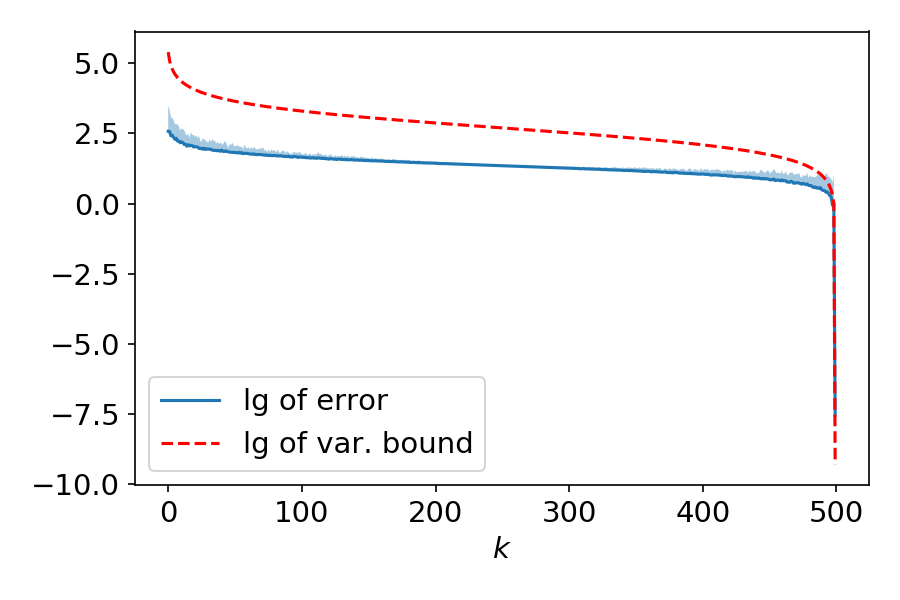}}
    \subfloat[$x = \frac{\pi}{4} \mathbf{1}$, $\delta = 0.001$]{\includegraphics[width = 0.4\textwidth]{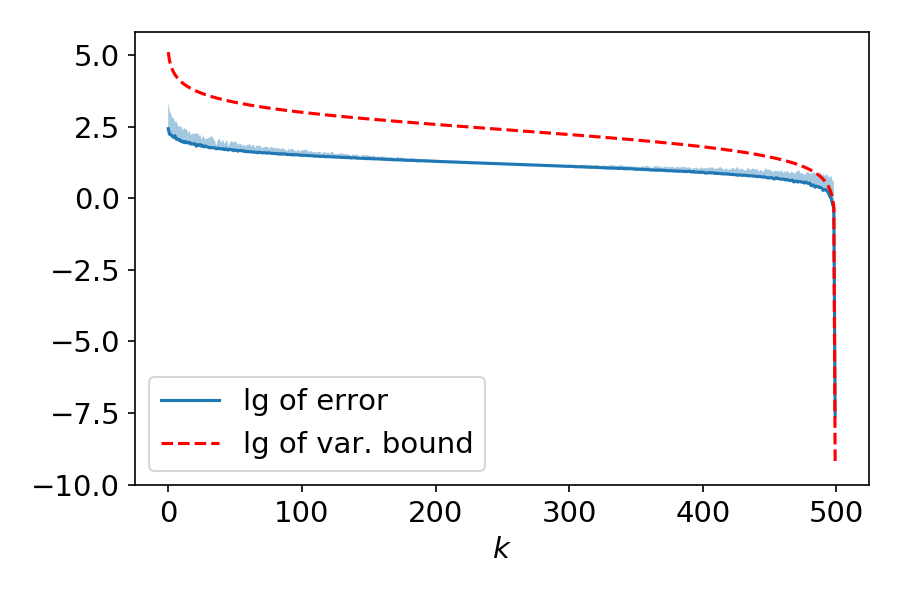}} \\
    \subfloat[$x = 0$, $\delta = 0.01$]{\includegraphics[width = 0.4\textwidth]{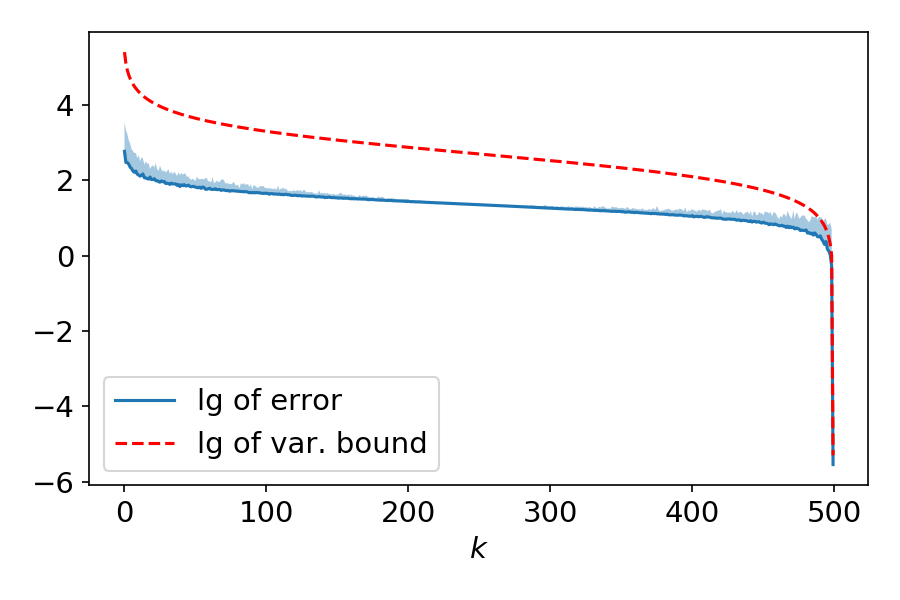}}
    \subfloat[$x = \frac{\pi}{4} \mathbf{1}$, $\delta = 0.01$]{\includegraphics[width = 0.4\textwidth]{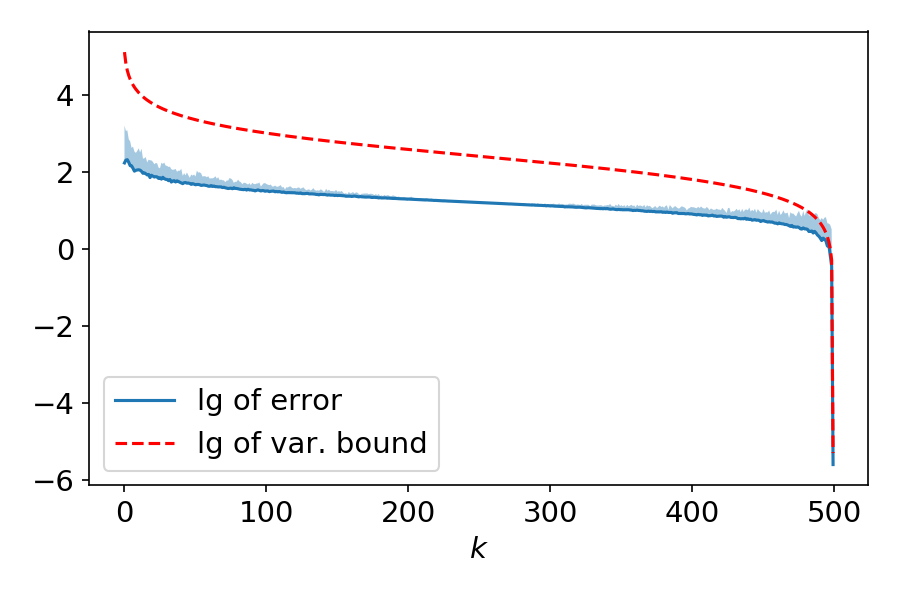}} \\ 
    \subfloat[$x = 0$, $\delta = 0.1$]{\includegraphics[width = 0.4\textwidth]{var_trend_x0_delta01.png}}
    \subfloat[$x = \frac{\pi}{4} \mathbf{1}$, $\delta = 0.1$]{\includegraphics[width = 0.4\textwidth]{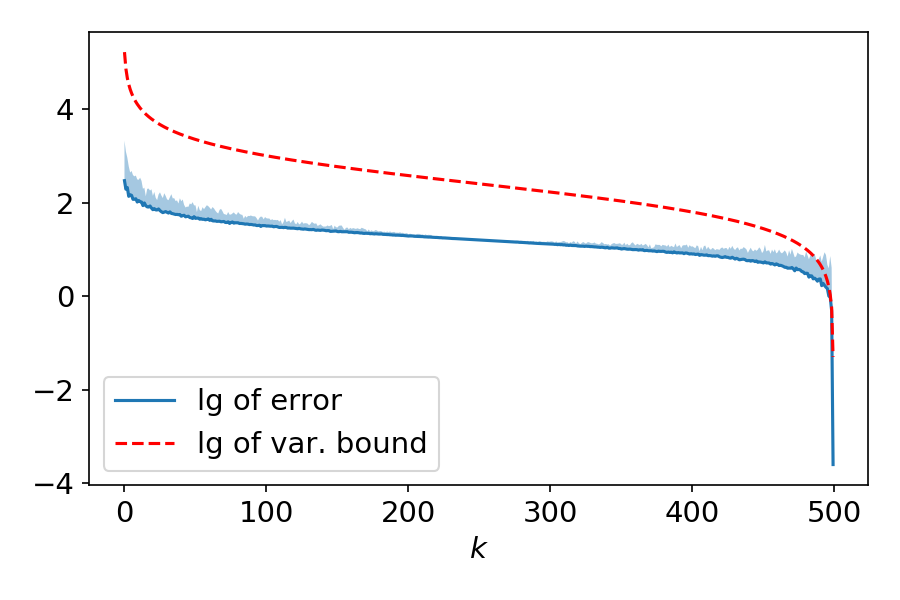}}
    \caption{Errors of gradient estimators $ \wh{\nabla} f_k^\delta (x) $ with $k$ ranging from $1$ to $n$, in base-10 log-scale. Here $n = 500$. The underlying test function is $ f (x) = \exp ( (x_1 - 1) (x_2 + 2)) + \sum_{j=1}^{500} \sin (x_j)  $, where $x_j$ denotes the $j$-th component of vector $x$. The location for estimation $x$, and the finite difference granularity $\delta$ are labeled in the captions of the subfigures. The solid blue curve plots the errors of gradient estimator in logarithmic scale, and is averaged over 10 runs. The shaded area above the solid curve shows 10 times standard deviation of the errors in logarithmic scale.
    The dashed red curve, as a function of $k$, is $ c (k) = \lg \( \| \nabla f (x) \|^2 \( \frac{n}{k} - 1 \) + \delta^2 \( \frac{n^2}{k} - n \) \| \nabla f (x) \| + \frac{ \delta^4 n^2 }{ k } \)$, which is the base-10 log of of variance bound for the gradient estimators (up to constants). \label{fig:trend}} 
\end{figure}

%% file: exp-hess.tex
\section{Empirical Results for the Hessian Estimators}

The Hessian estimators are empirically studied, in the same way that the gradient estimators are studied. Similar to the gradient case, $(i)$ the errors of the Hessian estimators also draw an ``S''-shape curve in logarithmic scale (See Figure \ref{fig:trend-hess}); $(ii)$ When $k = n$, the stochastic Hessian estimator (Eq. \ref{eq:def-hess-est}) outperform the classic entry-by-entry Hessian estimator (See Tables \ref{tab:hess-entry1} and \ref{tab:hess-entry2}). 
When $k$ is much smaller than $n$, the supremacy of Hessian estimator (Eq. \ref{eq:def-hess-est}) is sometimes less pronounced. In particular, the estimator (Eq. \ref{eq:def-hess-est}) may have same level of accuracy as the estimator introduced by \citet{wang2022hess} (See Figure \ref{fig:hess-compare-app1} in the Appendix for details). 


\subsection{Comparison with Stochastic Estimators} 

For the same number of random direction specified by $k$, and finite difference granularity $\delta$, we compare the estimators $ \wh{\H} f_k^\delta (x) $ with: 
\begin{itemize} 
    \item The estimator via spherical sampling \citep{wang2022hess}: 
    \begin{align} 
        \wh{\H} f_{k,S}^\delta (x) : =  \frac{ n^2 }{ 8 k \delta^2 }  \sum_{i=1}^k \sum_{j=1}^k & \; \bigg( f (x + \delta v_i + \delta w_j ) - f (x + \delta v_i - \delta w_j ) \nonumber \\ 
        &- f (x - \delta v_i + \delta w_j ) + f (x - \delta v_i - \delta w_j ) \bigg) \( v_i w_j^\top + w_j v_i^\top  \),  \label{eq:hess-sphere}
    \end{align} 
    where $ v_1, v_2, \cdots, v_k $ and $ w_1, w_2, \cdots, w_k $ are $i.i.d.$ sampled from the uniform distribution over $\S^{n-1}$.  
    \item The estimator via Gaussian sampling and the Stein's identity \citep{balasubramanian2021zeroth}: 
    \begin{align}
        \wh{\H} f_{k,G}^\delta (x) : = \frac{ n }{ 2 k^2 \delta^2 } \sum_{i=1}^{k^2} \( f \( x + \frac{\delta v_i}{\sqrt{n}} \) - 2f (x) + f \( x - \frac{\delta v_i}{\sqrt{n}} \) \) (v_i v_i^\top - I ) , \label{eq:hess-gauss} 
    \end{align} 
    where $v_1, v_2, \cdots, v_k \overset{i.i.d.}{\sim} \mathcal{N} \( 0 , I \)$. Note that the finite difference step size is downscale by a factor of $\sqrt{n}$ so that the expected granularity is of order $\Theta (\delta)$. 
\end{itemize}

All methods are tested using the test function (Eq. \ref{eq:test}) with dimension $n = 100$. All function evaluations are \emph{noise-free}.

\begin{figure}[h]
    \centering
    \subfloat[$ x = 0, \delta = 0.1, k = 60 $]{\includegraphics[width = 0.4\textwidth]{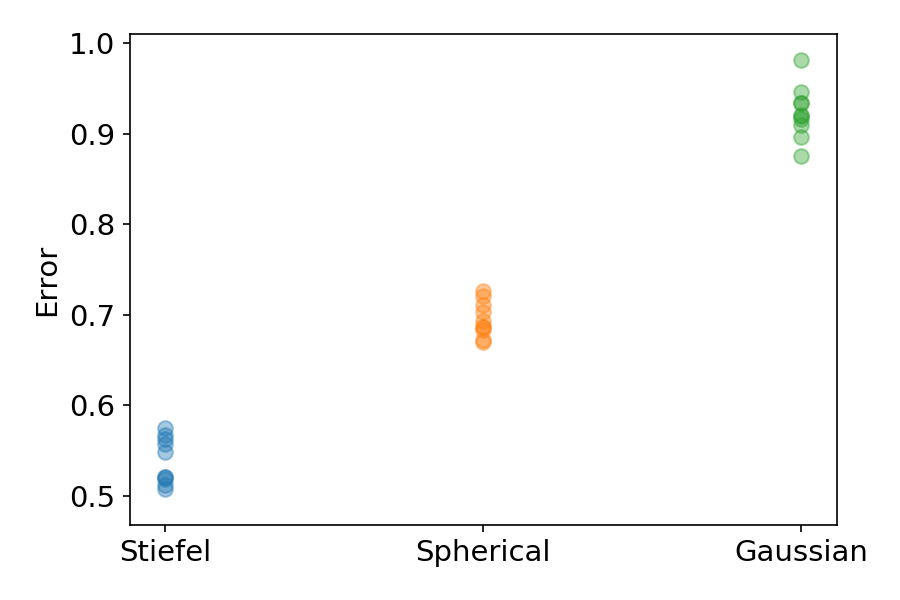}}
    \subfloat[$ x = 0, \delta = 0.1, k = 80 $]{\includegraphics[width = 0.4\textwidth]{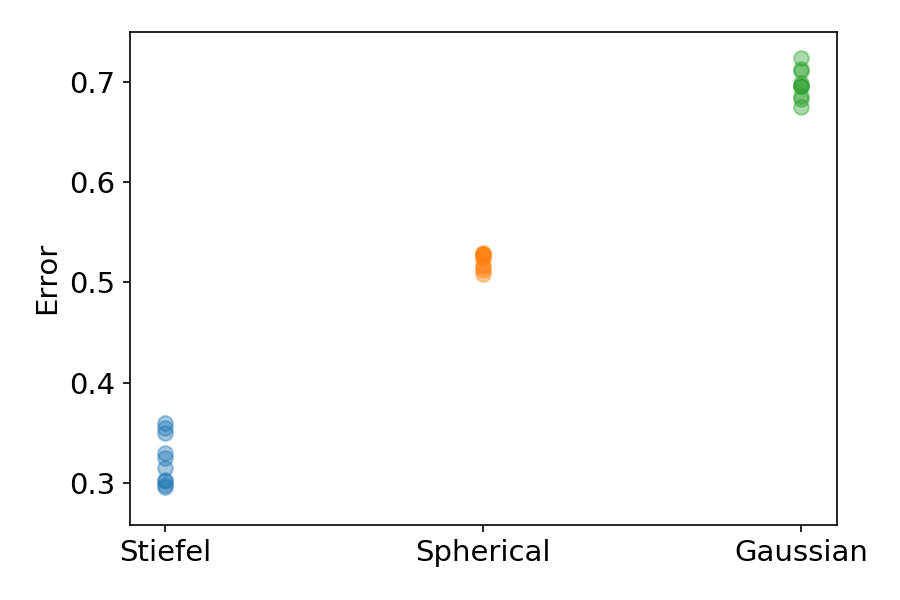}} \\
    \subfloat[$ x = 0, \delta = 0.01, k = 60 $]{\includegraphics[width = 0.4\textwidth]{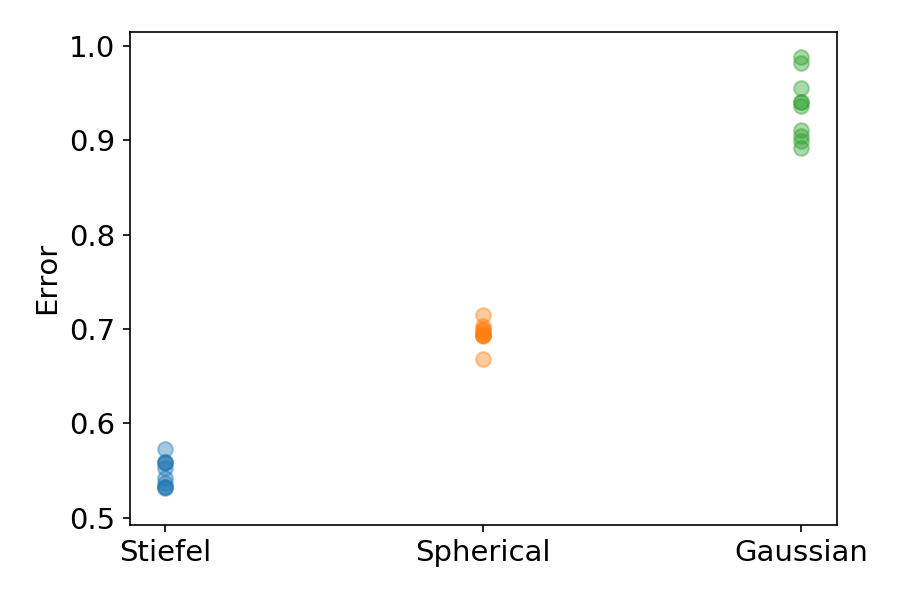}}
    \subfloat[$ x = 0, \delta = 0.01, k = 80 $]{\includegraphics[width =  0.4\textwidth]{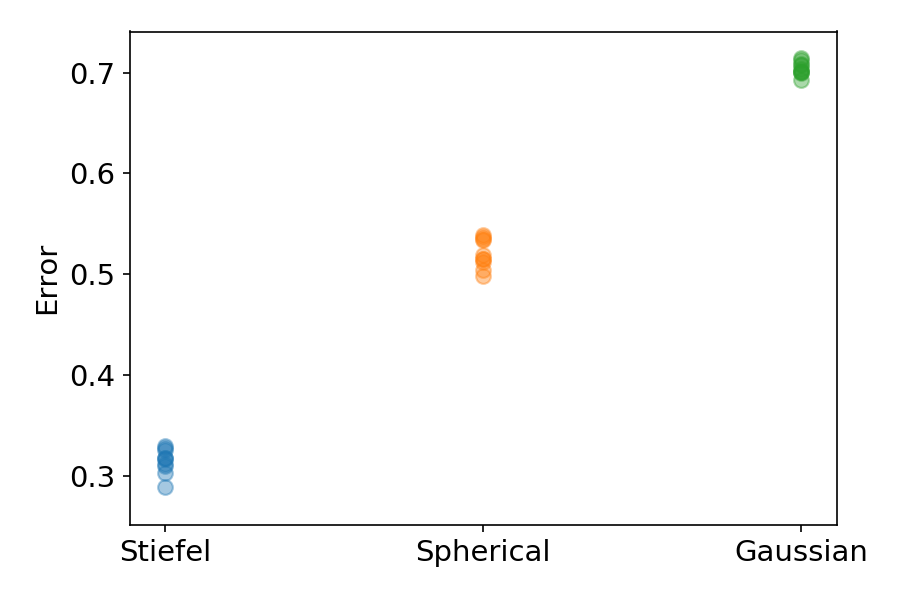}} \\ 
    \subfloat[$ x = 0, \delta = 0.001, k = 60 $]{\includegraphics[width = 0.4\textwidth]{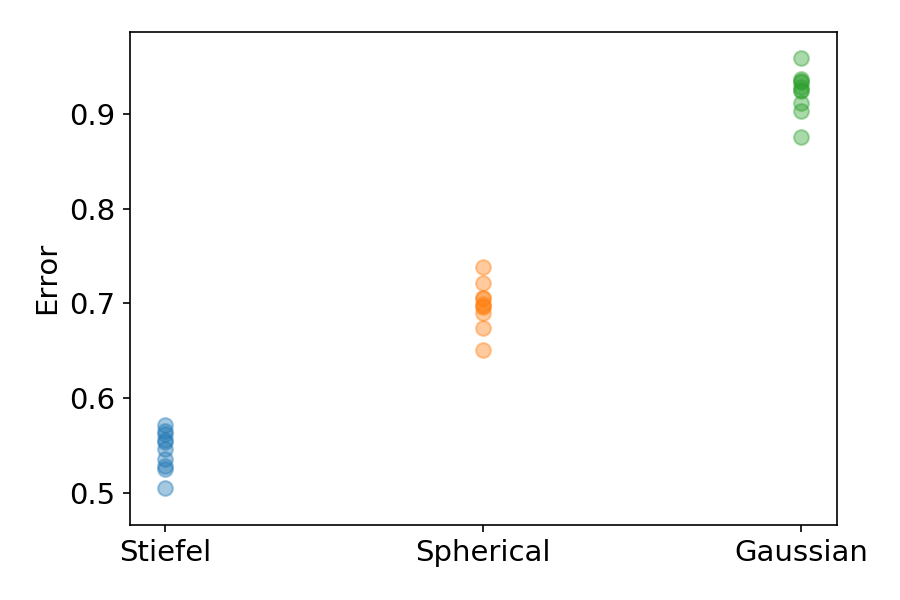}}
    \subfloat[$ x = 0, \delta = 0.001, k = 80 $]{\includegraphics[width = 0.4\textwidth]{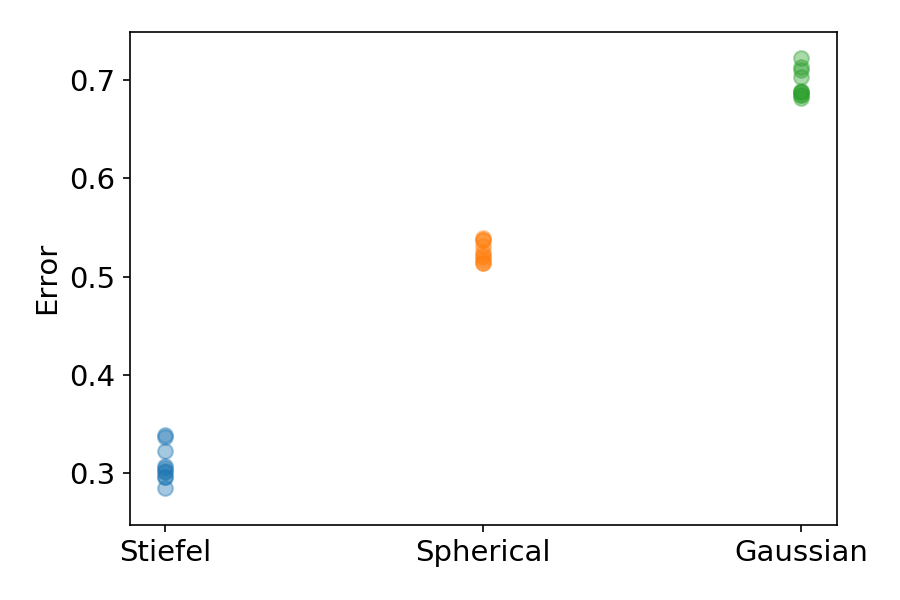}} 
    \caption{Errors of Hessian estimators on the test function defined in (Eq. \ref{eq:test}) with $n=100$. Each subfigure corresponds to a different combination of the location for estimation $x$, the finite difference granularity $ \delta $, and number of random directions $k$. The $x$-axis labels the estimators: ``Stiefel'' is the estimator $\wh{\H} f_k^\delta (x)$ (Eq. \ref{eq:def-hess-est}); ``Spherical'' is the estimator $\wh{\H}{f}_{k,S}^\delta (x)$ (Eq. \ref{eq:hess-sphere}); ``Gaussian'' is the estimator $\wh{\H}{f}_{k,G}^\delta (x)$ (Eq. \ref{eq:hess-gauss}). The $y$-axis is the error of the estimator. The error is $ \left\| \wh{\H} f_{k}^\delta (x) - \nabla^2 f (x) \right\|_F $ (or $ \left\| \wh{\H} f_{k,S}^\delta (x) - \nabla^2 f (x) \right\|_F $, $\left\| \wh{\H} f_{k,G}^\delta (x) - \nabla^2 f (x) \right\|_F $). 
    Each dot represents one observed error of one estimator. Each estimator is evaluated 10 times (thus 10 dots for each estimator). For example, in subfigure (a), the 10 blue dots scattered above ``Stiefel'' show the errors of 10 evaluations of $\wh{\H}{f}_k^\delta (x)$ with parameters $x = 0, \delta = 0.1, k = 100$. 
    More results for other values of $(x,\delta,k)$ can be found in Appendix \ref{app:additional}. \label{fig:hess-compare-main}}
\end{figure}

\subsection{Comparison with the Entry-wise Estimator} 

The estimator (Eq. \ref{eq:def-hess-est}) is also compared with the entry-wise estimator: 
\begin{align*}
    \wh{\H} f_{E}^\delta (x) := [ \wh{\H}_{ij} f_{E}^\delta (x) ]_{i,j\in [n]}, 
\end{align*}
where $ \wh{\H}_{ij} f_{E}^\delta (x) = \frac{1}{4\delta^2} \( f (x + \delta \e_i + \delta \e_j ) - f (x - \delta \e_i + \delta \e_j ) - f (x + \delta \e_i - \delta \e_j ) + f (x - \delta \e_i - \delta \e_j ) \) $ and $\e_i$ is the vector with $1$ on the $i$-th entry and $0$ on all other entries. The comparison results are summarized in Tables \ref{tab:hess-entry1} and \ref{tab:hess-entry2}.



\subsection{Empirical Verification of the Theorems} 

As discussed in the introduction (Figure \ref{fig:intro}), the error is highly aligned with the variance bound. Here we present the Hessian counterpart of Figure \ref{fig:intro}, with different values of $\delta$ and $x$. 
\begin{figure} 
    \centering
    \subfloat[$x = 0 $, $\delta = 0.001$]{\includegraphics[width = 0.4\textwidth]{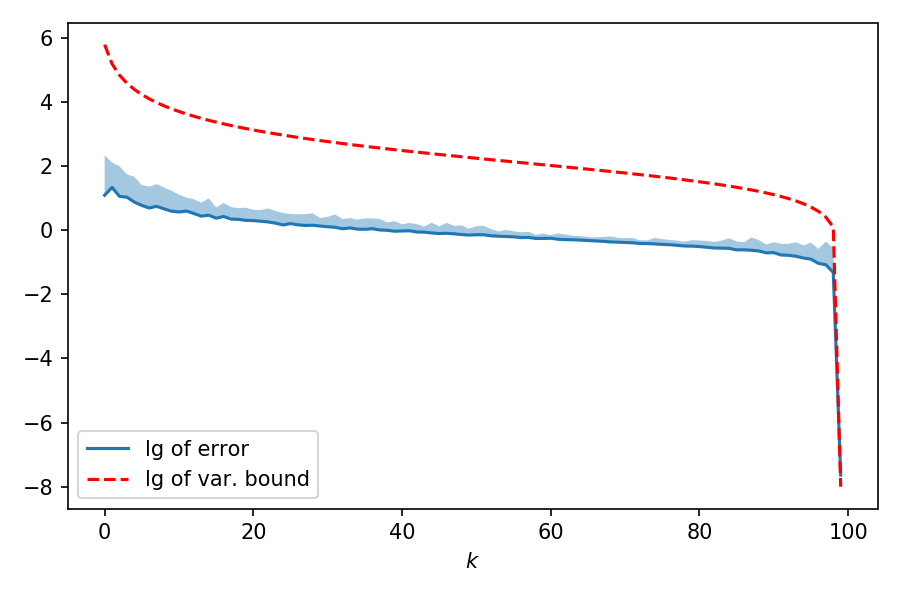}} 
    \subfloat[$x = \frac{\pi}{4} \mathbf{1}$, $\delta = 0.001$]{\includegraphics[width = 0.4\textwidth]{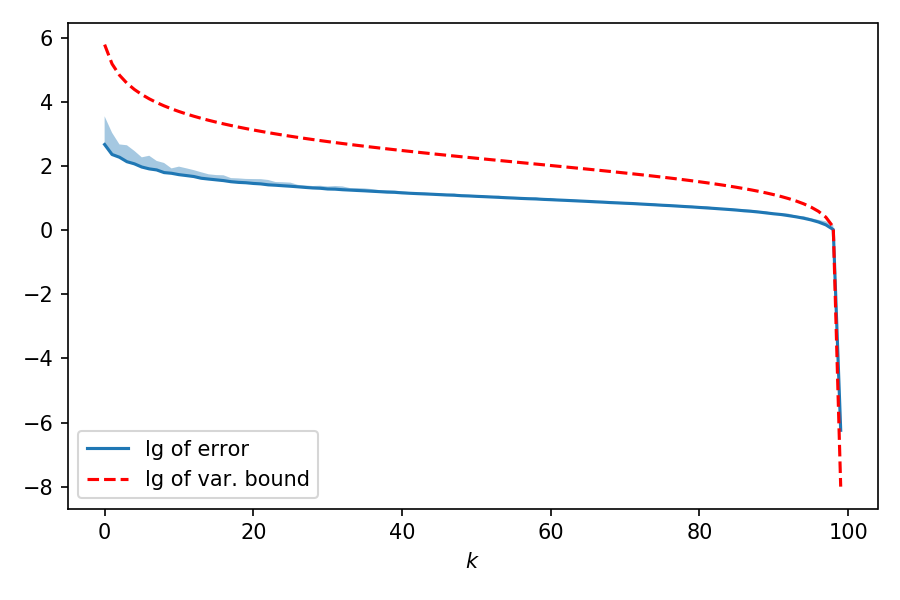}} \\
    \subfloat[$x = 0 $, $\delta = 0.01$]{\includegraphics[width = 0.4\textwidth]{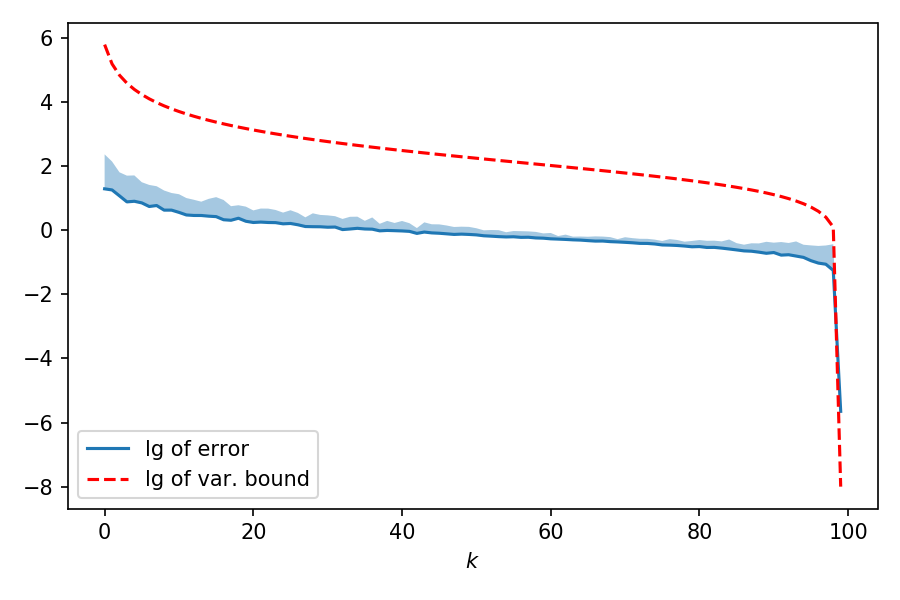}}
    \subfloat[$x = \frac{\pi}{4} \mathbf{1}$, $\delta = 0.01$]{\includegraphics[width = 0.4\textwidth]{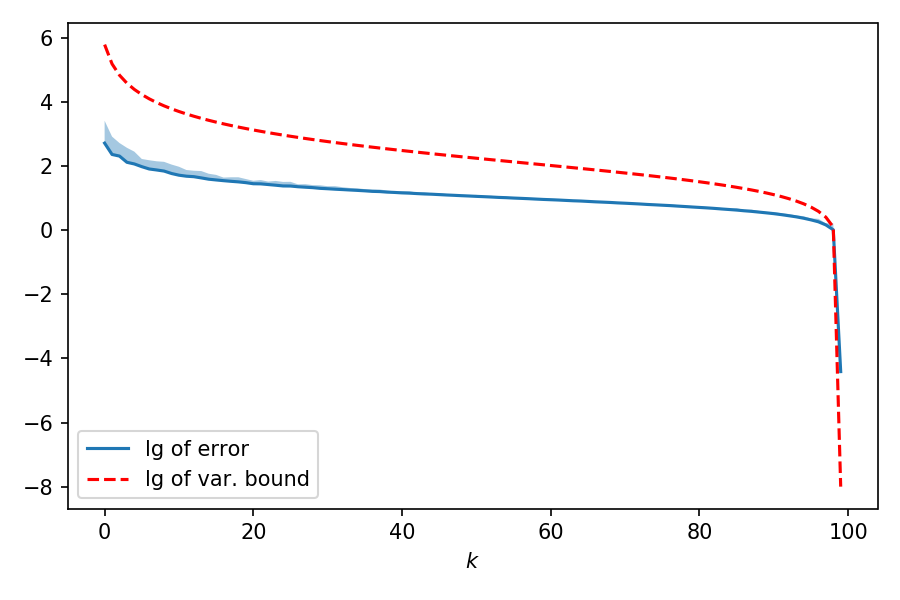}} \\ 
    \subfloat[$x = 0 $, $\delta = 0.1$]{\includegraphics[width = 0.4\textwidth]{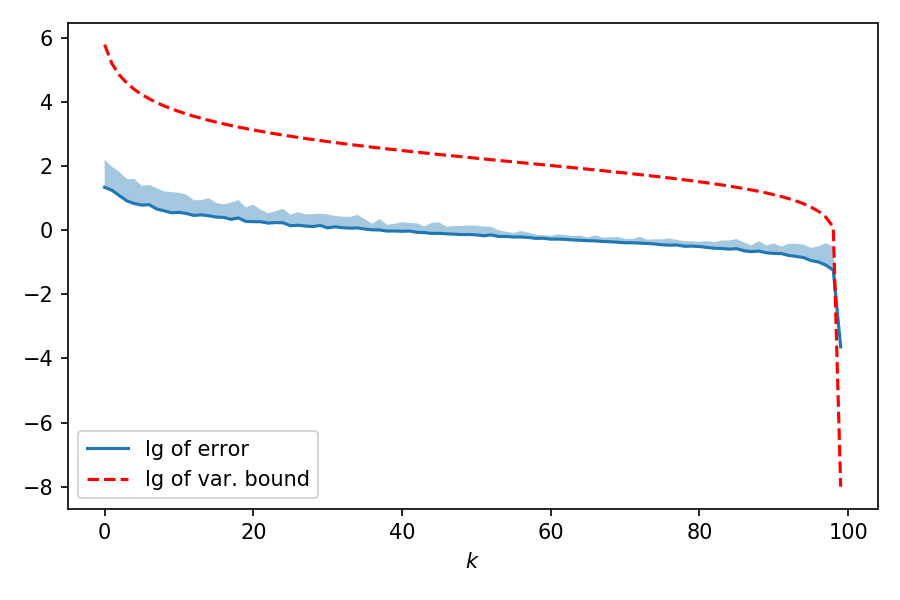}} 
    \subfloat[$x = \frac{\pi}{4} \mathbf{1}$, $\delta = 0.1$]{\includegraphics[width = 0.4\textwidth]{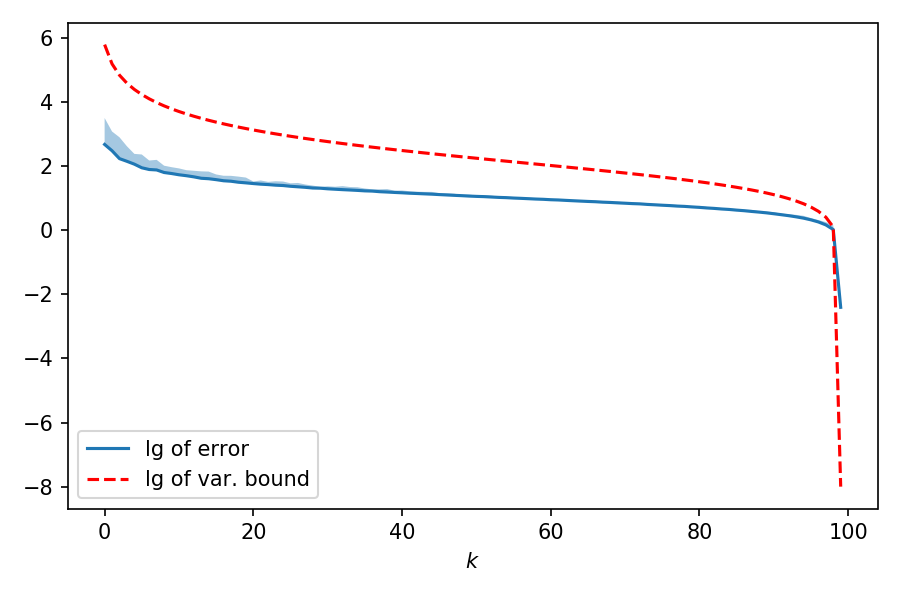}}
    \caption{Errors of Hessian estimators $ \wh{\H} f_k^\delta (x) $ as $k$ ranges from $1$ to $n$, in base-10 log-scale. Here $n = 100$. The underlying test function is $ f (x) = \exp ( (x_1 - 1) (x_2 + 2)) + \sum_{j=1}^{100} \sin (x_j)  $, where $x_j$ denotes the $j$-th component of vector $x$. The location for estimation $x$, and the finite difference granularity $\delta$ are labeled in the captions of the subfigures. The solid blue curve plots errors of the Hessian estimators in logarithmic scale (in terms of Frobenius norm), and is averaged over 10 runs. The shaded area above the solid curve shows 10 times standard deviation of the errors in logarithmic scale. 
    The dashed red curve, as a function of $k$, is $ c (k) = \lg \( \left\| \nabla^2 f (x) \right\|_F^2 \( \frac{n^2}{k^2} - 1 \) + 2 \delta^2 \left\| \nabla^2 f (x) \right\| \( \frac{n^4}{k^2} - n^2 \) + \| \nabla^2 f (x) \| \frac{n^4 \delta^4 }{k^2} \)$, which is the base-10 log of of variance bound for the Hessian estimators (up to constants).  \label{fig:trend-hess}} 
\end{figure}

%% file: conclusion.tex
\section{Discussions and Conclusion}

\subsection{Implications on Zeroth Order Optimization}
    In this paper, we focus on the statistical properties of the gradient/Hessian estimators. We briefly discuss the implications on zeroth order optimization algorithms before concluding the paper. 
    Consider the zeroth order gradient descent algorithm
    \begin{align}
        x_{t+1} = x_t - \eta \wh{\nabla} f_k^\delta (x_t) ,  \label{eq:gd}
    \end{align}
    where $ \wh{\nabla} f_k^\delta (x_t) $ is the gradient estimator, and $\eta$ is a learning rate. 
    One virtue of our variance reduction result is that it provides bounds on $ \E \[ \| \wh{\nabla} f_k^\delta (x_t) \|^2 \] $ and thus $ \E \[ \| x_{t+1} - x_t \|^2 \] $, in terms of $\delta$ and $k$. 
    Recall an $L$-smooth function $f$ satisfies  
    \begin{align*} 
        f (x_{t+1}) \le f (x_t) + \nabla f (x_t)^\top (x_{t+1} - x_t) + \frac{L}{2} \| x_{t+1} - x_t \|^2. 
    \end{align*} 
    With Theorem \ref{thm:grad-var}, we can take expectation on both sides of the above equation and get an in expectation bound on $ f (x_{t+1}) $ in terms of $ f (x_t) $, $ \nabla f (x_t) $, the estimation granularity $\delta$, and the number of function evaluation $k$. For $L$-smooth functions, larger $k$ means more function evaluations, but it also allows bigger learning rates (thus potentially fewer iterations). More detailed study of the downstream usage of the estimators are studied in a separate work \citep{wang2022almost}. 

\begin{remark}
    An intriguing fact is that zeroth order optimization
    with \emph{noisy} function evaluations and \emph{noise-free} function evaluations are two very different problems. In a \emph{noisy} environment, the \emph{minimax} lower bound states that given $t$, there exists a strongly convex function such that no zeroth order algorithm can converge faster than $\Omega ( n / \sqrt{t} ) $ \citep{jamieson2012query,shamir2013complexity}. On contrary, with \emph{noiseless} function evaluations, Nesterov and Spokoiny have shown that zeroth-order algorithms can achieve $O (n^2 / t^2)$ convergence rate \citep{nesterov2017random}. 
\end{remark}

\subsection{Conclusion}

We introduce gradient and Hessian estimators using random orthogonal frames sampled from the Stiefel manifold. The methods extend previous gradient/Hessian estimators based on spherical sampling \citep{flaxman2005online,wang2022hess}. Theoretically and empirically, we show that the variance of the estimation is reduced, and the accuracy of the estimation is improved. 
Refined bias bounds via Taylor expansion of higher orders are also provided. 

%% file: appendix.tex
\section{Proofs}

\subsection{Proof of Proposition \ref{prop:smooth}} 

\begin{proof} 

    For any $ \tau > 0 $ and $ x,v \in \R^n$ with $\| v \| = 1$, define $x_{\tau,v} = x + \tau v$. When $ \tau $ is small, Taylor's theorem and $(p,L)$-smoothness of $f$ give 
    \begin{align*}
        \left| \frac{ \partial^{p-1} f (x_{\tau,v} ) [ v ] - \partial^{p-1} f (x) [v] }{\tau} \right| 
        \overset{\textcircled{1}}{=} 
        \left| \partial^{p} f (z_{\tau,v}) \[ v \]  \right| , 
    \end{align*} 
    where $z_{\tau, v}$ depends on $\tau $ and $v$ and $ \lim_{\tau \rightarrow 0} z_{\tau, v} = x $ for any $v$. 
    
    Since the $f$ is $(p,L)$-smooth, for any $v \in \S^{n-1}$ (the unit sphere in $\R^n$), it holds that 
    \begin{align*} 
        \left| \frac{ \partial^{p-1} f (x_{\tau,v} ) [ v ] - \partial^{p-1} f (x) [v] }{\tau} \right| 
        \overset{\textcircled{2}}{\le} L. 
    \end{align*} 
    
    Combining \textcircled{1} and \textcircled{2} gives
    \begin{align*}
        \left| \partial^{p} f (z_{\tau,v}) \[ v \] \right| 
        \le 
        L, 
    \end{align*} 
    for any $v \in \S^{n-1}$ and any sufficiently small $\tau$. Thus for any $x \in \R^n$, we have 
    \begin{align*} 
        L 
        \ge 
        \sup_{v \in \S^{n-1}} \lim_{\tau \rightarrow 0} \left| \partial^{p} f (z_{\tau,v}) \[ v \] \right| 
        = 
        \left\| \partial^{p} f ({x}) \right\| . 
    \end{align*} 

\end{proof} 


\subsection{Proofs of Propositions \ref{prop:proj}, \ref{prop:fourth} and \ref{prop:mat}} 

We first prove the following proposition, which will be useful in proving Proposition \ref{prop:fourth}. 

\begin{proposition}
    \label{prop:p}
    Let $v$ be uniformly sampled from $\S^{n-1}$ ($n \ge 2$). It holds that 
    \begin{align*}
        \E \[ v_i^p \] = \frac{ (p-1) (p-3) \cdots 1 }{ n (n+2) \cdots (n+p - 2) } 
    \end{align*}
    for all $i = 1,2,\cdots,n$ and any positive even integer $p$. 
\end{proposition}

\begin{proof} 
    Let $ (r, \varphi_1, \varphi_2, \cdots, \varphi_{n-1}) $ be the spherical coordinate system. We have, for any $i = 1,2,\cdots, n$ and an even integer $p$, 
    \begin{align*} 
        \E \[ v_1^p \] 
        =& \; 
        \frac{1}{A_n} \int_{0}^{2\pi} \int_0^\pi \cdots \int_{0}^\pi \cos^p ( \varphi_1 ) \sin ^{n-2}(\varphi _{1})\sin ^{n-3}(\varphi _{2})\cdots \sin(\varphi _{n-2}) \,d\varphi _{1}\, d\varphi_{2} \cdots d \varphi_{n-1} , 
    \end{align*} 
    where $A_n$ is the surface area of $ \S^{n-1} $. 
    Let 
    \begin{align*}
        I (n,p) 
        := 
        \int_{0}^\pi \sin^n ( x ) \cos^p (x) \, dx. 
    \end{align*} 
    
    Clearly, $ I (n,p) = I (n, p - 2 ) - I ( n + 2, p - 2 ) $. By integration by parts, we have $ I ( n + 2, p - 2 ) = \frac{n+1}{p-1} I (n,p) $. The above two equations give $ I (n,p) = \frac{p-1}{n+p} I (n, p-2) $. 
    
    Thus we have $ \E \[ v_1^p \] = \frac{ I (n-2, p) }{ I ( n - 2, 0 ) } = \frac{ I (n-2, p) }{ I ( n - 2, p-2 ) } \frac{ I (n-2, p-2) }{ I ( n - 2, p-4 ) }  \cdots \frac{ I (n-2, 2) }{ I ( n - 2, 0 ) } = \frac{ (p-1) (p-3) \cdots 1 }{ n (n+2) \cdots (n+p - 2) } $. 
    We conclude the proof by by symmetry. 
    
\end{proof}

\begin{proof}[Proof of Proposition \ref{prop:proj}]
    Let $p = 2$ in Proposition \ref{prop:p}, we have $ \E \[ v_i^2 \] = \frac{1}{n} $ for all $i$. In addition, when $i \neq j$, $ \E \[  v_i v_j | v_j = a \] = 0 $ for any $a$. Thus $ \E \[  v_i v_j \] = 0 $ for all $i \neq j$. 
\end{proof} 

\begin{proof}[Proof of Proposition \ref{prop:fourth}]
    By Proposition \ref{prop:p}, we have 
    \begin{align*}
        \E \[ v_i^4 \] = \frac{3}{n(n+2)}. 
    \end{align*}
    By symmetry and that $ \(\sum_{i=1}^n v_i^2 \)^2 = 1 $, for any $i \neq j$, we have 
    \begin{align*}
        \E \[ v_i^2 v_j^2 \] = \frac{ 1 - \frac{3n}{n(n+2)} }{n(n-1)} = \frac{1}{n(n+2)} . 
    \end{align*}
    In addition, similar to the proof for $\E \[ v_i v_j \] = 0$ for $i\neq j$ (Proposition \ref{prop:proj}), we have 
    $ \E \[ v_i v_j v_k v_l \] = 0 $ for any $i,j,k,l \in \{ 1,2,\cdots, n \} $ and $i \notin \{ j,k,l \}$. 
\end{proof}

\subsection{Proof of Proposition \ref{prop:mat}}

\begin{proof}
    Let $\delta_k^l$ be the Kronecker delta. Using Einstein's notation, Proposition \ref{prop:proj} is equivalent to $ \E \[ v_i v^j \] = \frac{1}{n} \delta_i^j $. Thus we have 
    \begin{align*}
        \E \[ \( v^\top A w \) v w^\top  \]
        = 
        \E \[ v_k A_l^k w^l v^i w_j \] 
        = 
        \frac{1}{n^2} A_l^k \delta_{k}^i \delta_j^l 
        = 
        \frac{1}{n^2} A_j^i, 
    \end{align*} 
    which concludes the proof. 
    
\end{proof}




\subsection{Proof of Theorem \ref{thm:grad-bias}(a)} 
\label{app:f-delta} 
The original proof was due to \citet{flaxman2005online}. Here we present a proof via the divergence theorem. This version of proof will also assist the proof for Theorem \ref{thm:hess-bias}(a). 
Define 
\begin{align*}
    f^\delta (x) = \frac{1}{\delta^n V_n} \int_{\B^n} f (x + \delta v) \, dv , 
\end{align*}
where $V_n$ is the volume of $ \B^n $. 

\begin{lemma} 
    \label{lem:grad-bias}
    Let $f$ be a smooth function. For any $k \in \mathbb{N}_+$ and $v_1, v_2, \cdots, v_k \in \R^n$ sampled from the Stiefel sampling process, it holds that 
    \begin{align*}
        \frac{n}{\delta} \E \[ \nabla \wh{f}_k^\delta ( x ) \] 
        = \nabla {f}^\delta (x) , 
    \end{align*}
    for all $x \in \R^n$, and all $k = 1,2,\cdots, n$. 
\end{lemma} 


\begin{proof}
    Without loss of generality, let $x = 0$. Derivations for other values of $x$ follows similar arguments. 
    Let $u$ be an arbitrary unit vector in $ \R^n $, and let $U$ be the constant vector field generated by $u$. Let $X = f U$, which is the vector field $U$ multiplied by the function values of $f$. Apply the divergence theorem to this vector field and the region enclosed by $ \delta \S^{n-1} $ gives 
    \begin{align*} 
        \int_{ \delta \B^n } \nabla \cdot X \, dV 
        = 
        \int_{ \delta \S^{n-1} } X \cdot d \vec{S} . 
    \end{align*} 
    The above equation is equivalent to 
    \begin{align*}
        \left< \int_{ v \in \delta \B^n } \nabla f ( v) \, dv , u \right>
        = 
        \left< \int_{ v \in \delta \S^{n-1} } f ( v) \frac{v}{\| v \|} \, d v , u \right> , \quad \forall x \in \R^n
    \end{align*} 


    By the dominated convergence theorem (or Leibniz integral rule), we can exchange the gradient and the integral to get 
    \begin{align*}
        \left< \nabla \int_{ v \in \delta \B^n }  f (v) \, dv , u \right>
        {=} 
        \left< \int_{ v \in \delta \S^{n-1} } f (v) \frac{v}{\| v \|} \, d v , u \right> , \quad \forall u \in \R^n , 
    \end{align*}

    or equivalently 
    \begin{align*} 
        \left< \delta^{n} V_{n} \nabla {f}^\delta (0) , u \right> 
        \overset{\textcircled{1}}{=}
        \left< \int_{v \in \S^{n-1}} f ( \delta v) v \, dv , u \right> , 
        \quad \forall u \in \R^n , 
    \end{align*} 
    where on the right-hand-side a change of integration from $\delta \S^{n-1}$ to $ \S^{n-1} $ is used. 

    By Proposition \ref{prop:uniform}, it holds that, for any $v_i$ generated from the {Stiefel} sampling process, 
    \begin{align*}
        \E \[ f (\delta v_i ) v_i \] \overset{\textcircled{2}}{=} \frac{1}{\delta^{n-1} A_n } \int_{v_i \in \S^{n-1}} f  (\delta v_i ) v_i \, d v_i. 
    \end{align*}
    
    Combining \textcircled{1} and \textcircled{2} gives
    \begin{align*} 
        \left< \frac{n}{\delta} \E \[ f (\delta v_i ) v_i \], u \right> = \left< \nabla {f}^\delta (0), u \right>. 
    \end{align*} 
    
    The above equation concludes the proof since $ u $ is an arbitrary (unit) vector in $ \R^n $. 
\end{proof} 

\begin{proof}[Proof of Theorem \ref{thm:grad-bias}(a)]
    Let $V_n$ be the volume of $ \B^n $, and let $ A_n $ be the area of $\S^{n-1}$. If $f$ is $(2,L_1)$-smooth, we have 
    \begin{align*}
        \left\| \nabla f^\delta (x) - \nabla f (x) \right\| 
        \le 
        \int_{\delta \B^n} \frac{1}{\delta^n V_n} \left\|  \nabla  f^\delta (x + v) - \nabla f (x) \right\| \, d v 
        \le \int_{0}^\delta \frac{1}{\delta^n V_n} L_1 r^n A_n \, d r = \frac{L_1 n \delta }{n+1} , 
    \end{align*} 
    where the last equation uses $ A_n = n V_n $. Combine the above result with Lemma \ref{lem:grad-bias} finishes the proof. 
\end{proof}

\subsection{Proof of Theorem \ref{thm:hess-bias}(a)} 

Define 
\begin{align*}
    \wt{f}^\delta (x) = \frac{1}{\delta^n V_n}  \frac{1}{\delta^n V_n} \int_{\B^n} \int_{\B^n} f (x + \delta v + \delta w) \, dw \, dv , 
\end{align*}
where $V_n$ is the volume of $ \B^n $.

\begin{lemma} 
    \label{lem:hess-bias}
    Let $f$ be twice continuously differentiable. It holds that 
    \begin{align*}
        \frac{n^2}{\delta^2} \E \[ \wh{\H} f_k^\delta (x) \] = \nabla^2 \wt{f}^\delta (x),  
    \end{align*} 
    for any $ x \in \R^n $ and $ k = 1,2,\cdots,n $. 
\end{lemma} 

\begin{proof}
    Without loss of generality, we consider $x = 0$. Also, by linearity of expectation, it suffices to prove that for any $ v,w $ independently uniformly sampled from $\S^{n-1}$, we have 
    \begin{align*}
        \E_{v,w \overset{i.i.d.}{\sim} \S^{n-1}} \[ f (\delta v + \delta w) v w^\top \] 
        = 
        \nabla^2 \wt{f}^\delta (0). 
    \end{align*}


    Let $ z $ be an arbitrary unit vector in $\R^n$. Let $g_z (x) = \nabla_z {f}^\delta (x) = \left< \nabla {f}^\delta (x), z \right>$, where $f^\delta$ is defined in Appendix \ref{app:f-delta}. 
    Let $ u $ be an arbitrary unit vector in $\R^n$, and let $ U $ be the constant vector field generated by $u$. Let $ X_z := g_z U $ be the vector field $U$ multiplied by the function values of $g_z$. Apply the divergence theorem to this vector field gives 
    \begin{align*}
        \int_{\delta \B^n} \nabla \cdot X_z \, dV = \int_{\delta \S^{n-1}} X_z \cdot d \vec{S}. 
    \end{align*}

    The above equation is equivalent to 
    \begin{align*} 
        \left< \int_{\delta \B^n} \nabla g_z (v) \, dv , u \right> 
        = 
        \left< \int_{\delta \S^{n-1}} g_z (v) \frac{v}{\| v \|} \, dv , u \right> . 
    \end{align*}

    By the dominated convergence theorem, we can exchange the gradient and the integral to get 
    \begin{align*}
        \left< \nabla \int_{ \delta \B^n }  g_z (v) \, dv , u \right>
        {=} 
        \left< \int_{ \delta \S^{n-1} } g_z (v) \frac{v}{\| v \|} \, d v , u \right> , \quad \forall u \in \R^n ,   
    \end{align*}
    which is equivalent to 
    \begin{align*}
        \left< \nabla \int_{ \delta \B^n }  \left< \nabla {f}^\delta (v), z \right> \, dv , u \right> 
        \overset{\textcircled{1}}{=} 
        \left< \int_{ \delta \S^{n-1} } \left< \nabla {f}^\delta (v), z \right> \frac{v}{\| v \|} \, d v , u \right> , \quad \forall u \in \R^n . 
    \end{align*}

    By Lemma \ref{lem:grad-bias}, the right-hand-side of \textcircled{1} is 
    \begin{align*} 
        &\; \left< \int_{ \delta \S^{n-1} } \left< \nabla {f}^\delta (v), z \right> \frac{v}{\| v \|} \, d v , u \right> \\ 
        =& \; 
        \left< \int_{ \S^{n-1} } \left< \frac{n}{\delta} \int_{ \S^{n-1} } f (\delta v + \delta w) w \, dw, z \right> v \, d v , u \right> \\ 
        =& \; 
        \frac{n}{\delta} \delta^{n-1} A_n u^\top \E_{v,w} \[ f (\delta v + \delta w) v w^\top \] z, 
    \end{align*} 
    where $A_n$ is the surface area of $\S^{n-1}$. 
    
    By dominated convergence theorem, we can interchange the integral and the directional derivative. Thus the left-hand-side of \textcircled{1} is 
    \begin{align*} 
        \left< \nabla \int_{ \delta \B^n }  \left< \nabla {f}^\delta (v), z \right> \, dv , u \right> 
        =
        \delta^n V_n u^\top \nabla^2 \wt{f}^\delta (0) z, 
    \end{align*}
    where $V_n$ is the volume of $ \B^n $. 
    
    Since $ A_n = n V_n$, collecting terms gives
    \begin{align*}
        u^\top \( \nabla^2 \wt{f}^\delta (0) \) z
        = 
        \frac{n^2}{\delta^2} u^\top \E_{v,w} \[ f (\delta v + \delta w) v w^\top \] z, 
    \end{align*} 
    
    We conclude the proof by noting that the above is true for any (unit) vectors $u$ and $z$. 
\end{proof}

\begin{proof}[Proof of Theorem \ref{thm:hess-bias}(a)]
    
    Let $V_n$ be the volume of $ \B^n $, and let $ A_n $ be the area of $\S^{n-1}$. If $f$ is $(3,L_2)$-smooth, we have 
    \begin{align*}
        \left\| \nabla^2 \wt{f}^\delta (x) - \nabla^2 f (x) \right\| 
        \le& \;  
        \int_{\delta \B^n} \int_{\delta \B^n} \frac{1}{\delta^{2n} V_n^2} \left\|  \nabla^2  \wt{f}^\delta (x + v + w) - \nabla^2 f (x) \right\| \, dw \, d v \\ 
        \le& \;  
        \int_{0}^\delta \int_{0}^\delta \frac{1}{\delta^{2n} V_n^2 } L_2 (r + s ) r^{n-1} s^{n-1} A_n^2 \, d r\, d s \\
        =& \; 
        \frac{2n L_2 \delta }{n+1}, 
    \end{align*} 
    where the last equation uses $ A_n = n V_n $. Combine the above result with Lemma \ref{lem:hess-bias} finishes the proof. 
\end{proof}

\section{Supplementary Figures} 
\label{app:additional}

\begin{figure}[h]
    \centering
    \subfloat[$ x = 0, \delta = 0.1, k = 100 $]{\includegraphics[width = 0.45\textwidth]{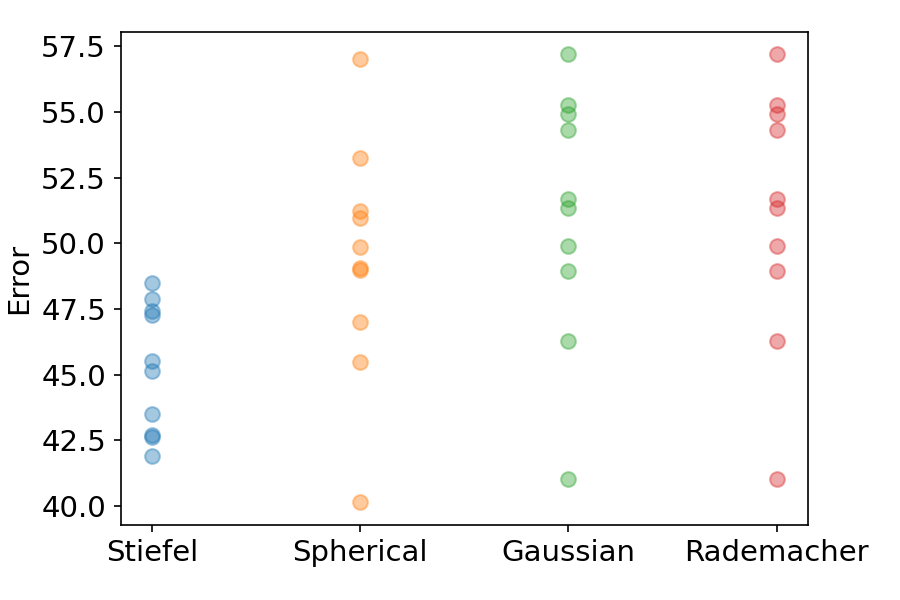}}
    \subfloat[$ x = 0, \delta = 0.1, k = 200 $]{\includegraphics[width = 0.45\textwidth]{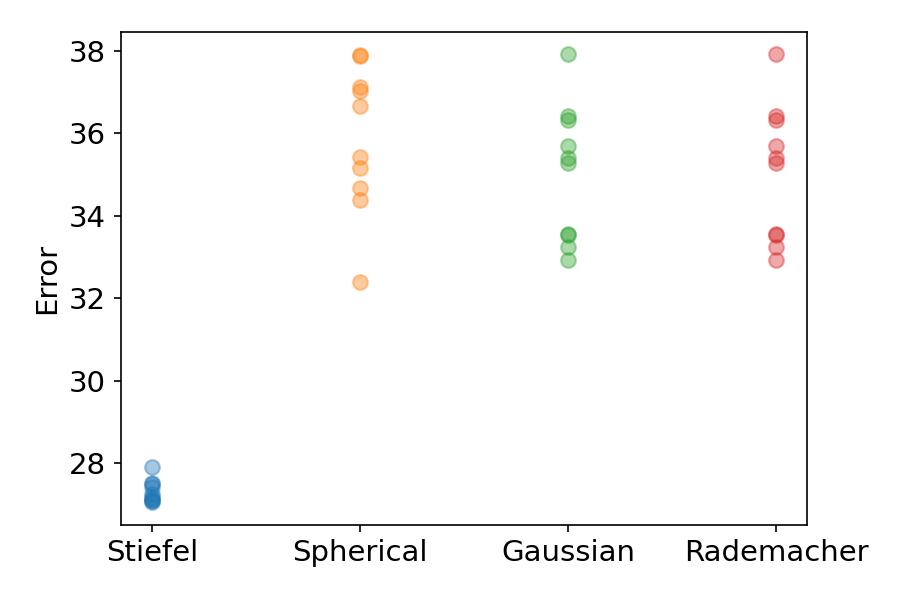}} \\
    \subfloat[$ x = 0, \delta = 0.01, k = 100 $]{\includegraphics[width = 0.45\textwidth]{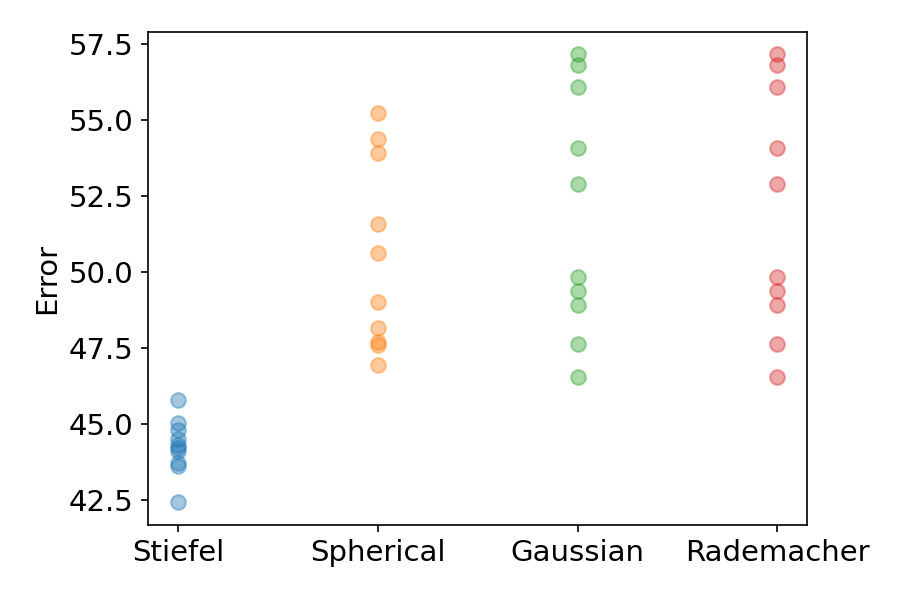}}
    \subfloat[$ x = 0, \delta = 0.01, k = 200 $]{\includegraphics[width =  0.45\textwidth]{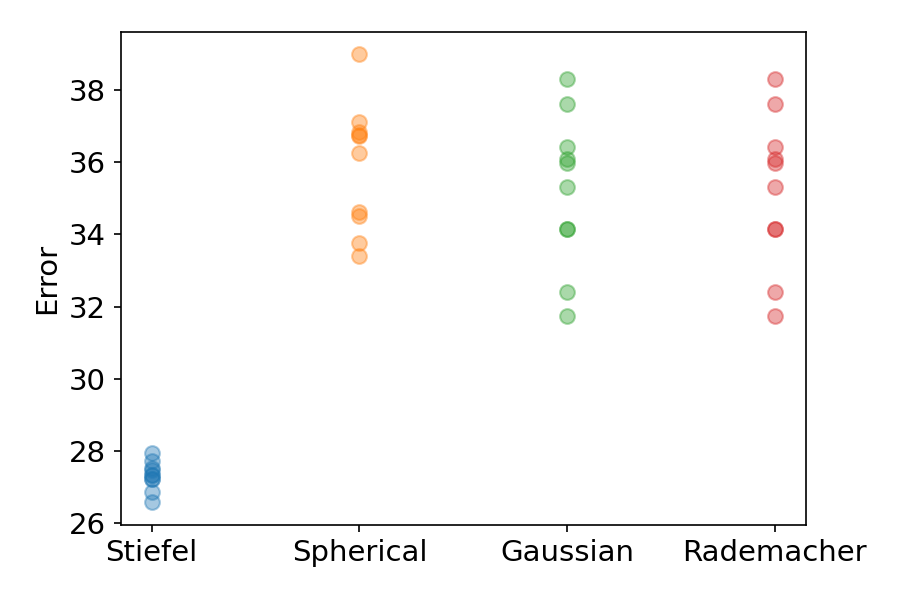}} \\ 
    \subfloat[$ x = 0, \delta = 0.001, k = 100 $]{\includegraphics[width = 0.45\textwidth]{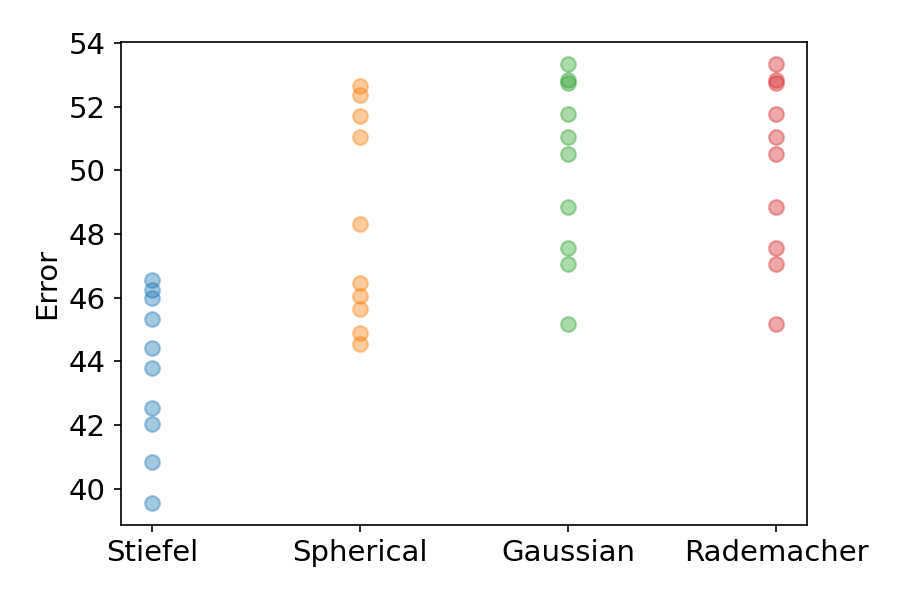}}
    \subfloat[$ x = 0, \delta = 0.001, k = 200 $]{\includegraphics[width = 0.45\textwidth]{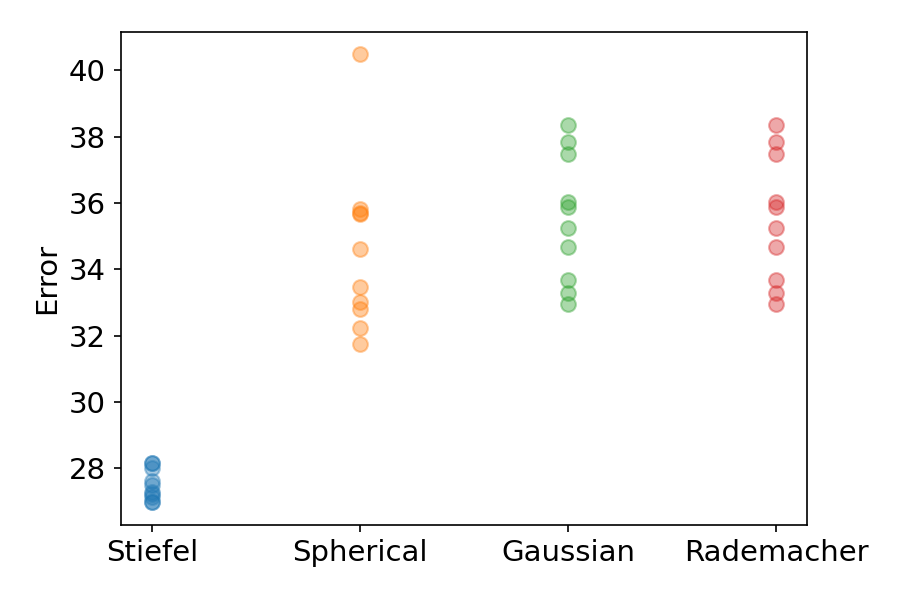}} 
    \caption{Errors of $ \wh{\nabla} f_k^\delta (x) $, $ \wh{\nabla} f_{k,B}^\delta (x) $, $ \wh{\nabla} f_{k,G}^\delta (x) $, $ \wh{\nabla} f_{k,R}^\delta (x) $ on test function (Eq. \ref{eq:test}). Each subfigure corresponds to a different combination of the location for estimation $x$, the finite difference granularity $ \delta $, and number of random directions $k$. See caption of Figure \ref{fig:grad-compare-main} for detailed illustration. 
    }
\end{figure}

\begin{figure}[h] 
    \centering
    \subfloat[$ x = \frac{\pi}{4} \mathbf{1}, \delta = 0.1, k = 100 $]{\includegraphics[width = 0.45\textwidth]{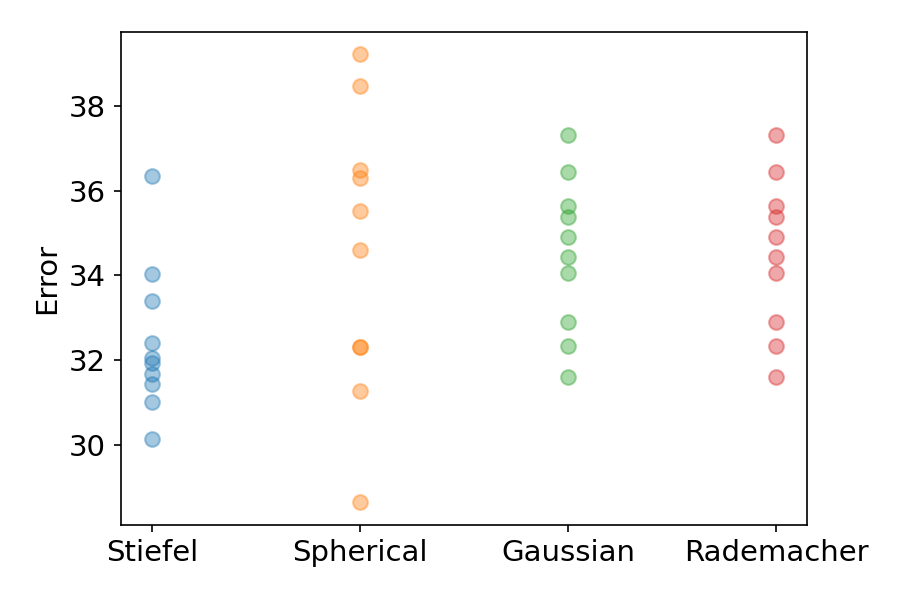}}
    \subfloat[$ x = \frac{\pi}{4} \mathbf{1}, \delta = 0.1, k = 200 $]{\includegraphics[width = 0.45\textwidth]{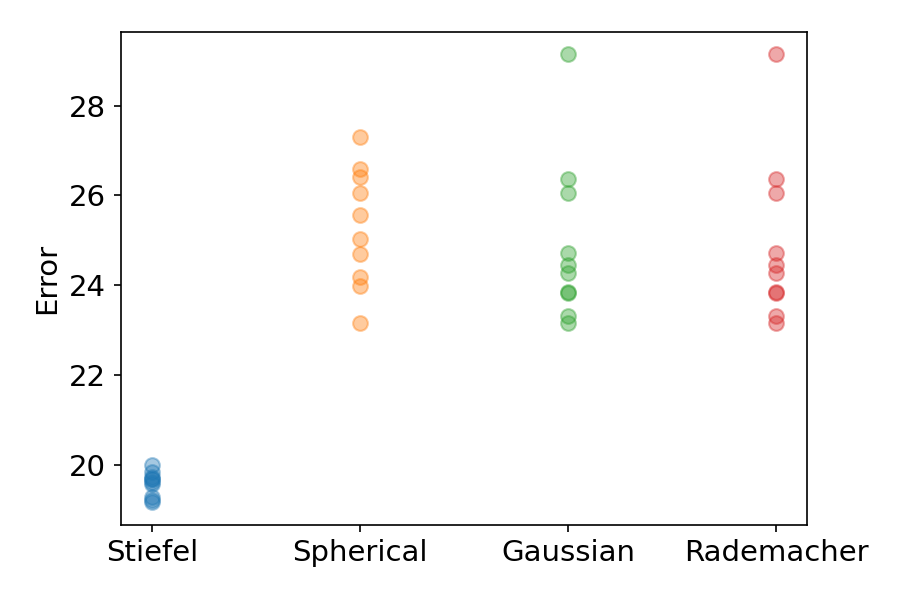}} \\
    \subfloat[$ x = \frac{\pi}{4} \mathbf{1}, \delta = 0.01, k = 100 $]{\includegraphics[width = 0.45\textwidth]{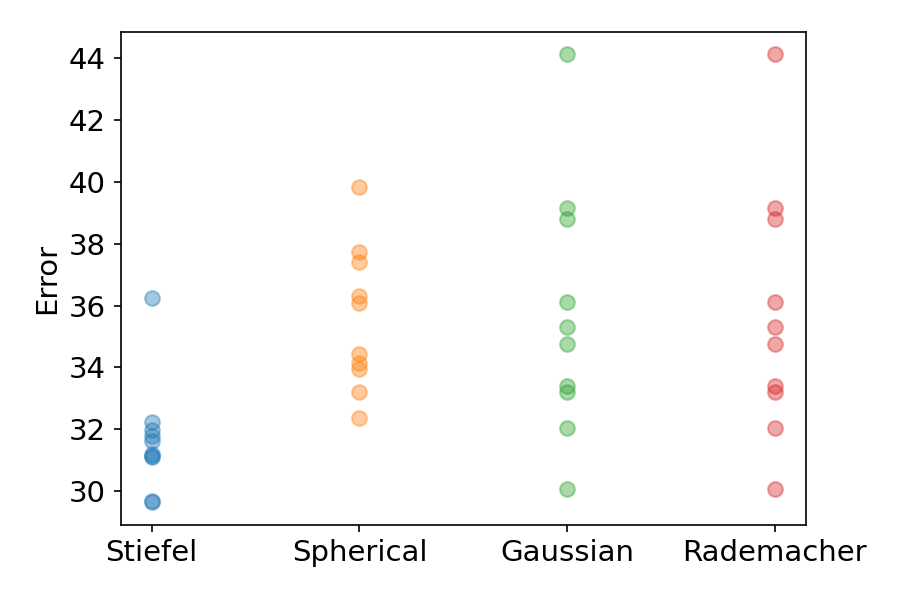}}
    \subfloat[$ x = \frac{\pi}{4} \mathbf{1}, \delta = 0.01, k = 200 $]{\includegraphics[width =  0.45\textwidth]{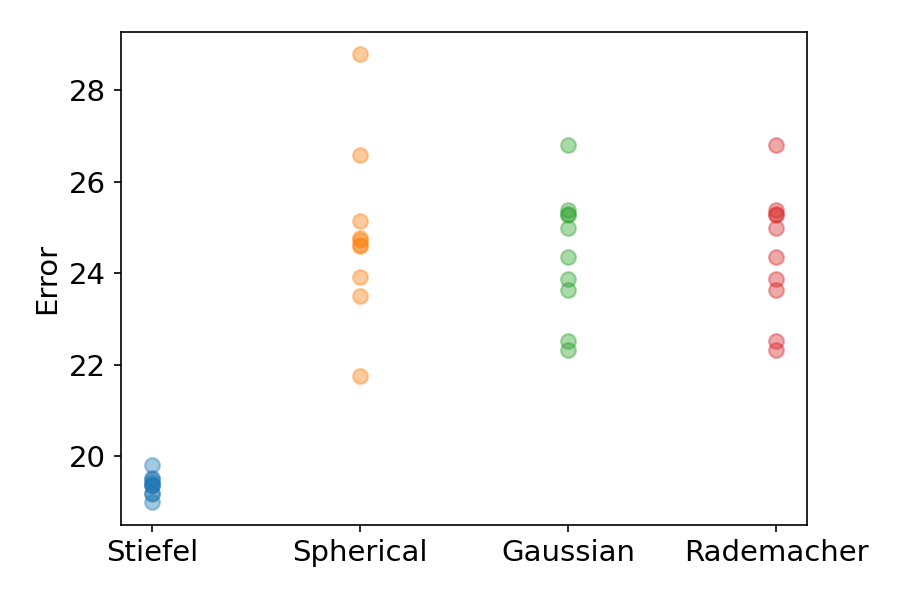}} \\ 
    \subfloat[$ x = \frac{\pi}{4} \mathbf{1}, \delta = 0.001, k = 100 $]{\includegraphics[width = 0.45\textwidth]{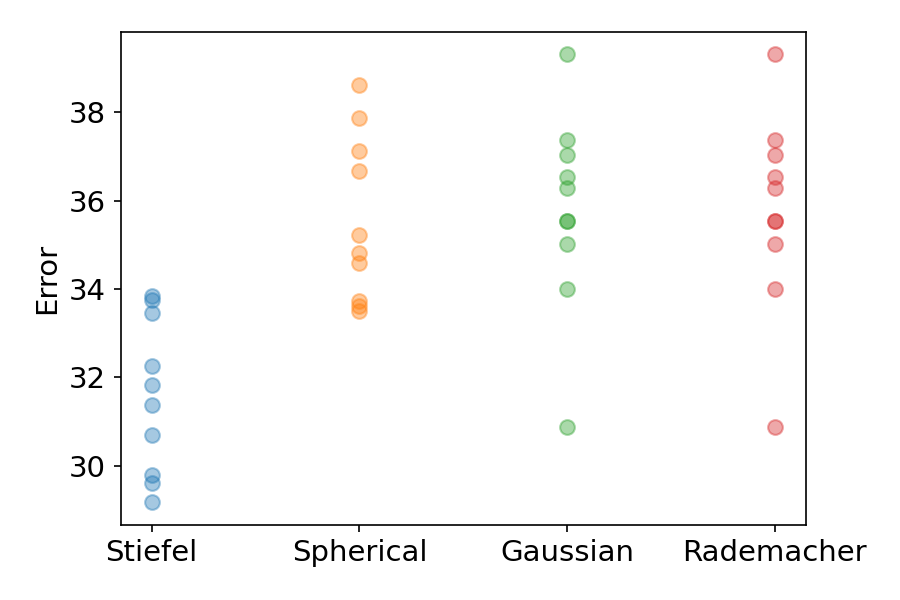}}
    \subfloat[$ x = \frac{\pi}{4} \mathbf{1}, \delta = 0.001, k = 200 $]{\includegraphics[width = 0.45\textwidth]{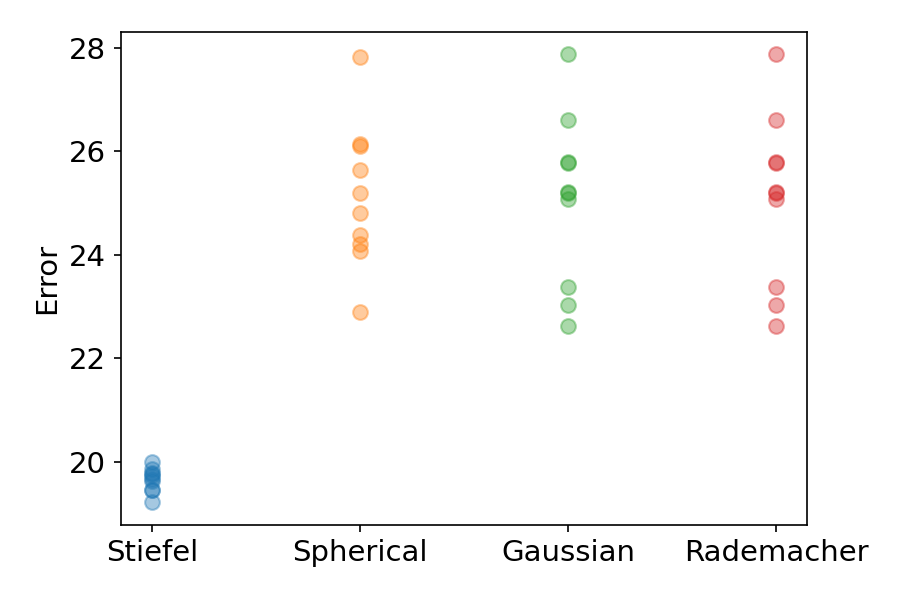}} 
    \caption{Errors of $ \wh{\nabla} f_k^\delta (x) $, $ \wh{\nabla} f_{k,B}^\delta (x) $, $ \wh{\nabla} f_{k,G}^\delta (x) $, $ \wh{\nabla} f_{k,R}^\delta (x) $ on test function (Eq. \ref{eq:test}). Each subfigure corresponds to a different combination of the location for estimation $x$, the finite difference granularity $ \delta $, and number of random directions $k$. See caption of Figure \ref{fig:grad-compare-main} for detailed illustration. 
    }
\end{figure}

\begin{figure}[h] 
    \centering
    \subfloat[$ x = \frac{\pi}{4} \mathbf{1}, \delta = 0.1, k = 300 $]{\includegraphics[width = 0.45\textwidth]{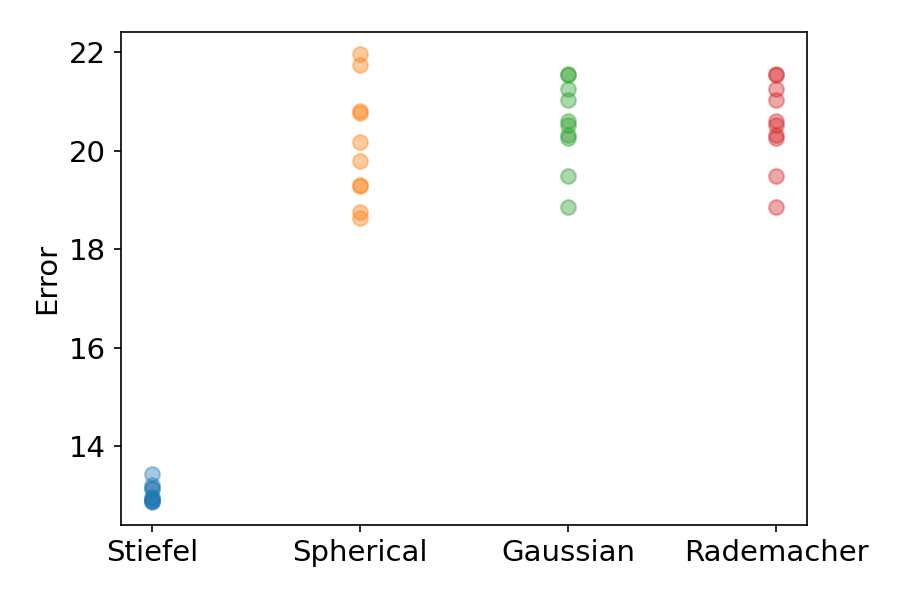}}
    \subfloat[$ x = \frac{\pi}{4} \mathbf{1}, \delta = 0.1, k = 400 $]{\includegraphics[width = 0.45\textwidth]{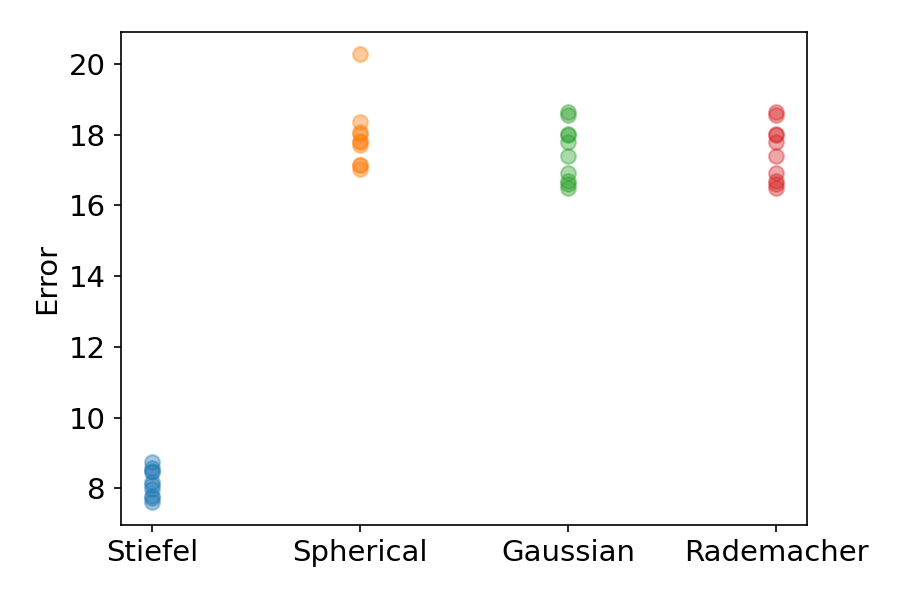}} \\
    \subfloat[$ x = \frac{\pi}{4} \mathbf{1}, \delta = 0.01, k = 300 $]{\includegraphics[width = 0.45\textwidth]{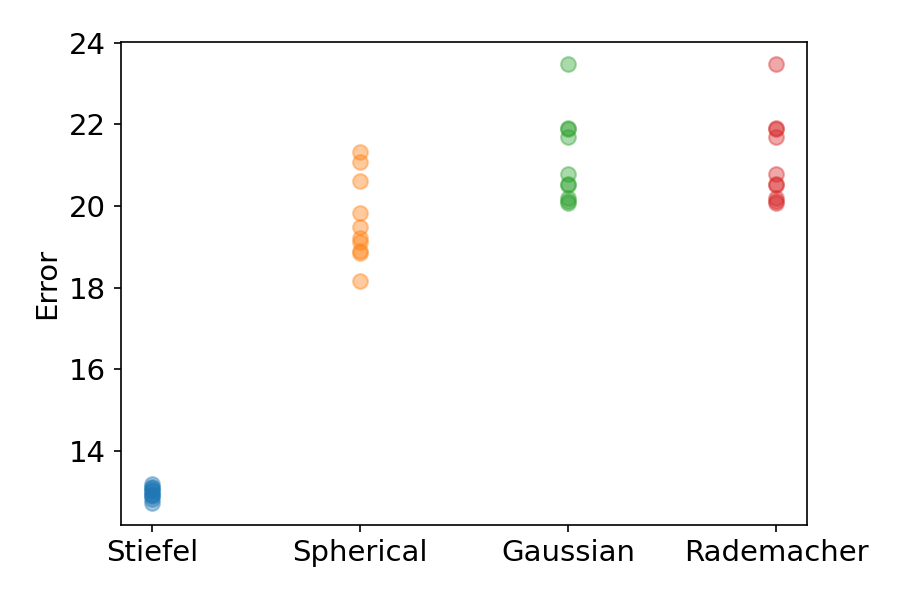}}
    \subfloat[$ x = \frac{\pi}{4} \mathbf{1}, \delta = 0.01, k = 400 $]{\includegraphics[width =  0.45\textwidth]{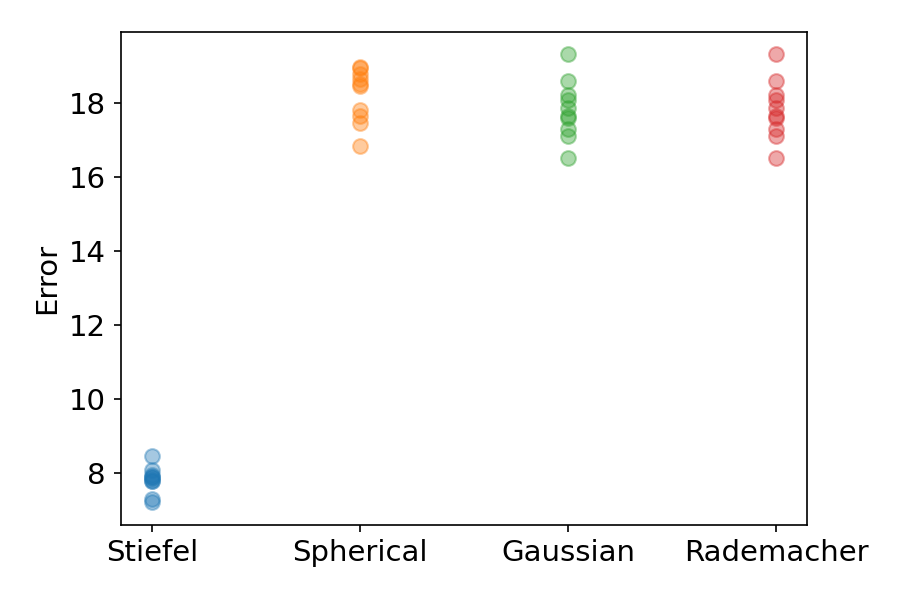}} \\ 
    \subfloat[$ x = \frac{\pi}{4} \mathbf{1}, \delta = 0.001, k = 300 $]{\includegraphics[width = 0.45\textwidth]{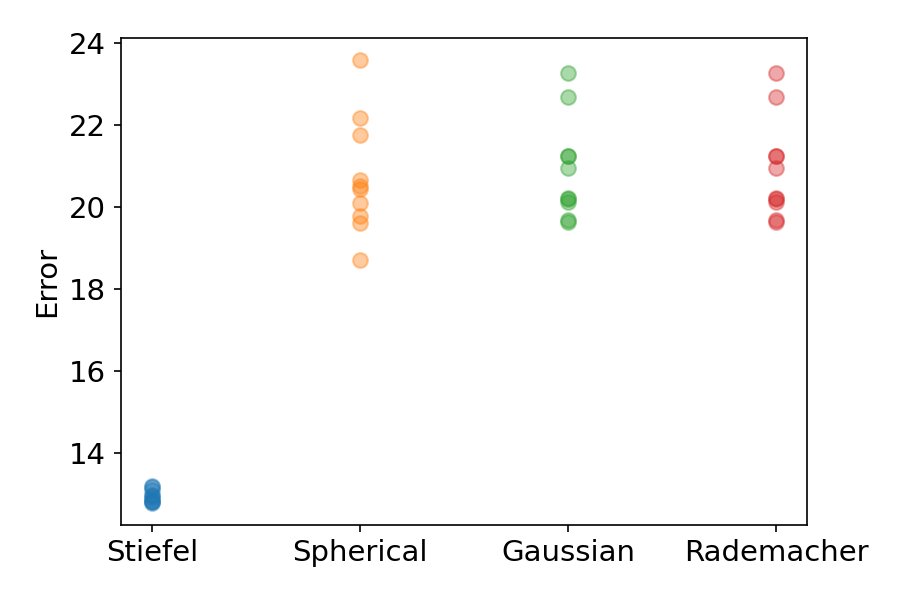}}
    \subfloat[$ x = \frac{\pi}{4} \mathbf{1}, \delta = 0.001, k = 400 $]{\includegraphics[width = 0.45\textwidth]{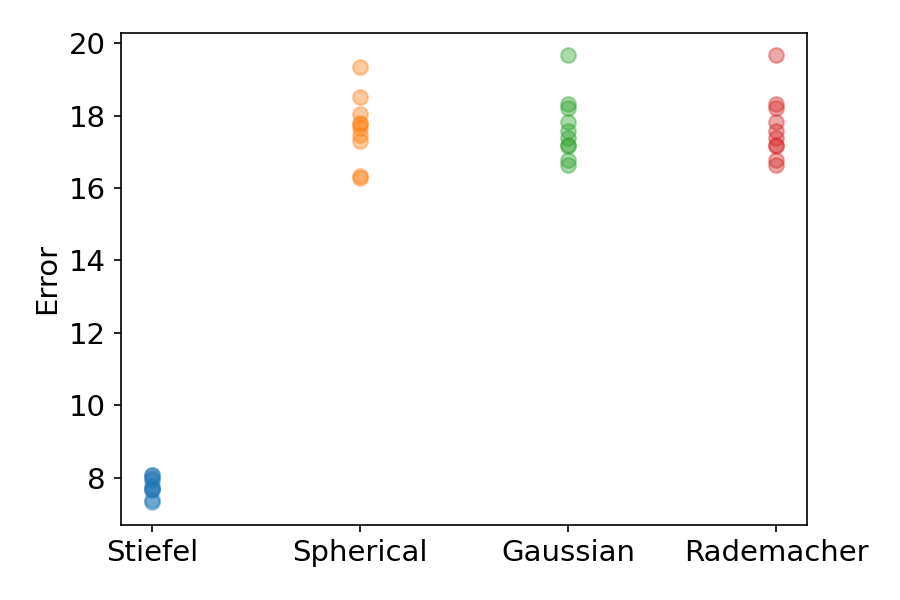}} 
    \caption{Errors of $ \wh{\nabla} f_k^\delta (x) $, $ \wh{\nabla} f_{k,B}^\delta (x) $, $ \wh{\nabla} f_{k,G}^\delta (x) $ on test function (Eq. \ref{eq:test}). Each subfigure corresponds to a different combination of the location for estimation $x$, the finite difference granularity $ \delta $, and number of random directions $k$. See caption of Figure \ref{fig:grad-compare-main} for detailed illustration. 
    }
\end{figure}

\begin{figure}
    \centering
    \subfloat[$ x = 0, \delta = 0.1, k = 100 $]{\includegraphics[scale = 0.45]{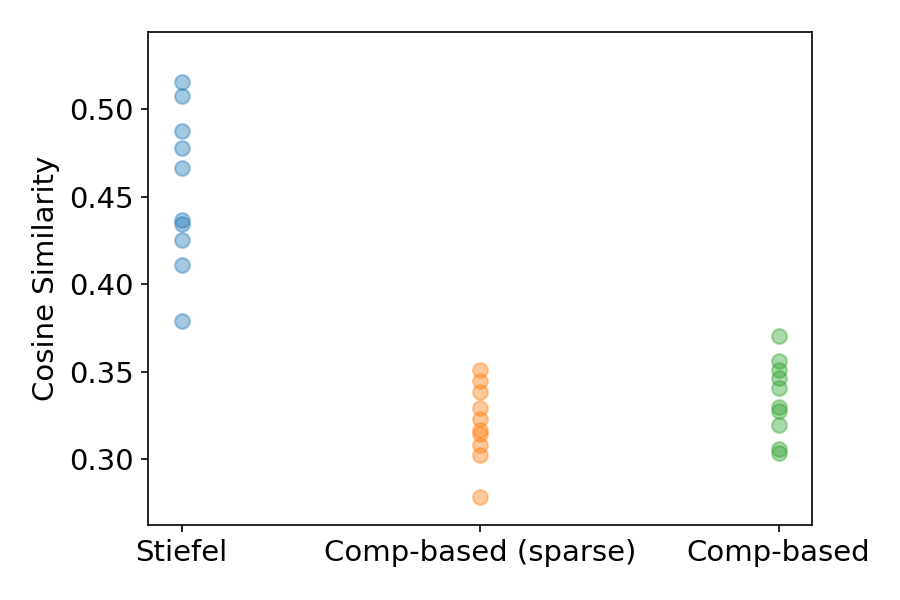}} 
    \subfloat[$ x = 0, \delta = 0.1, k = 200 $]{\includegraphics[scale = 0.45]{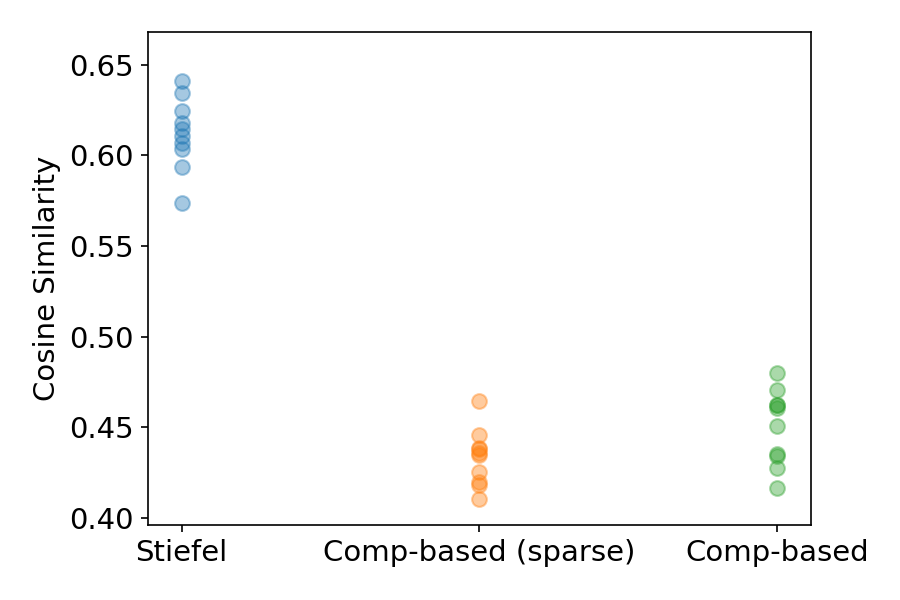}} \\
    \subfloat[$ x = 0, \delta = 0.01, k = 100 $]{\includegraphics[scale = 0.45]{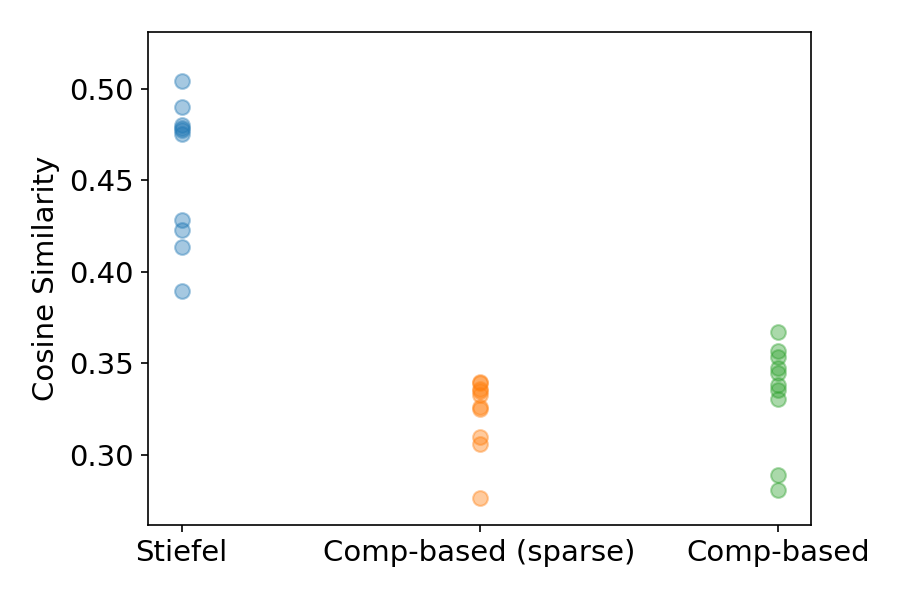}} 
    \subfloat[$ x = 0, \delta = 0.01, k = 200 $]{\includegraphics[scale = 0.45]{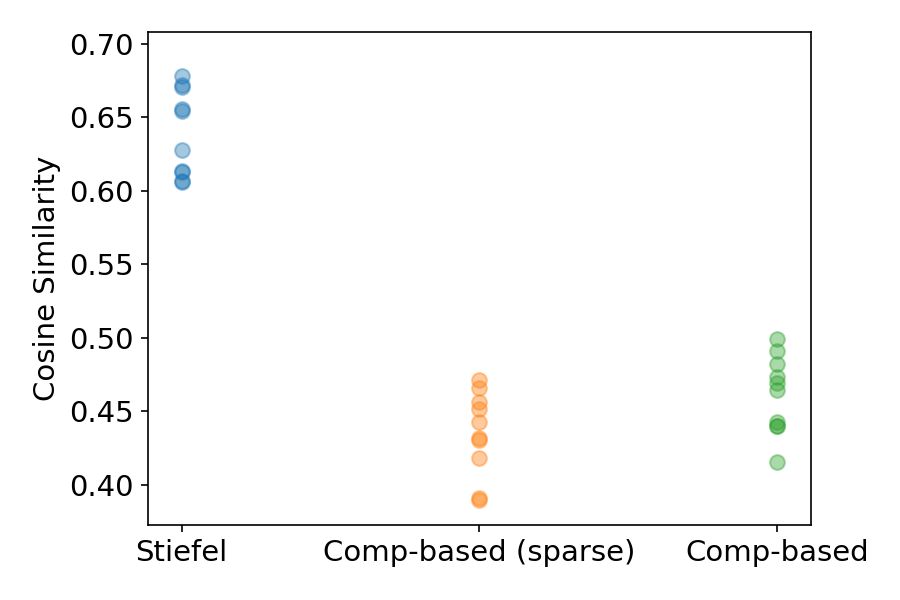}} \\
    \subfloat[$ x = 0, \delta = 0.001, k = 100 $]{\includegraphics[scale = 0.45]{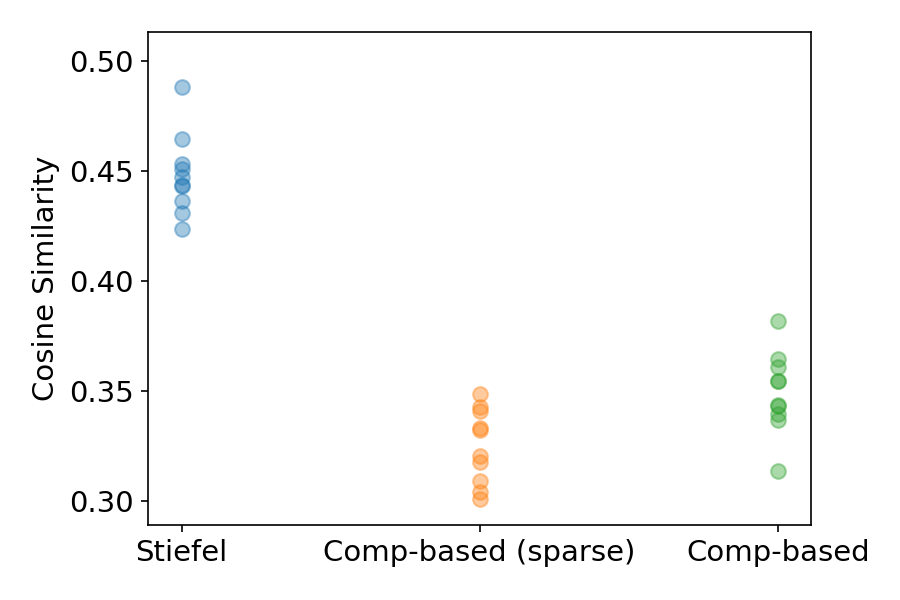}} 
    \subfloat[$ x = 0, \delta = 0.001, k = 200 $]{\includegraphics[scale = 0.45]{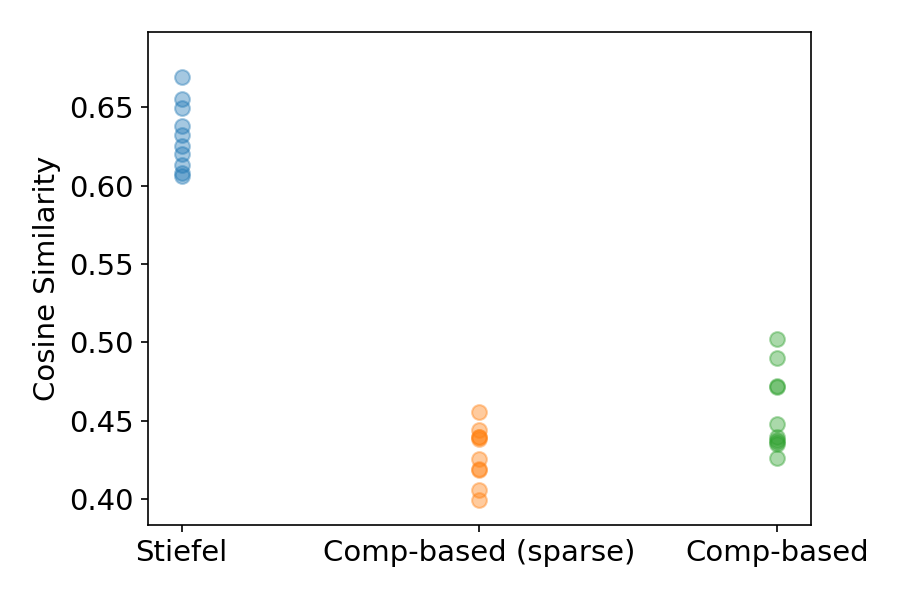}} 
    \caption{Cosine similarity between the gradient estimation and the true gradient, using the test function defined in (Eq. \ref{eq:test}). 
    Each subfigure corresponds to a different combination of the location for estimation $x$, the finite difference granularity $ \delta $, and number of random directions $k$. See the caption of Figure \ref{fig:grad-compare-main2} for detailed illustration.} 
\end{figure}

\begin{figure}
    \centering
    \subfloat[$ x = \frac{\pi}{4} \mathbf{1}, \delta = 0.1, k = 100 $]{\includegraphics[scale = 0.45]{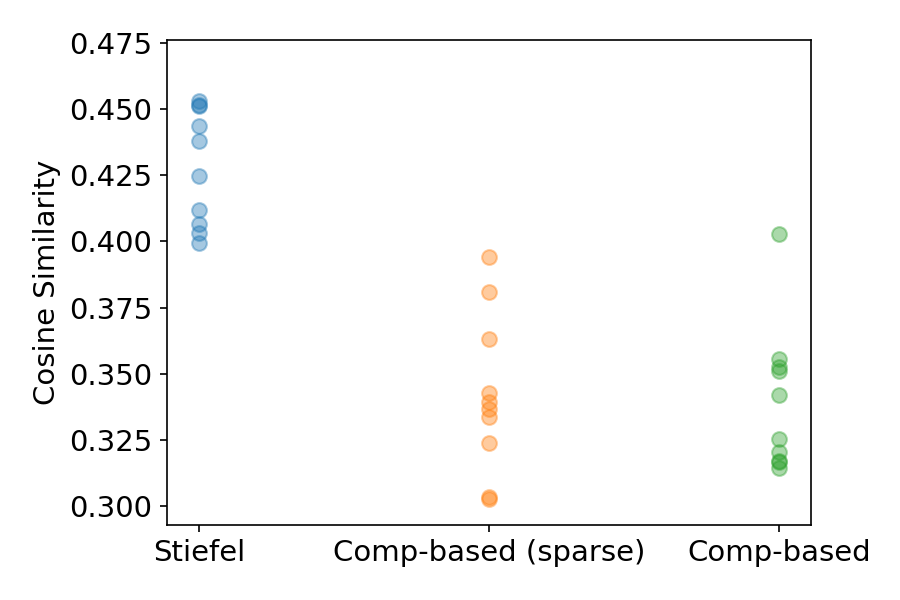}} 
    \subfloat[$ x = \frac{\pi}{4} \mathbf{1}, \delta = 0.1, k = 200 $]{\includegraphics[scale = 0.45]{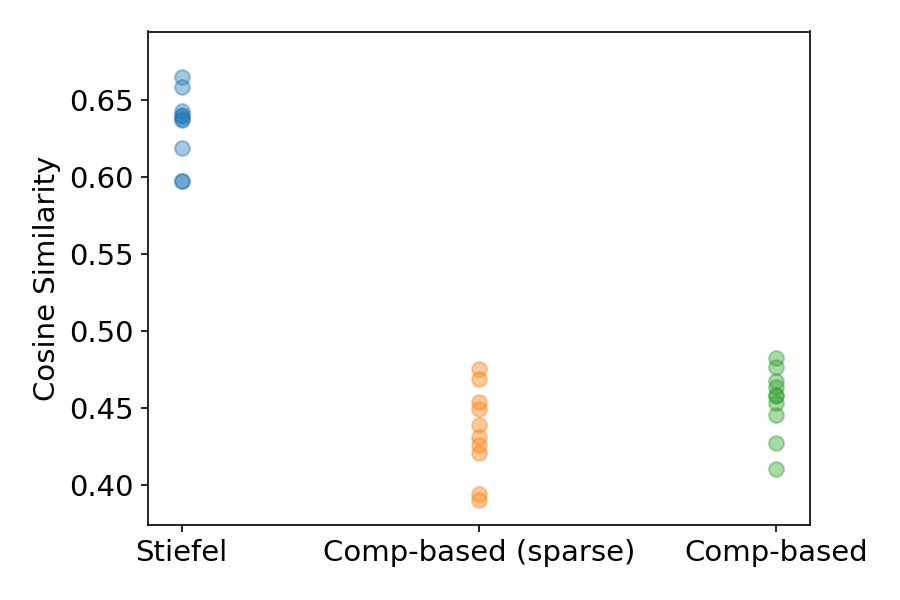}} \\
    \subfloat[$ x = \frac{\pi}{4} \mathbf{1}, \delta = 0.01, k = 100 $]{\includegraphics[scale = 0.45]{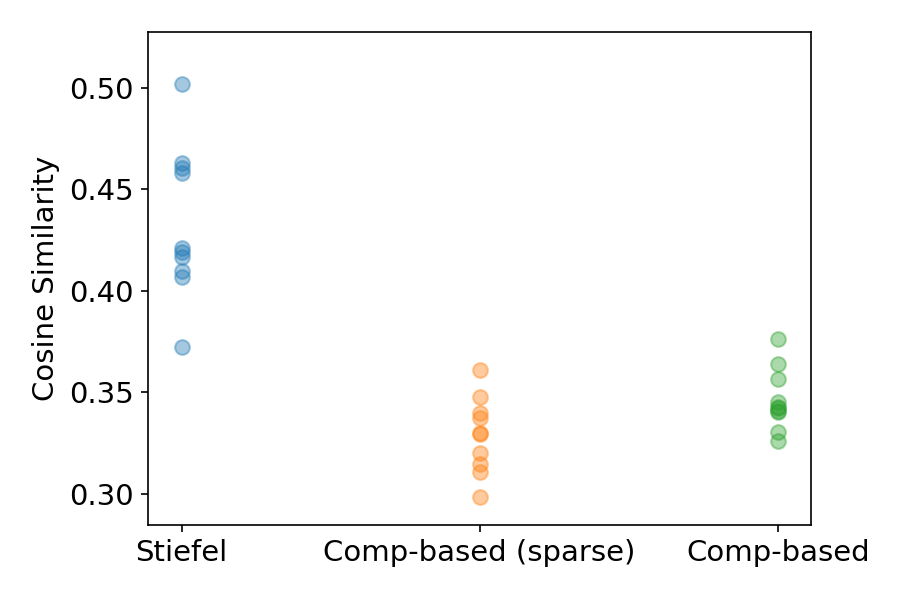}} 
    \subfloat[$ x = \frac{\pi}{4} \mathbf{1}, \delta = 0.01, k = 200 $]{\includegraphics[scale = 0.45]{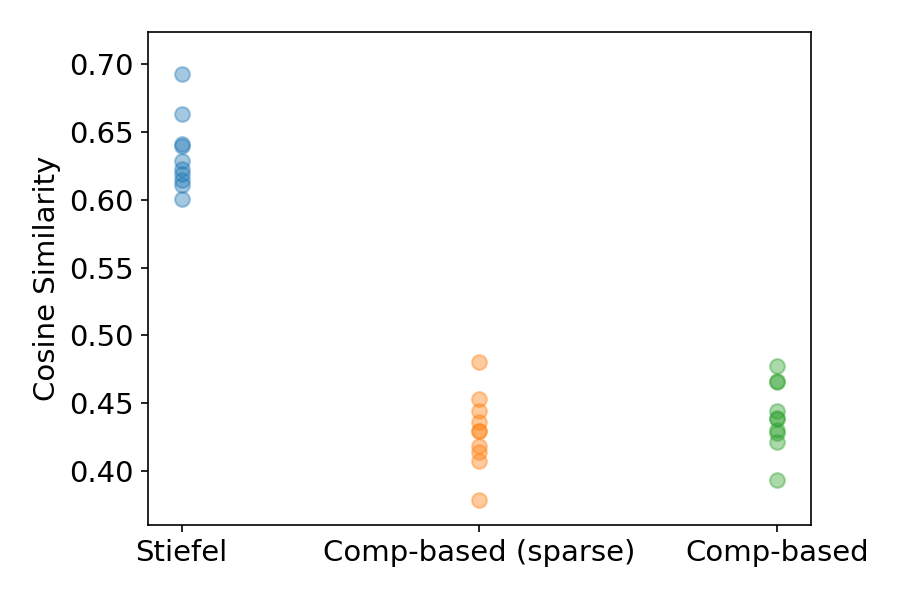}} \\
    \subfloat[$ x = \frac{\pi}{4} \mathbf{1}, \delta = 0.001, k = 100 $]{\includegraphics[scale = 0.45]{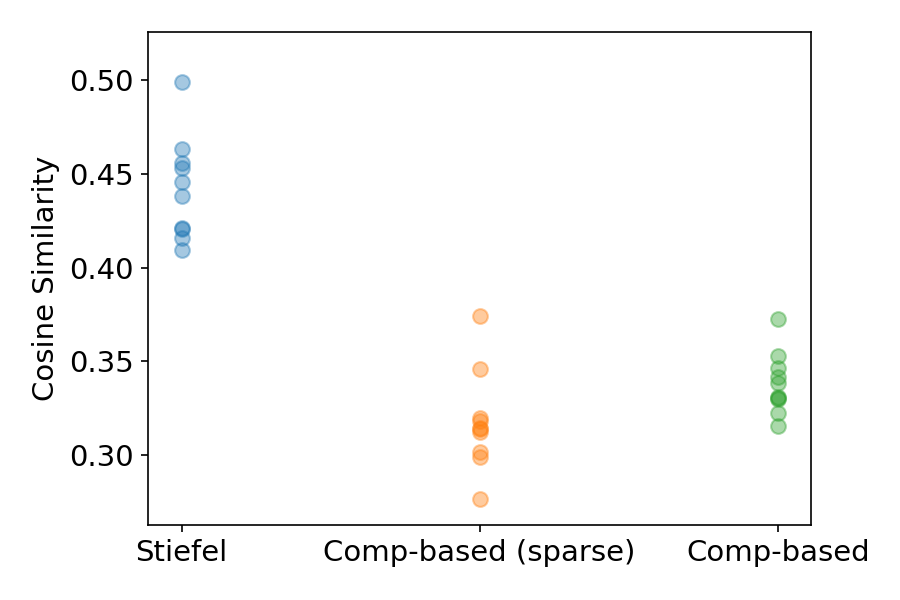}} 
    \subfloat[$ x = \frac{\pi}{4} \mathbf{1}, \delta = 0.001, k = 200 $]{\includegraphics[scale = 0.45]{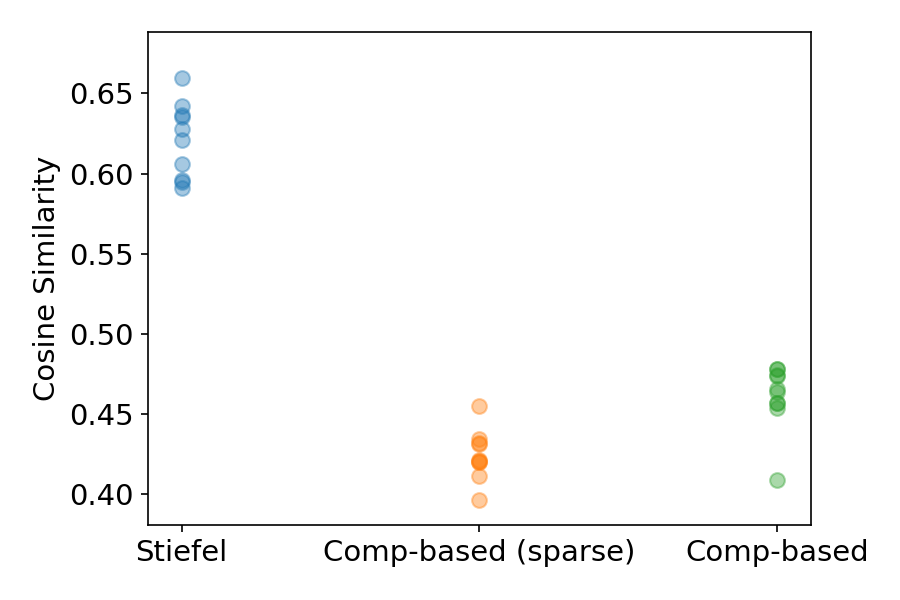}} 
    \caption{Cosine similarity between the gradient estimation and the true gradient, using the test function defined in (Eq. \ref{eq:test}). 
    Each subfigure corresponds to a different combination of the location for estimation $x$, the finite difference granularity $ \delta $, and number of random directions $k$. See the caption of Figure \ref{fig:grad-compare-main2} for detailed illustration.} 
\end{figure}

\begin{figure}
    \centering
    \subfloat[$ x = \frac{\pi}{4} \mathbf{1}, \delta = 0.1, k = 300 $]{\includegraphics[width = 0.48\textwidth]{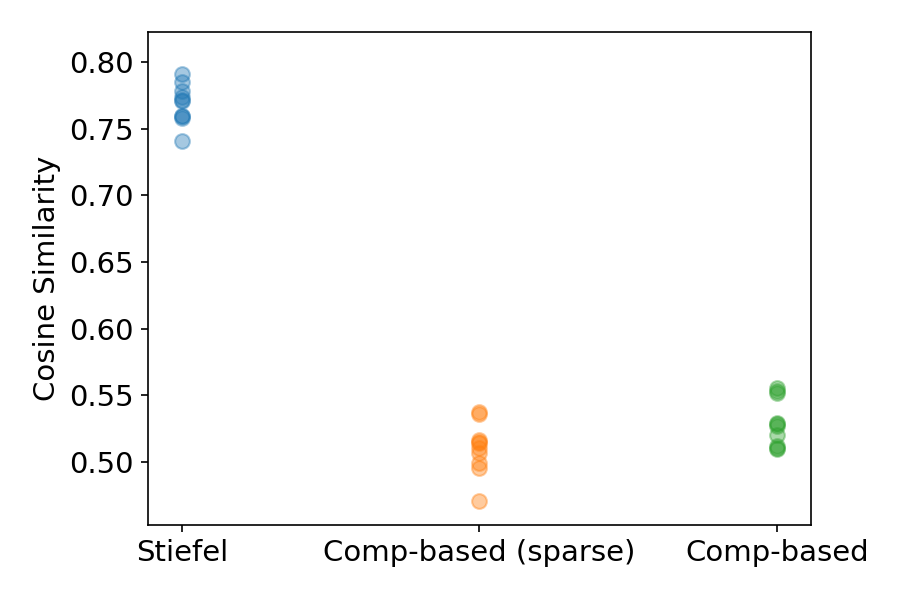}} 
    \subfloat[$ x = \frac{\pi}{4} \mathbf{1}, \delta = 0.1, k = 400 $]{\includegraphics[width = 0.48\textwidth]{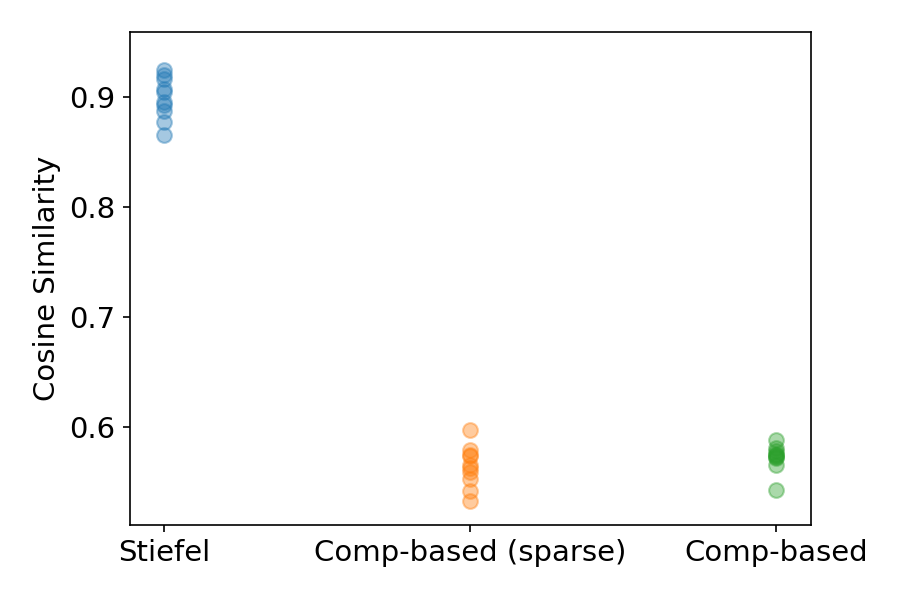}} \\
    \subfloat[$ x = \frac{\pi}{4} \mathbf{1}, \delta = 0.01, k = 300 $]{\includegraphics[width = 0.48\textwidth]{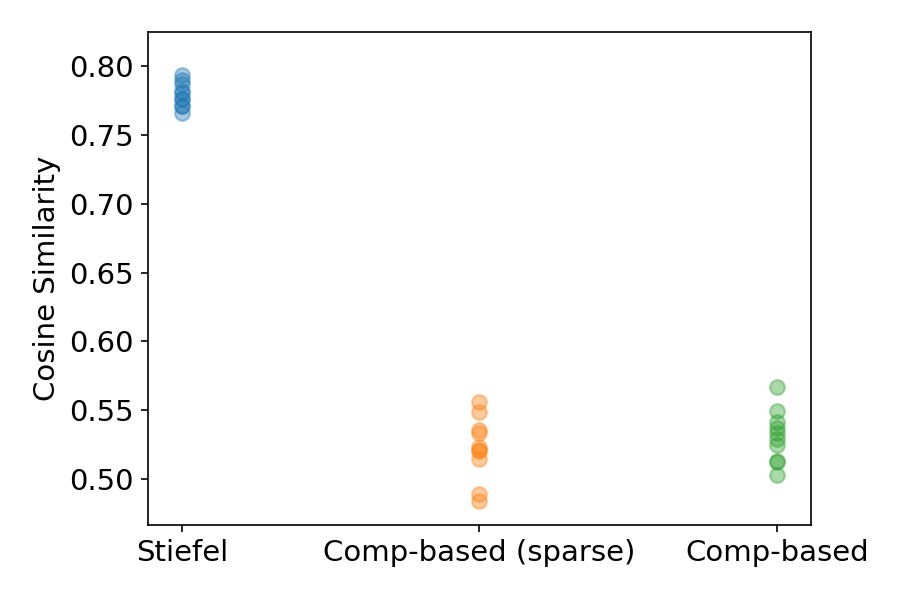}} 
    \subfloat[$ x = \frac{\pi}{4} \mathbf{1}, \delta = 0.01, k = 400 $]{\includegraphics[width = 0.48\textwidth]{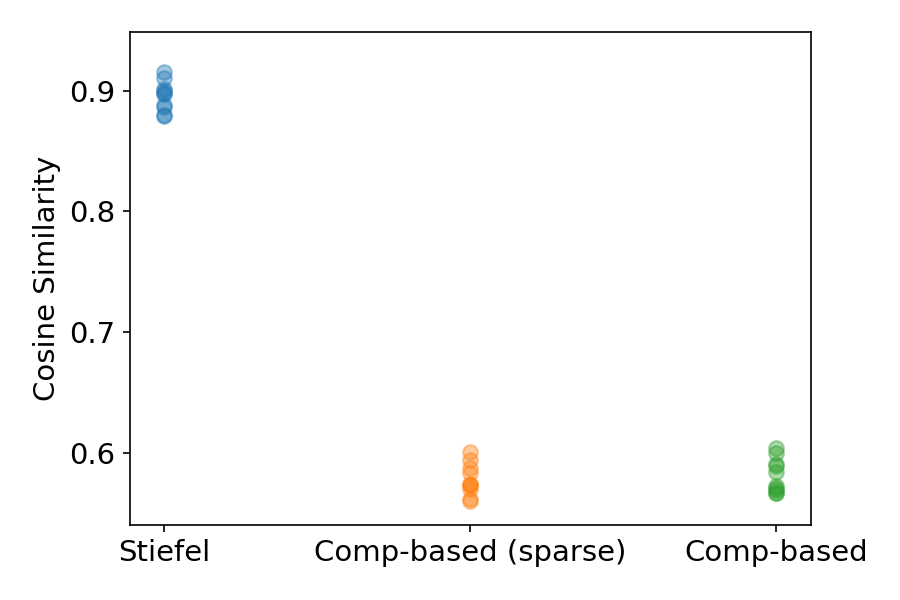}} \\
    \subfloat[$ x = \frac{\pi}{4} \mathbf{1}, \delta = 0.001, k = 300 $]{\includegraphics[width = 0.48\textwidth]{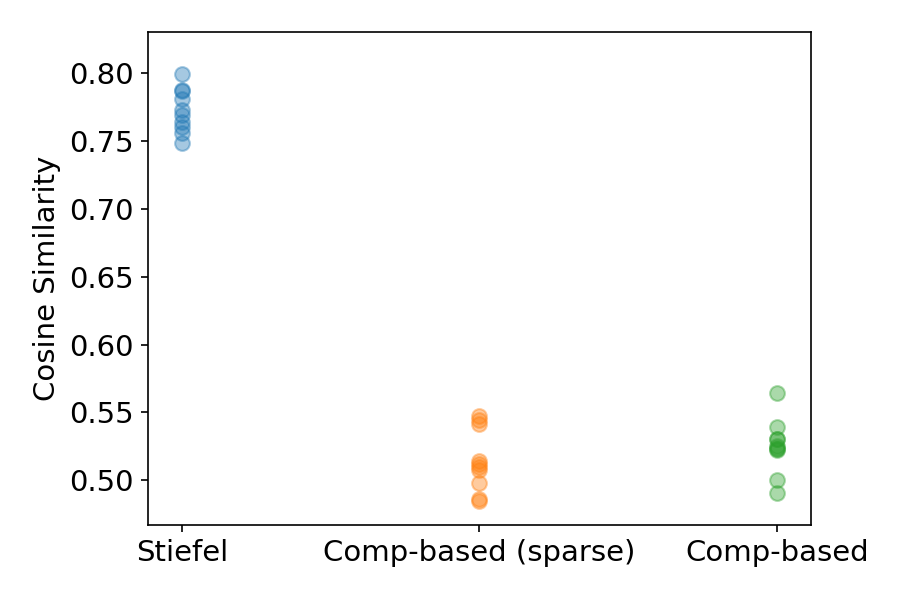}} 
    \subfloat[$ x = \frac{\pi}{4} \mathbf{1}, \delta = 0.001, k = 400 $]{\includegraphics[width = 0.48\textwidth]{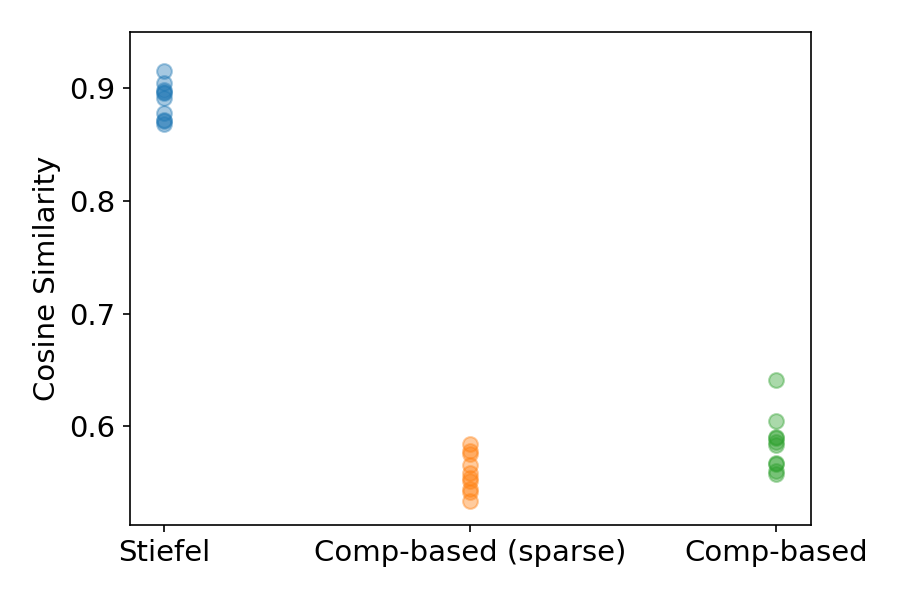}} 
    \caption{Cosine similarity between the gradient estimation and the true gradient, using the test function defined in (Eq. \ref{eq:test}). 
    Each subfigure corresponds to a different combination of the location for estimation $x$, the finite difference granularity $ \delta $, and number of random directions $k$. See the caption of Figure \ref{fig:grad-compare-main2} for detailed illustration.} 
\end{figure}

\begin{figure}[h] 
    \centering
    \subfloat[$ x = 0, \delta = 0.1, k = 20 $]{\includegraphics[width = 0.48\textwidth]{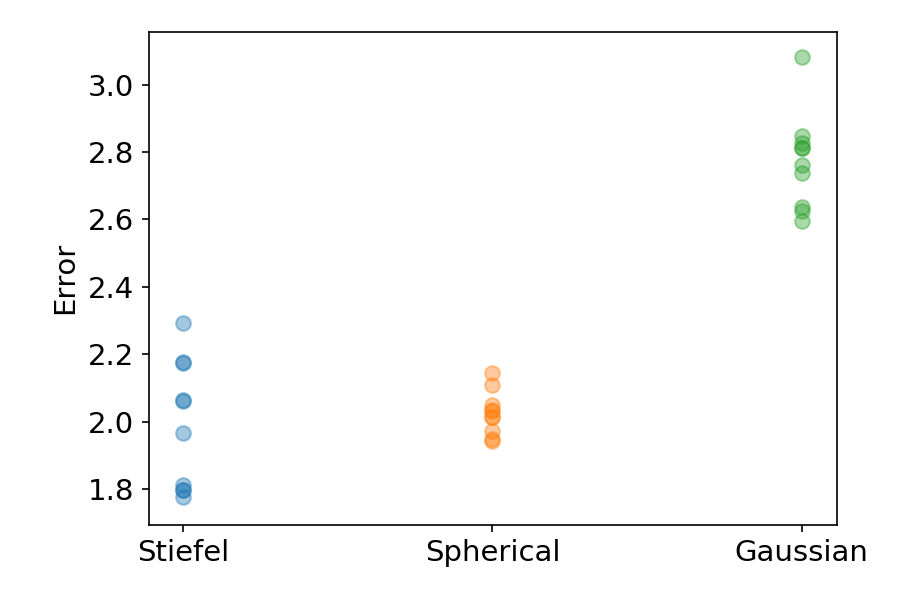}}
    \subfloat[$ x = 0, \delta = 0.1, k = 40 $]{\includegraphics[width = 0.48\textwidth]{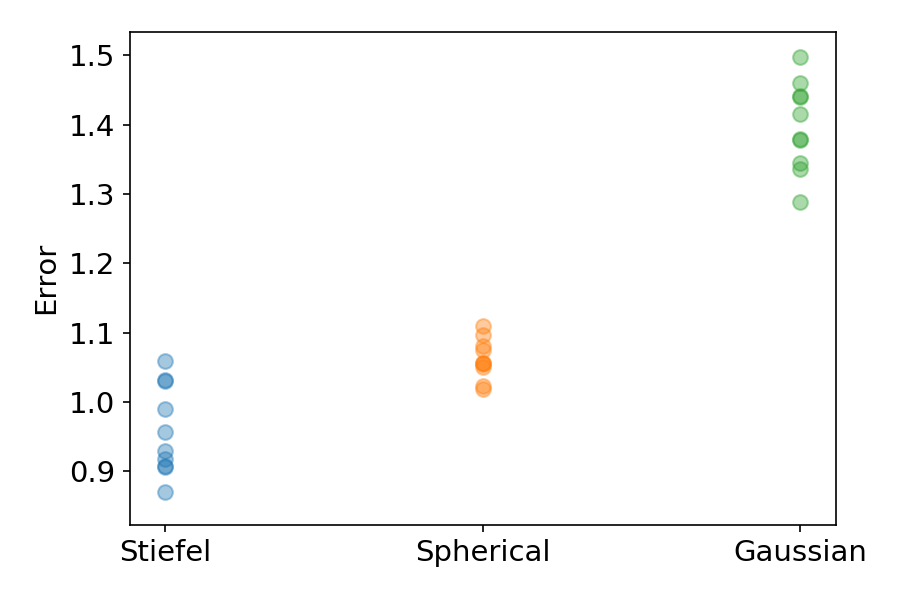}} \\
    \subfloat[$ x = 0, \delta = 0.01, k = 20 $]{\includegraphics[width = 0.48\textwidth]{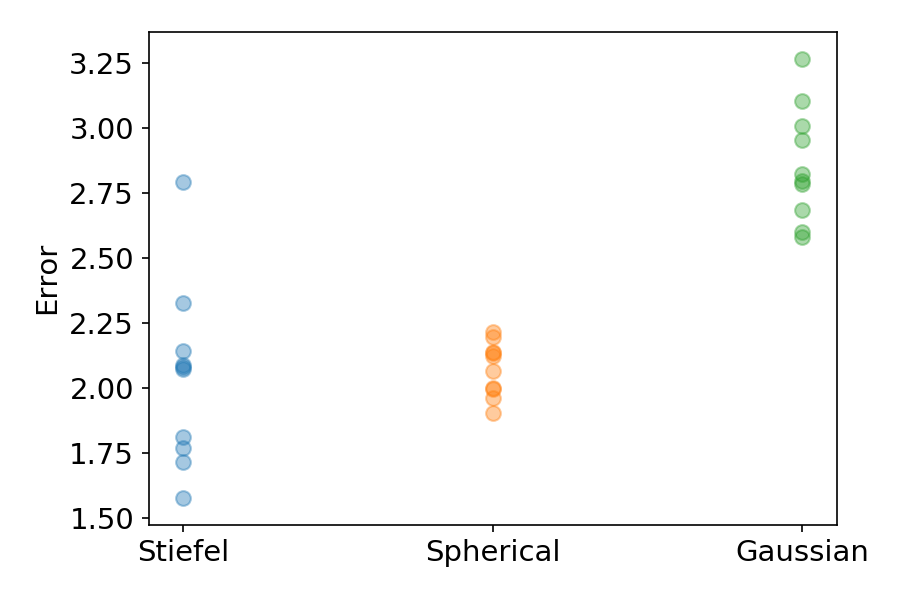}}
    \subfloat[$ x = 0, \delta = 0.01, k = 40 $]{\includegraphics[width =  0.48\textwidth]{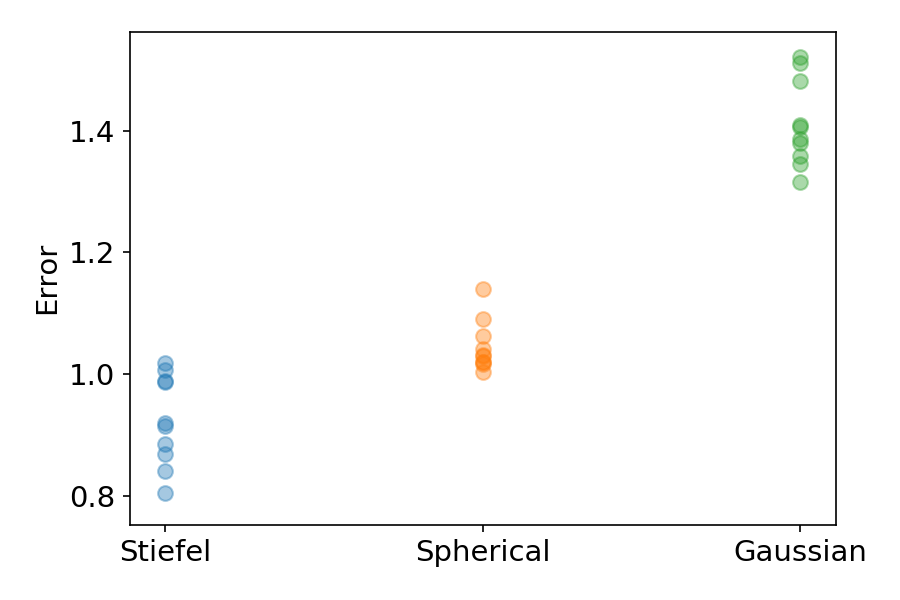}} \\ 
    \subfloat[$ x = 0, \delta = 0.001, k = 20 $]{\includegraphics[width = 0.48\textwidth]{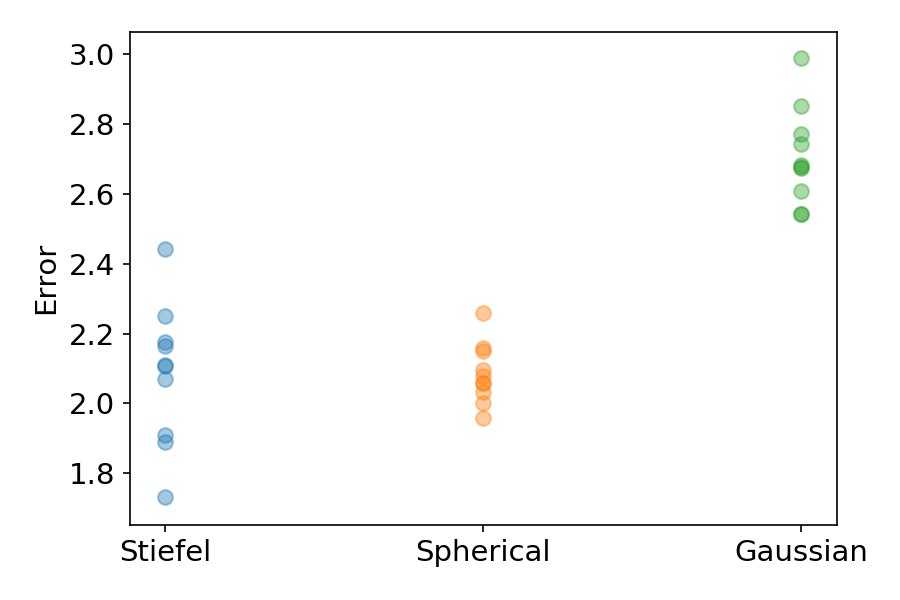}}
    \subfloat[$ x = 0, \delta = 0.001, k = 40 $]{\includegraphics[width = 0.48\textwidth]{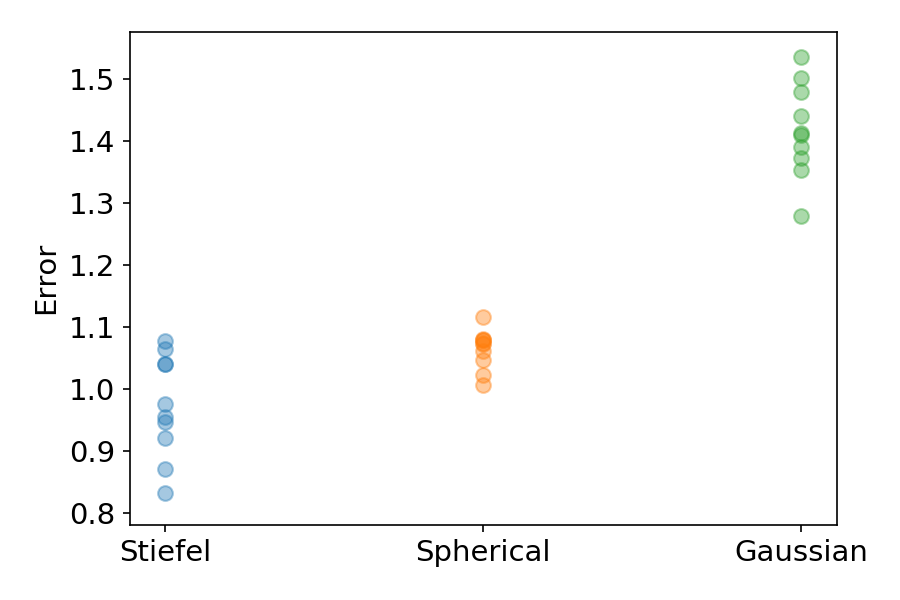}} 
    \caption{Errors of Hessian estimators on the test function defined in (Eq. \ref{eq:test}). 
    Each subfigure corresponds to a different combination of the location for estimation $x$, the finite difference granularity $ \delta $, and number of random directions $k$. See the caption of Figure \ref{fig:hess-compare-main} for detailed illustration. 
    \label{fig:hess-compare-app0}}
\end{figure}

\begin{figure}[h]
    \centering
    \subfloat[$ x = \frac{\pi}{4} \mathbf{1}, \delta = 0.1, k = 20 $]{\includegraphics[width = 0.48\textwidth]{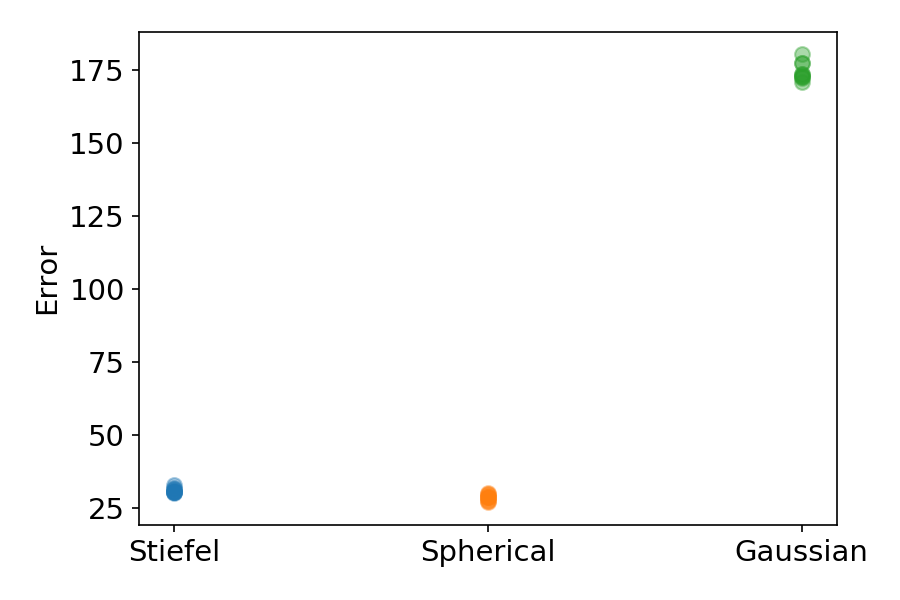}}
    \subfloat[$ x = \frac{\pi}{4} \mathbf{1}, \delta = 0.1, k = 40 $]{\includegraphics[width = 0.48\textwidth]{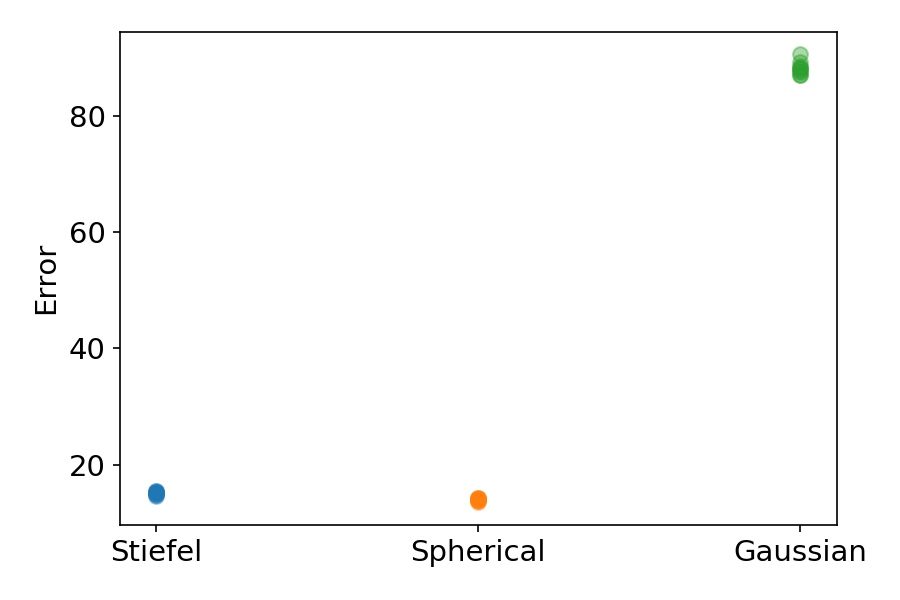}} \\ 
    \subfloat[$ x = \frac{\pi}{4} \mathbf{1}, \delta = 0.01, k = 20 $]{\includegraphics[width = 0.48\textwidth]{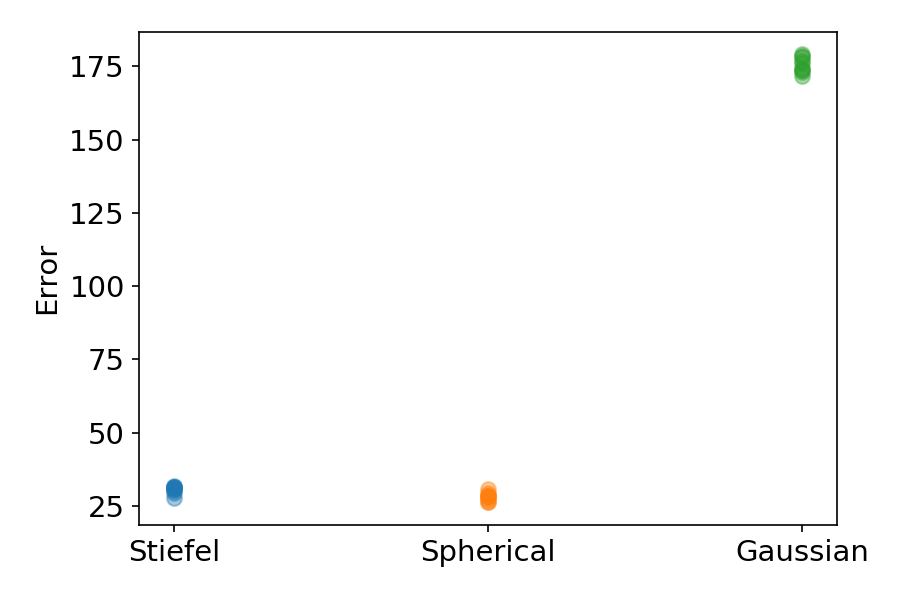}}
    \subfloat[$ x = \frac{\pi}{4} \mathbf{1}, \delta = 0.01, k = 40 $]{\includegraphics[width =  0.48\textwidth]{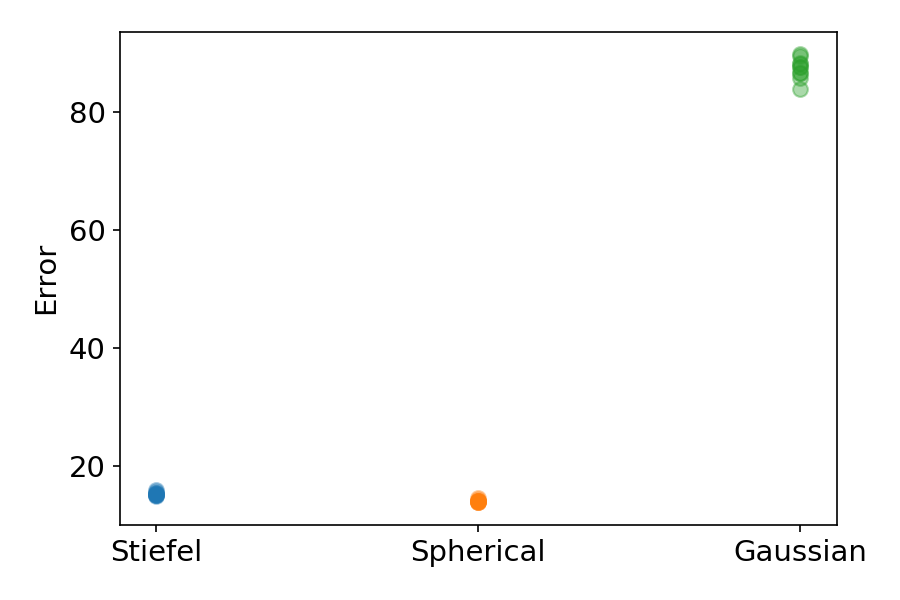}} \\ 
    \subfloat[$ x = \frac{\pi}{4} \mathbf{1}, \delta = 0.001, k = 20 $]{\includegraphics[width = 0.48\textwidth]{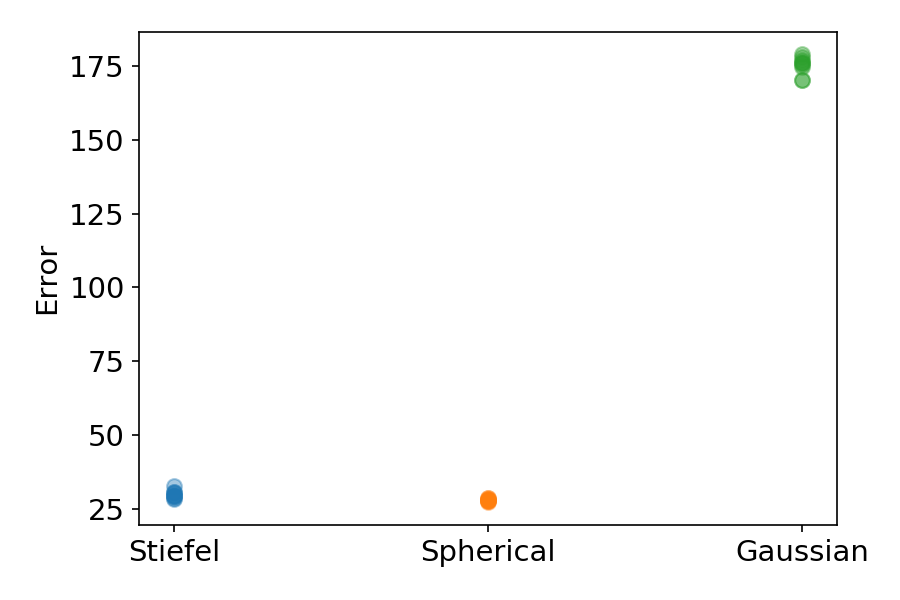}}
    \subfloat[$ x = \frac{\pi}{4} \mathbf{1}, \delta = 0.001, k = 40 $]{\includegraphics[width = 0.48\textwidth]{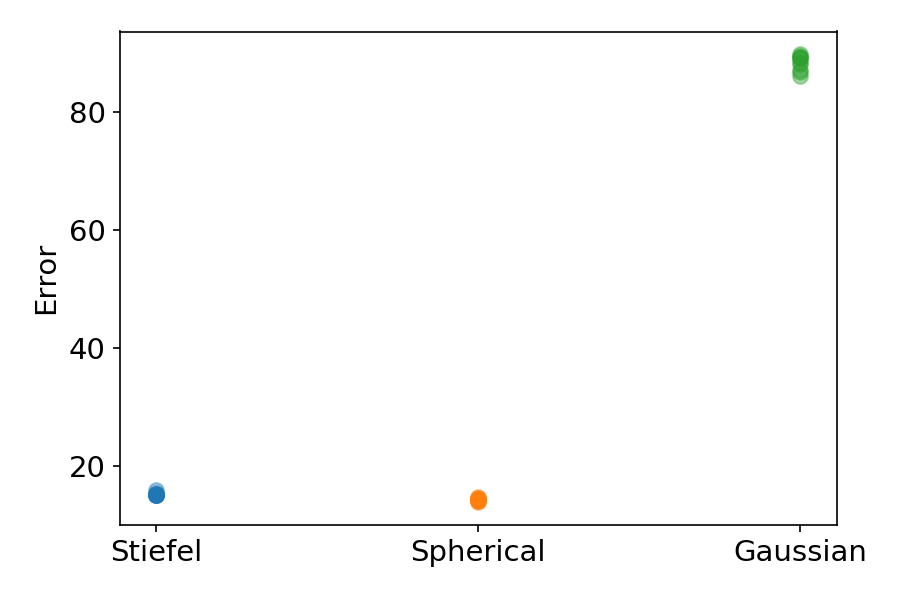}} 
    \caption{Errors of Hessian estimators on the test function defined in (Eq. \ref{eq:test}). 
    Each subfigure corresponds to a different combination of the location for estimation $x$, the finite difference granularity $ \delta $, and number of random directions $k$. This figure shows that when $k$ is much smaller than $n$, $\wh{\H} f_{k,S}^\delta (x)$ \citep{wang2022hess} can achieve same level of accuracy as $ \wh{\H} f_k^\delta (x) $. See the caption of Figure \ref{fig:hess-compare-main} for detailed illustration. 
    \label{fig:hess-compare-app1}}
\end{figure} 

\begin{figure}[h]
    \centering
    \subfloat[$ x = \frac{\pi}{4} \mathbf{1}, \delta = 0.1, k = 60 $]{\includegraphics[width = 0.48\textwidth]{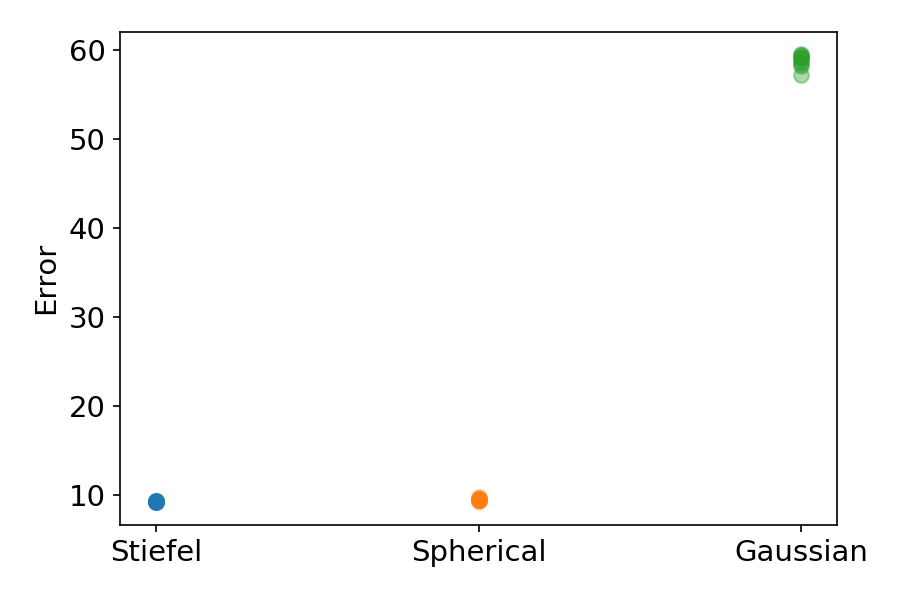}}
    \subfloat[$ x = \frac{\pi}{4} \mathbf{1}, \delta = 0.1, k = 80 $]{\includegraphics[width = 0.48\textwidth]{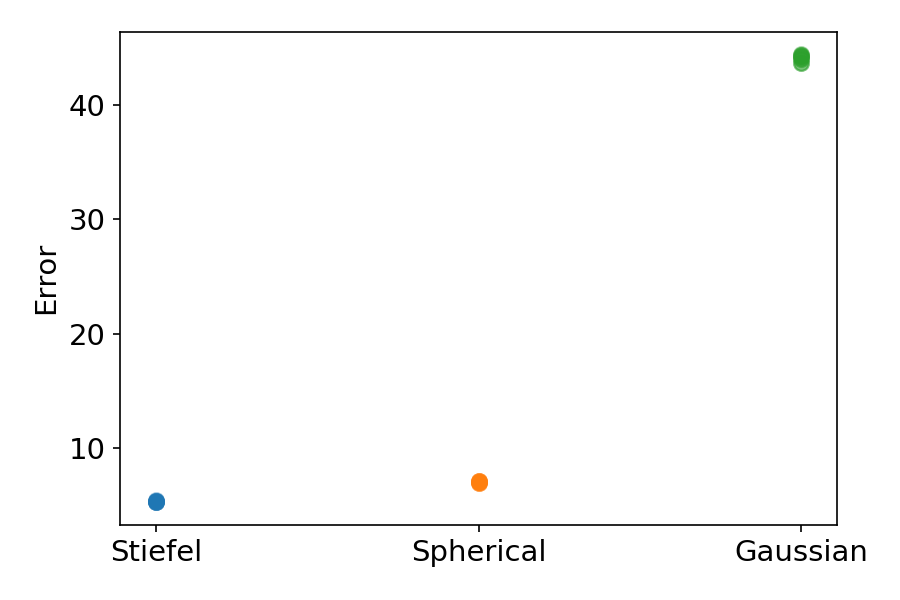}} \\
    \subfloat[$ x = \frac{\pi}{4} \mathbf{1}, \delta = 0.01, k = 60 $]{\includegraphics[width = 0.48\textwidth]{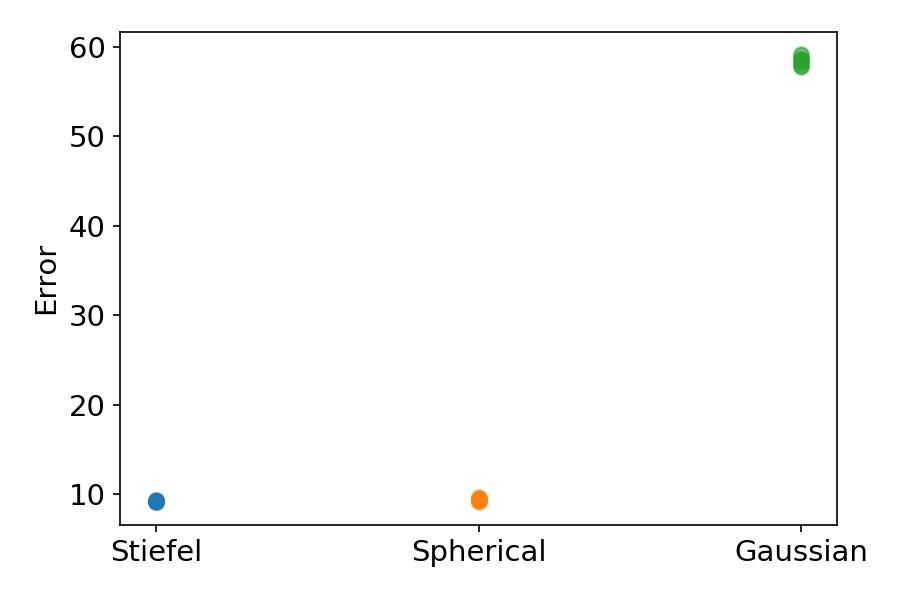}}
    \subfloat[$ x = \frac{\pi}{4} \mathbf{1}, \delta = 0.01, k = 80 $]{\includegraphics[width =  0.48\textwidth]{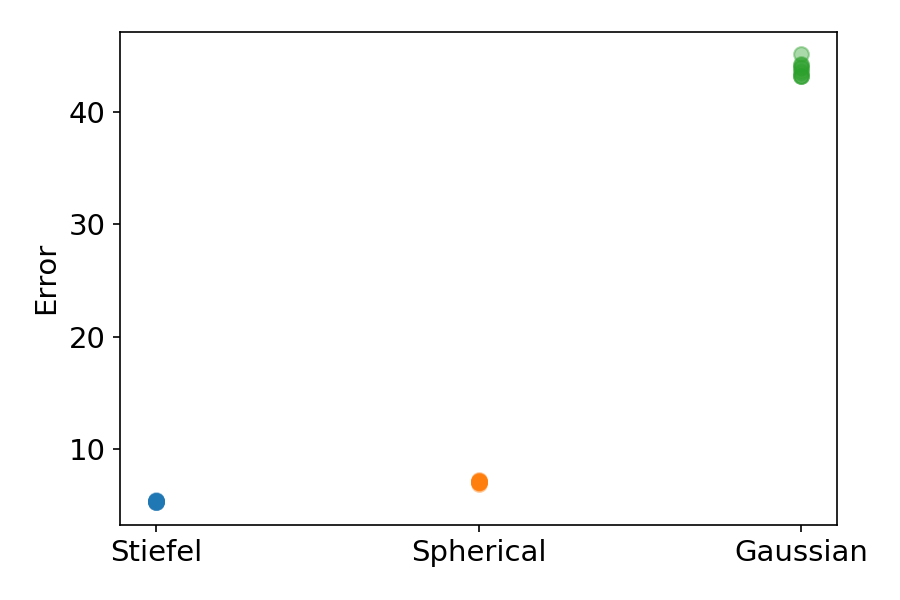}} \\ 
    \subfloat[$ x = \frac{\pi}{4} \mathbf{1}, \delta = 0.001, k = 60 $]{\includegraphics[width = 0.48\textwidth]{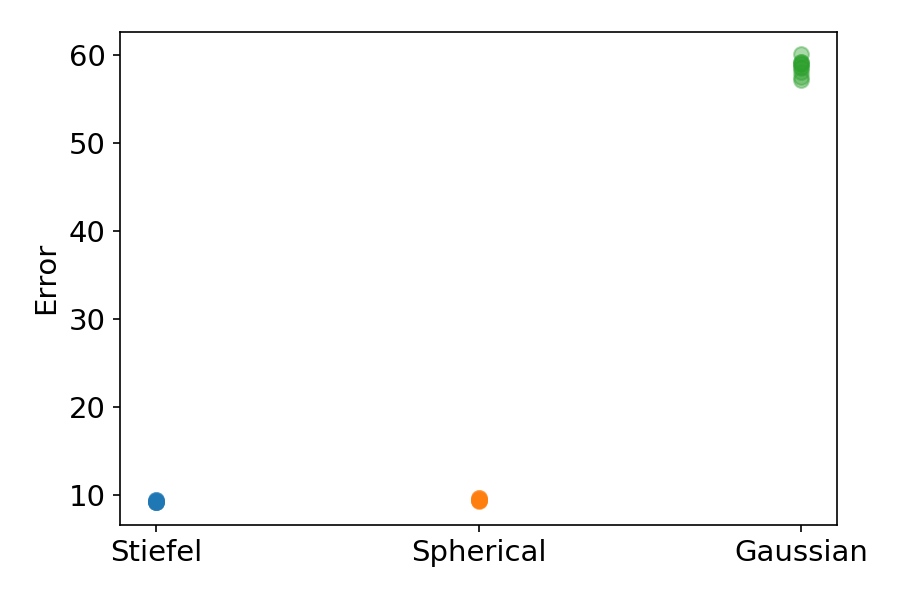}}
    \subfloat[$ x = \frac{\pi}{4} \mathbf{1}, \delta = 0.001, k = 80 $]{\includegraphics[width = 0.48\textwidth]{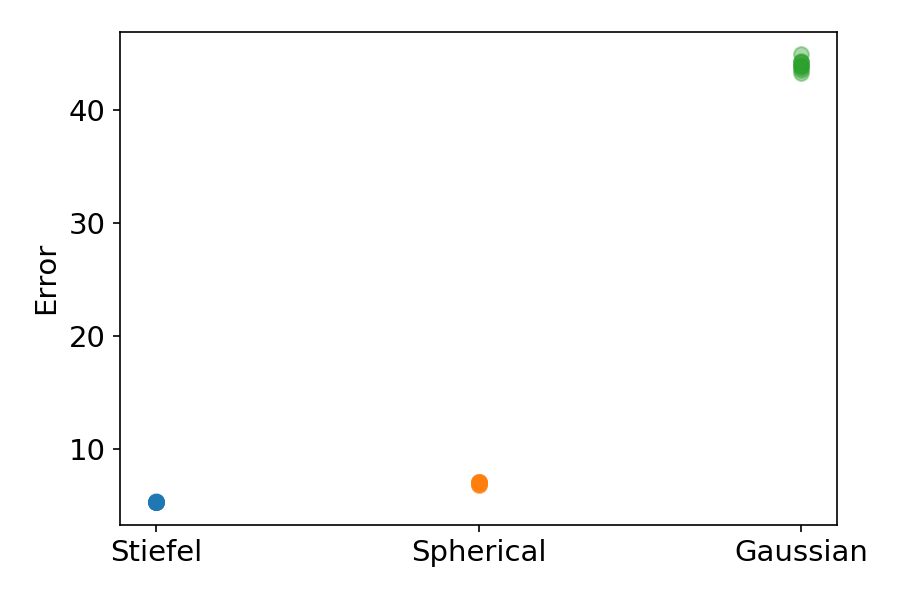}} 
    \caption{Errors of Hessian estimators on the test function defined in (Eq. \ref{eq:test}). 
    Each subfigure corresponds to a different combination of the location for estimation $x$, the finite difference granularity $ \delta $, and number of random directions $k$. See the caption of Figure \ref{fig:hess-compare-main} for detailed illustration. 
    \label{fig:hess-compare-app2}}
\end{figure}